\documentclass{article}

\PassOptionsToPackage{numbers,compress}{natbib}

\makeatletter
\g@addto@macro\normalsize{%
  \setlength\abovedisplayskip{5pt plus 2pt minus 2pt}%
  \setlength\belowdisplayskip{5pt plus 1pt minus 1pt}%
  \setlength\abovedisplayshortskip{4pt plus 2pt minus 2pt}%
  \setlength\belowdisplayshortskip{2pt plus 1pt minus 1pt}%
}
\makeatother

\usepackage[preprint]{neurips_2026}

\usepackage[utf8]{inputenc}
\usepackage[T1]{fontenc}

\usepackage{microtype}
\usepackage{graphicx}
\usepackage{subfigure}
\usepackage{booktabs} 
\usepackage{multirow}
\usepackage{url}
\usepackage{amsfonts}
\usepackage{nicefrac}
\usepackage{algorithm}
\usepackage{algorithmic}
\usepackage{float}
\usepackage{wrapfig}
\usepackage{capt-of}

\usepackage{hyperref}
\hypersetup{hidelinks}

\usepackage{amsmath}
\usepackage{amssymb}
\usepackage{mathtools}
\usepackage{amsthm}
\usepackage{comment}

\usepackage{xcolor}
\definecolor{code-gray}{rgb}{0.5,0.5,0.5}
\newcommand{\GrayComment}[1]{\textcolor{code-gray}{// #1}}

\usepackage[capitalize,noabbrev]{cleveref}
\usepackage[most]{tcolorbox}

\theoremstyle{plain}
\newtheorem{theorem}{Theorem}[section]
\newtheorem{proposition}[theorem]{Proposition}

\newtheorem{corollary}[theorem]{Corollary}
\theoremstyle{definition}
\newtheorem{definition}[theorem]{Definition}

\theoremstyle{remark}

\usepackage[disable,textsize=tiny]{todonotes}

\title{Distribution-Conditioned Transport}

\author{%
  Nic Fishman$^{*,1}$ \quad Gokul Gowri$^{*,2}$ \quad Paolo L. B. Fischer$^{*,3,4}$ \\
  \textbf{Marinka Zitnik}$^{4}$ \quad \textbf{Omar Abudayyeh}$^{4,5}$ \quad \textbf{Jonathan Gootenberg}$^{3,4}$ \\
  \small $^{1}$Harvard University \quad $^{2}$MIT \\
  \small $^{3}$Beth Israel Deaconess Medical Center \quad $^{4}$Harvard Medical School \\
  \small $^{5}$Brigham and Women's Hospital \quad $^{*}$Equal contribution \\
  \small Correspondence: \texttt{njwfish@gmail.com}, \texttt{gokulg@mit.edu}
}
\hypersetup{pdfauthor={Nic Fishman, Gokul Gowri, Paolo L. B. Fischer, Marinka Zitnik, Omar Abudayyeh, Jonathan Gootenberg}}

\begin{document}

\maketitle

\begin{abstract}

Learning a transport model that maps a source distribution to a target distribution is a canonical problem in machine learning, but scientific applications increasingly require models that can generalize to source and target distributions unseen during training. We introduce distribution-conditioned transport (DCT), a framework that conditions transport maps on learned embeddings of source and target distributions, enabling generalization to unseen distribution pairs drawn from a meta-distribution over admissible source--target pairs. DCT also enables semi-supervised learning for distributional forecasting problems: because it can learn from partially paired collections of distributions, it can leverage distributions observed at only one timepoint to improve transport prediction. DCT is agnostic to the underlying transport mechanism, supporting models ranging from flow matching to distributional divergence-based models (e.g. Wasserstein, MMD). We demonstrate the practical performance benefits of DCT on synthetic benchmarks and four applications in biology: batch effect transfer in single-cell genomics, perturbation prediction from mass cytometry data, learning clonal transcriptional dynamics in hematopoiesis, and forecasting T-cell receptor repertoires.

\end{abstract}

\section{Introduction}
\label{sec:intro}

\begin{wrapfigure}{r}{0.40\textwidth}
    \vspace{-3.6\intextsep}
    \centering
    \includegraphics[width=0.90\linewidth]{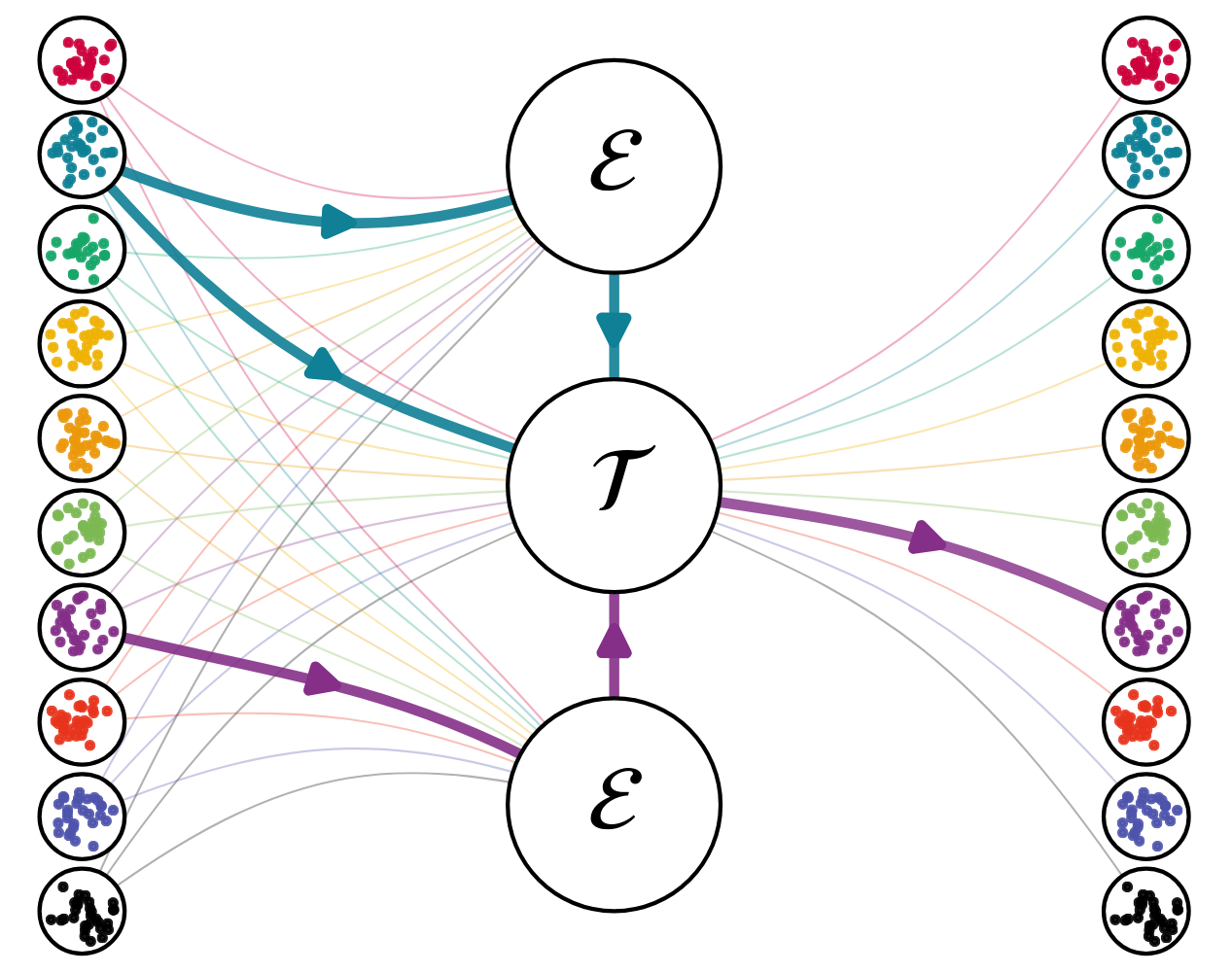}
    \caption{\textit{Distribution-conditioned transport.} A source distribution (teal) is pushed to a target distribution (purple) by a transport model $\mathcal{T}$ conditioned on source and target embeddings from $\mathcal{E}$. The same interface supports supervised source-conditioned training, any-to-any source--target conditioning, and zero-shot evaluation on new admissible pairs drawn from $Q_{\text{joint}}$.}
    \label{fig:schematic}
    \vspace{-0.4\intextsep}
\end{wrapfigure}

Modern datasets in the sciences increasingly have multiscale structure. An example of this complexity can be seen in single-cell RNA-sequencing (scRNA-seq) data. Rather than making measurements of cells from a single experimental condition, modern datasets sample cells from many distinct donors \citep{Yazar2022-bo}, timepoints \citep{Qiu2024-sb}, perturbations \citep{zapatero2023trellis}, or clones \citep{Weinreb2020-vh}, each of which can be thought of as inducing its own distribution over cell states. Experimental techniques such as clonal lineage-traced genomics have enabled the generation of datasets containing snapshots of thousands of distinct populations evolving in parallel \citep{Weinreb2020-vh,biddy2018single,wagner2020lineage}. These datasets can be sparse: not all time-marginals are observed for all populations. For example, some populations may be observed at both initial and final times $t_0$ and $t_1$, while others are ``orphans'' observed at only a single timepoint.

This multiscale structure can be formalized as a hierarchical model. Suppose we have sets of samples (e.g., cell gene expression profiles) from $n$ different conditions (e.g., donors). For each condition $i$, we observe a set of $m$ samples $S_i = \{x_{ij}\}_{j=1}^m$, where each sample is a vector $x_{ij} \in \mathcal{X}$. Each condition induces a distribution $P_i$, and we can think of these condition-specific distributions as being drawn from a shared metadistribution $Q$ over the space of probability measures $\mathcal{P}(\mathcal{X})$:
$x_{ij} \sim P_i$ and $P_i \sim Q$. This induces a learning problem over $P(\mathcal{X})$, where each distribution $P_i$ is a task drawn from $Q$. The hierarchical dataset consists of these $n$ sample sets: $\mathcal{D} = \{S_i\}_{i=1}^n$.

Many practical challenges in analyzing these hierarchical datasets can be framed as transport problems. For example, in scRNA-seq data, one often is interested in how a distribution would appear under different technical effects (batch integration), how distributions change in response to unseen treatment conditions (perturbation prediction), or forecasting how distributions evolve (dynamic inference). While a wide variety of tools exist to model transport between a single pair of distributions \citep{goodfellow2014gan,rezende2015nf,genevay2018sinkhorn,liu2022rectified_flow,lipman2023flow_matching,albergo2023stochastic_interpolants}, these hierarchical datasets and distributional tasks require transport models which can generalize across a broad range of distributions.

Several lines of work have begun to develop tools for this setting. Multimarginal stochastic interpolants (MMSI) learn dynamics between any pair of a fixed set of $K$ distributions, solving a $K$-to-$K$ transport problem \citep{albergo2023multimarginal}, but cannot condition on a continuous space of distributions nor generalize to distributions unseen during training. In another direction, \citet{fishman2025generative} develop approaches to learn ``autoencoders'' on the space of distributions by simultaneously learning to embed distributions and sample from the distribution conditional on that embedding. Meta flow matching (MFM) \citep{Atanackovic2024-ml} couples a population encoder with a flow-matching transport map for generalization across source populations. While effective in settings where many source-target pairs are available, the MFM framework cannot ingest unpaired marginal distributions (e.g., cell populations observed at a single timepoint). Furthermore, MFM is intended for settings where source and target populations have near deterministic relationships, and by conditioning only on the source, it cannot capture residual variation in the target that is not determined by the source population. 

In this work, we introduce \textit{distribution-conditioned transport} (DCT), a framework for solving transport problems on the space of distributions that unifies and generalizes these perspectives. We show how generative distribution embeddings (GDEs) developed in \citet{fishman2025generative} can be coupled with a broad class of transport models, enabling any-to-any transport by conditioning on both source and target distribution embeddings. This formulation enables us to generalize across distributions, predicting transport maps between population pairs unseen during training but drawn from the same meta-distribution over admissible source--target pairs (generalizing MMSI), as well as allowing us to make use of unstructured, partial observations such as orphan marginals (generalizing MFM). We demonstrate the effectiveness of our approach on synthetic benchmarks and real-world problems in biology. Our codebase is available \href{https://github.com/njwfish/DistributionConditionedTransport}{here}.

\section{Methods}\label{sec:methods}

Suppose we have sets of samples (e.g., cell gene expression profiles) from $n$ different conditions (e.g., donors). Each condition $i$ induces a distribution $P_i$ over a sample space $\mathcal{X}$, and we observe a sample set $S_i = \{x_{ij}\}_{j=1}^{m_i} \sim P_i$, where each $x_{ij} \in \mathcal{X}$. The distributions are drawn from a meta-distribution $Q$ over $\mathcal{P}(\mathcal{X})$: $x_{ij} \sim P_i$, $P_i \sim Q$. The hierarchical dataset is $\mathcal{D} = \{S_i\}_{i=1}^n$. We assume the admissible family admits a useful \emph{population embedding} $z_i^\star = \phi(P_i) \in \mathbb{R}^d$, a fixed-dimensional representation of the distribution for the transport problem of interest.

Throughout, $Q$ denotes only the marginal meta-distribution over individual populations. Transport tasks are specified by laws over ordered source--target pairs. We use $Q_{\text{pair}}$ for the supervised pair law that generates observed source--target pairs, and $Q_{\text{joint}}$ for the user-chosen admissible pair law used to train and evaluate source-target-conditioned models. In a fully any-to-any setting $Q_{\text{joint}} = Q \otimes Q$; in forecasting applications it may instead restrict support to forward-time pairs. Our generalization claims are with respect to new pairs drawn from $Q_{\text{joint}}$, not arbitrary pairs in $\mathcal{P}(\mathcal{X})^2$. Our theory below controls the additional error from estimating the conditioning embeddings from finite samples for a fixed pair.

We are interested in two classes of transport problems over this structure:

\textbf{Supervised (one-to-one) transport.} In many applications, the distributions admit natural source--target pairings: for example, cell populations before and after a perturbation. We formalize this through the pair law $Q_{\text{pair}}$ over paired distributions $(P^{\mathrm{src}}, P^{\mathrm{tgt}}) \sim Q_{\text{pair}}$, and the goal is to predict $P^{\mathrm{tgt}}$ from a new $P^{\mathrm{src}}$ at test time. Each pair induces population embeddings $(z_{\mathrm{src}}^\star, z_{\mathrm{tgt}}^\star) = (\phi(P^{\mathrm{src}}), \phi(P^{\mathrm{tgt}}))$. A useful way to view the forecasting problem is through the conditional law of $z_{\mathrm{tgt}}^\star$ given $z_{\mathrm{src}}^\star$: its conditional mean captures the expected target embedding, while its residual spread captures target variation not determined by the observed source. In clonal dynamics during differentiation, for instance, two clones with similar progenitor compositions can adopt different fates because of microenvironmental signals absent from the early-time snapshot.

\textbf{Unsupervised (any-to-any) transport.} In other settings, no source--target pairing is given by the data, and we instead transport between distribution pairs sampled from an admissible pair law $Q_{\text{joint}}$. This arises in two types of practical scenarios: (i) no specific pairing exists, as in batch integration, where we want to model what cells from one batch would look like under another's conditions; and (ii)~pairings exist, but are unobserved for a fraction of the data, as in lineage tracing, where some clones appear at multiple timepoints \citep{Weinreb2020-vh} and others are orphan marginals.

We first describe the distribution encoders that both approaches share (Section~\ref{subsec:encoders}), then formalize source-conditioned transport and identify its limitations (Section~\ref{subsec:source_conditioned}), and finally introduce source-target-conditioned transport, which overcomes these limitations and extends naturally to the unsupervised setting (Section~\ref{subsec:source_target_conditioned}).

\subsection{Distribution encoders}
\label{subsec:encoders}

The transport approaches introduced in this work rely on population embeddings $z_i^\star = \phi(P_i)$, which are not directly observed and must be evaluated from finite samples $S_i$. Following \citet{fishman2025generative}, we estimate them using a \emph{distribution encoder}
$\mathcal E : S_{i} \mapsto z_{i} \in \mathbb R^d$
trained so that the \emph{distribution embedding} $z_i = \mathcal{E}(S_i)$ approximates $z_i^\star$.
The key requirement is that $z_{i}$ reflects distributional signal rather than sampling noise. The structural part of this requirement is \emph{distributional invariance}: (i)~\emph{permutation invariance}, so that reordering samples does not change $\mathcal E(S_i)$, and (ii)~\emph{proportional invariance}, so that uniformly duplicating samples does not change $\mathcal E(S_i)$.
When these hold, the encoder depends only on the empirical measure
$\smash{\widehat P_i = \frac{1}{m_i}\sum_{j=1}^{m_i} \delta_{x_{ij}}}$,
so that $\smash{\mathcal E(S_i) = \phi(\widehat{P}_i)}$ for some empirical-measure functional $\phi$. The existence of a population embedding $z_i^\star=\phi(P_i)$, and the convergence of $\smash{\phi(\widehat P_i)}$ to it, are regularity assumptions about the admissible family and the representation being used. 

The central property for our purposes is a central limit theorem: under the empirical-process and Hadamard differentiability assumptions in App.~\ref{app:dct_theory},
\begin{equation}
\sqrt{m_{i}}\,
\bigl(
    \mathcal E(S_{i}) - z^\star_{i}
\bigr)
\;\xrightarrow{\;\; d \;\;}\;
\mathcal N(0, \Sigma_{i}),
\label{eq:clt}
\end{equation}
so replacing population embeddings with empirical ones incurs $O_{\mathbb P}(m^{-1/2})$ plug-in error for downstream objectives satisfying the corresponding loss-regularity assumptions  \citep{fishman2025generative}. App.~\ref{app:dct_theory} states this for a fixed pair and fixed model parameters (see Prop.~\ref{prop:transport_plugin_clt}).

\subsection{Source-conditioned transport}
\label{subsec:source_conditioned}

A natural approach to the supervised problem is to condition the transport map on an embedding of the source distribution. Given paired observations $(P_1^{\mathrm{src}}, P_1^{\mathrm{tgt}}), \ldots, (P_n^{\mathrm{src}}, P_n^{\mathrm{tgt}}) \sim Q_{\text{pair}}$ with empirical samples $S_i^{\mathrm{src}} \sim P_i^{\mathrm{src}}$ and $S_i^{\mathrm{tgt}} \sim P_i^{\mathrm{tgt}}$, we encode $z_{i}^{\mathrm{src}}=\mathcal E(S_{i}^{\mathrm{src}})$ and learn
\begin{equation*}
\mathcal T:\ \mathcal X\times\mathbb R^d\to\mathcal X,
\qquad
x_{i}^\mathrm{src} \mapsto \hat x_{i}^\mathrm{tgt}
=
\mathcal T\bigl(x_{i}^\mathrm{src} \mid z_\mathrm{src}\bigr),
\end{equation*}
by jointly fitting $\mathcal E$ and $\mathcal T$ with
\begin{equation}
\mathcal L_{\mathrm{sc}}
=
\ell\!\left(
S_i^{\mathrm{tgt}},\;
\mathcal T(S_{i}^\mathrm{src} \mid \mathcal{E}(S_{i}^\mathrm{src}))
\right),
\label{eq:supervised_loss_transport}
\end{equation}
where $\ell$ is the native loss of the chosen transport mechanism (see Sec.~\ref{sec:transport_models}). Several recent works take this approach, using transport mechanisms including flow-matching \citep{Atanackovic2024-ml,klein2025cellflow} and MMD objectives \citep{he2025morph}; our formulation generalizes these under a common framework.

In the infinite data regime, when the target is a deterministic function of the source embedding, source-conditioned models with sufficient capacity can achieve exact transport. However, if there is residual variation in the target distribution given a fixed source population and observed covariates, the optimal source-conditioned predictor is a loss-dependent aggregate over plausible targets. For pointwise flow-matching, Energy, and squared MMD objectives this aggregate is the conditional mixture endpoint under the conditions in App.~\ref{app:sc_aggregation}, while the SWD objective selects a Fr\'echet-median-type aggregate because the training risk averages distances to separate target populations rather than to a pooled conditional target law. The important point is common across these losses: source conditioning cannot choose among target variation that is not determined by the source. In the lineage tracing example, two clones with similar early states may have distinct later fates, but a source-conditioned model commits to the aggregate selected by its loss.

Source conditioning also constrains what data the model can use. The loss~\eqref{eq:supervised_loss_transport} requires paired source--target observations, so distributions observed in isolation cannot contribute to training.

\subsection{Source-target-conditioned transport}
\label{subsec:source_target_conditioned}

To address the semi-supervised and unsupervised settings, we condition on both source and target embeddings so the transport map
\begin{equation*}
\mathcal T:\ \mathcal X\times\mathbb R^d\times\mathbb R^d\to\mathcal X,
\qquad
x_i \mapsto \widehat x_{i\to i'}
=
\mathcal T\!\bigl(x_i \mid z_i,\, z_{i'} \bigr).
\end{equation*}
Because both source and target are inputs, training can sample any ordered pair in the support of $Q_{\text{joint}}$. One key detail is that the practitioner needs to choose this admissible pair law: in the unsupervised setting the any-to-any pairing $Q_{\text{joint}} = Q \otimes Q$ may make sense, while in semi-supervised settings it may be better to learn all ``early'' distributions to all ``late'' distributions instead (see App.~\ref{app:pairing}).

The source-target model can train on all available data --- paired and unpaired --- by sampling pairs from $Q_{\text{joint}}$, and the target embedding provides control over the transport destination. However, source-target-conditioned models require both source and target population embeddings as input. To recover a source-only interface at test time, we keep the transport model fixed and learn a post-hoc latent predictor from the available $(z_{\mathrm{src}}, z_{\mathrm{tgt}})$ pairs. This decomposes the problem: the source-target model handles within-distribution transport, while the latent predictor chooses which target embedding to condition on.

We use two different latent-prediction interfaces. The expectation predictor estimates $\mathbb E[z_{\mathrm{tgt}}\mid z_{\mathrm{src}}]$ by squared-error regression, producing one target embedding and one transported population. The distributional predictor estimates the full conditional law of $z_{\mathrm{tgt}}\mid z_{\mathrm{src}}$ with an energy-score objective on individual paired embeddings. This is analogous to the Energy generator loss, but it is applied in latent space to single paired instances rather than to minibatches of transported samples. At evaluation time, samples from this latent predictor induce a deployable mixture over transported populations; target-aided best-of-$k$ selection is used only as an oracle-gap diagnostic in the appendices.

Concretely, in the lineage tracing example, the source-target model can train on all clones --- including those observed at a single timepoint --- learning transport over the full embedding space. The latent predictor is fit on the paired subset. The training pairing policy $Q_{\text{joint}}$ can encode domain structure: restricting to forward-time pairs ($t_1 \to t_2$) avoids uninformative within-timepoint or reverse-time transport.

\begin{figure}[tb]
\begin{minipage}[t]{0.48\textwidth}
\begin{algorithm}[H]
   \caption{Source-conditioned training.}
   \label{alg:dct-sc}
\small
\begin{algorithmic}[1]
   \STATE {\bfseries Input:} Dataset $\mathcal{D}$, loss $\ell$, pair law $Q_{\text{pair}}$
   \FOR{each training step}
      \STATE Sample $(u, v) \sim Q_{\text{pair}}$
      \STATE Subsample $\hat{S}_u \subset S_u, \hat{S}_v \subset S_v$
      \STATE $z_u \gets \mathcal{E}(\hat S_u)$ \GrayComment{\small embed source}
      \STATE $\mathcal{L} \gets \ell(\hat S_v, \mathcal{T}(\hat{S}_u \mid z_u))$
      \STATE Backpropagate $\mathcal{L}$
   \ENDFOR
\end{algorithmic}
\end{algorithm}
\end{minipage}
\hfill
\begin{minipage}[t]{0.48\textwidth}
\begin{algorithm}[H]
   \caption{Source-target-conditioned training.}
   \label{alg:dct-stc}
\small
\begin{algorithmic}[1]
   \STATE {\bfseries Input:} Dataset $\mathcal{D}$, loss $\ell$, pair law $Q_{\text{joint}}$
   \FOR{each training step}
      \STATE Sample $(u, v) \sim Q_{\text{joint}}$
      \STATE Subsample $\hat{S}_u \subset S_u, \hat{S}_v \subset S_v$
      \STATE $z_u, z_v \gets \mathcal{E}(\hat S_u), \mathcal{E}(\hat S_v)$ \GrayComment{\small embed both}
      \STATE $\mathcal{L} \gets \ell(\hat S_v, \mathcal{T}(\hat{S}_u \mid z_u, z_v))$
      \STATE Backpropagate $\mathcal{L}$
   \ENDFOR
\end{algorithmic}
\end{algorithm}
\end{minipage}
\end{figure}

\begin{tcolorbox}
\begin{proposition} (\textbf{Loss CLT (informal)})\label{prop:transport_plugin_clt}
Fix an admissible pair $(P_u,P_v)$ and condition on the current encoder and transport parameters during training. For finite-direction SWD, also condition on the sampled projection directions.
Let $\widehat S_{u,m}$ and $\widehat S_{v,m'}$ be the source and target minibatches.
The shared-minibatch STC loss is
\[
\widehat{\mathcal L}_{\theta,m,m'}
=
\ell\!\left(
\widehat S_{v,m'},\;
\mathcal T_\theta\!\left(
\widehat S_{u,m}
\mid
\mathcal E(\widehat S_{u,m}),\mathcal E(\widehat S_{v,m'})
\right)
\right),
\]
Under the empirical-process and smoothness assumptions in App.~\ref{app:dct_theory}, with $m/m'\to\rho\in(0,\infty)$,
\[
\sqrt m\left(
\widehat{\mathcal L}_{\theta,m,m'}-\mathcal L_{\theta}^\star
\right)
\xRightarrow{\;d\;}
Z_{\theta,u,v},
\]
where $\mathcal L_{\theta}^\star$ is the corresponding conditional population loss and $Z_{\theta,u,v}$ is a Gaussian fluctuation.
\end{proposition}
\end{tcolorbox}

\section{Related work}

\paragraph{Distribution embeddings} 
Distribution conditioned transport relies on a vector representation of a distribution, i.e., a distribution embedding. Kernel mean embeddings represent probability measures as points in a reproducing kernel Hilbert space \citep{Smola2007-bs, Muandet2012-id, Oliva2013-ij, Szabo2015-ss, Muandet2017-np}. In parallel, other approaches learn distribution embeddings which approximately preserve Wasserstein distances between distributions \citep{Haviv2024-qj}, in analogy to multidimensional scaling. Our distribution encoders follow \citet{fishman2025generative}, learning distributionally-invariant embeddings via permutation-invariant architectures with a particular class of pooling operators \citep{Zaheer2017-tu, Wagstaff2021-xj, Zhang2022-ah}.

\paragraph{Transport models}\label{sec:transport_models}
Estimating even a single transport map is theoretically challenging: minimax rates for smooth optimal transport maps degrade with dimension \cite{Hutter2019-ry}. In practice, however, the problem of learning transport maps in high dimensions appears feasible. Many methods for generative conditional transport exist, with demonstrated success across data domains \cite{albergo2023stochastic_interpolants, liu2022rectified_flow, gat2024discreteflowmatching}. All of these are admissable as conditional maps $\mathcal T(\cdot \mid z)$ or $\mathcal T(\cdot \mid z,z')$ in our framework. We review several notable transport mechanisms below.

Classical adversarial approaches learn source--target mappings via discriminators, including
GANs \citep{goodfellow2014gan} and Wasserstein GANs \citep{arjovsky2017wgan}.
It is also possible to learn transport maps by minimizing distributional divergences \citep{gneiting2007strictly,li2015gmmn,sutherland2016generative,genevay2018sinkhorn,kolouri2019generalized}.
Normalizing flows provide invertible transport parameterizations through continuous change-of-variable models \citep{rezende2015nf}.
Flow-based models, such as flow matching, rectified flows, and stochastic interpolants
\citep{lipman2023flow_matching,liu2022rectified_flow,albergo2023stochastic_interpolants},
model transport by learning a velocity field along a path interpolating between distributions.

On discrete sequences, machine translation methods
\citep{lample2018unsupervised_mt,artetxe2018unsupervised_nmt,lample2018phrase_neural_umt}
learn bidirectional maps between languages using only monolingual corpora, providing a discrete analogue of our unpaired any-to-any transport between distributions. More recently, flow-matching has been extended to the simplex and general discrete domains 
\citep{gat2024discreteflowmatching,cheng2024categorical,davis2024fisher}.

Our framework is orthogonal to the specific transport mechanism: any such model can be conditioned on distribution embeddings to yield a source-conditioned map fit by the loss in Eqn.~\eqref{eq:supervised_loss_transport} or a source--target--conditioned map as described in Sec.~\ref{subsec:source_target_conditioned}.

\paragraph{Multi-distribution transport} 

The seemingly harder problem of learning multiple transport maps simultaneously also appears feasible in practice. Several methods in biology \cite{Atanackovic2024-ml, klein2025cellflow, he2025morph} solve multi-distribution transport problems with high-dimensional data. In this work, we generalize these as instances of source-conditioned transport and extend them with source-target-conditioning.

Meta Flow Matching \citep{Atanackovic2024-ml} develops the idea of source-conditioning for learning transport maps across contexts. Our notion of source-conditioned models formalizes and generalizes their approach to a broader class of transport models, clarifying how unresolved target variation is aggregated when it is not determined by the source: pointwise flow-matching and MMD objectives target conditional mixture endpoints, whereas pairwise Wasserstein objectives target barycenter-type aggregates. A number of recent cellular perturbation response prediction methods are also source-conditioned in the same sense as meta flow matching \citep{klein2025cellflow,he2025morph,adduri2025predicting}.


Multimarginal stochastic interpolants (MMSI) studies a form of any-to-any transport between a fixed set of $K$ distributions \citep{albergo2023multimarginal}. This is a particular instance of the unsupervised modelling problem we consider, where $Q$ is chosen to cover a discrete set of distributions that are all seen at training time. 
Additionally, MMSI enforces that transport paths go through all $K$ distributions to learn a barycenter, but in DCT we do not impose any constraint. 

Style-transfer and unpaired image-translation methods address a similar problem. CycleGAN learns separate maps between two domains \citep{zhu2017cyclegan}, while StarGAN and MUNIT \citep{choi2018stargan,huang2018munit} condition the generator on a domain identifier or an instance-level style code. These approaches are restricted to a fixed, finite set of domains and cannot generalize transport to a new domain without retraining.


\paragraph{Transport models for biological data}

The family of transport models described above are widely used in biology. They have attracted significant recent interest for at least two problems: predicting cellular perturbation responses from scRNA-seq, imaging, and mass cytometry data \citep{Atanackovic2024-ml, klein2025cellflow, adduri2025predicting, he2025morph, Bunne2023-cd}, and inferring cellular dynamics from snapshot observations \citep{Schiebinger2019-pu, Lavenant2021-pb, Tong2023-xg, Tong2023-ye, Vinyard2025-rq}.

\section{Unsupervised Transport}

Using both synthetic data and real scRNA-seq data, we evaluate whether source-target-conditioned DCTs can address the any-to-any unsupervised transport problem in practice. We compare source-target-conditioned DCTs using a mean pooled deep-sets encoder (see App.~\ref{app:standard_encoder}) to $K$-to-$K$ models implemented using a ``one-hot'' $\mathcal{E}$ rather than a distributionally invariant architecture. Using the FM transport map, this is effectively a MMSI model \citep{albergo2023multimarginal}, with minor differences we discuss in App. \ref{app:baselines}. To enable test-time generalization, $K$-to-$K$ models are conditioned on the most similar training set distributions. We implement both approaches across sliced Wasserstein distance (SWD) models, Energy score models, and flow matching (FM) models (see App. \ref{app:standard_generators}), and evaluate using squared MMD in the main text with SWD and Energy distances available for all experiments in the appendix (\ref{app:standard_metrics}). We also include targeted encoder ablations in the appendices: learned set encoders outperform fixed kernel mean embeddings (KME) in the unsupervised settings where KME is run, while Set Transformer variants give similar or stronger conclusions (Apps.~\ref{app:gaussian} and~\ref{app:batch}). Targeted real-data ablations often focus on the Energy generator, which is the strongest or near-strongest continuous-data generator in the main real-data benchmarks; the primary benchmark tables report all generator families and metrics.

\subsection{Gaussian transport}\label{subsec:unsup_gauss}

\begin{wrapfigure}[12]{r}{0.52\textwidth}
    \vspace{-6.5\intextsep}
    \centering
    \includegraphics[width=0.93\linewidth]{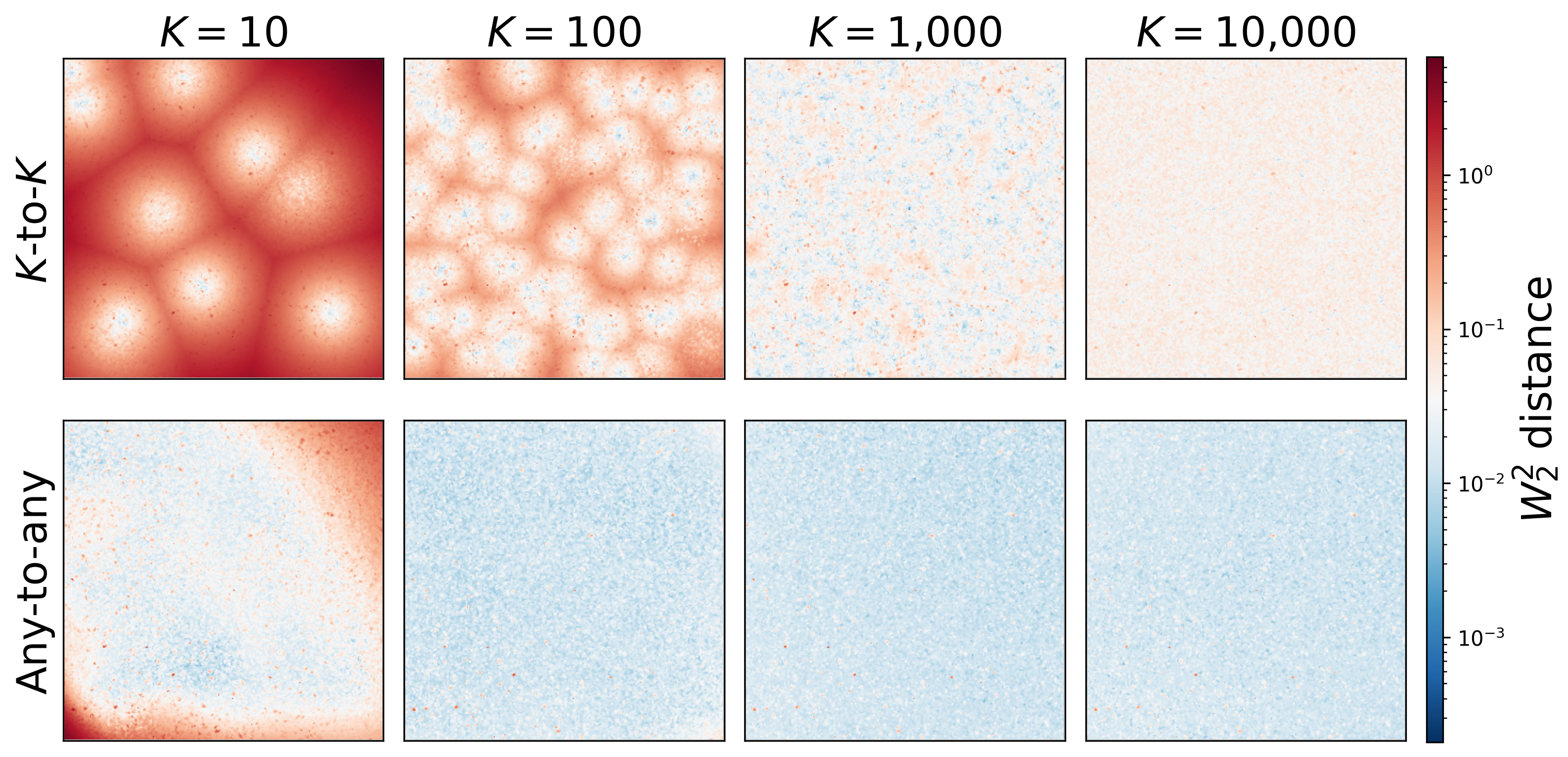}
    \caption{\textit{MVN transport error landscape.} A fixed source distribution is transported to target MVNs with means on a grid $\mu \in [0,5]^2$. Color shows Gaussian $W_2^2$ error for $K$-to-$K$ models (top) and any-to-any DCT models (bottom) as the number of unique training distributions $K$ increases.}
    \label{fig:mvn}\vspace{-1.0\intextsep}
\end{wrapfigure}

We first study transport between simple distributions in two dimensions, where we can control the data-generating process and directly visualize performance. 

We construct datasets of bivariate normal (MVN) distributions and Gaussian mixture models (GMM). For MVN, we sample parameters $(\mu_i, \Sigma_i)$, $i = 1,\dots,n$, with means $\mu_i \in [0,5]^2$ drawn from a uniform prior and covariances $\Sigma_i \in \mathbb R^{2\times 2}$ drawn from an inverse–Wishart prior. For GMM, we additionally sample mixture weights from a Dirichlet prior and per-component parameters (see App.~\ref{app:gaussian_dgp} for details). We choose $K \in \{10, 10^2, 10^3, 10^4\}$ distinct parameter sets and draw $n = 50{,}000$ sample sets.

\begin{wrapfigure}{L}{0.50\textwidth}
    \vspace{-1.5\intextsep}
    \centering
    \includegraphics[width=0.96\linewidth]{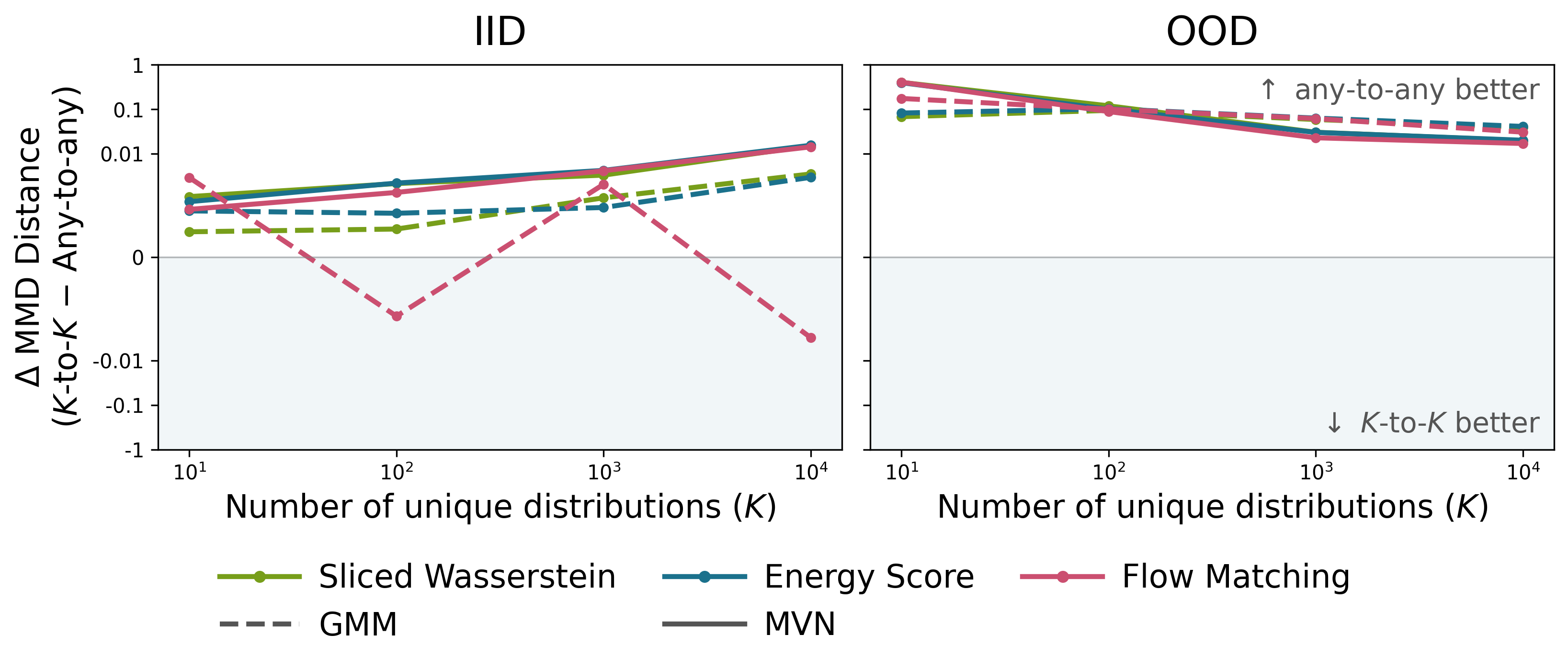}
    \caption{\textit{Unsupervised transport generalization.} The plot reports $\Delta$MMD, defined as $K$-to-$K$ error minus any-to-any DCT error, for IID and OOD targets across MVN/GMM benchmarks, generator families, and values of $K$. Positive values indicate lower error for any-to-any conditioning.}
    \label{fig:unsup}\vspace{-0.8\intextsep}
\end{wrapfigure}

For each $K$ we train two types of model. The first is a conventional \emph{$K$-to-$K$} architecture that treats each of the $K$ distinct distributions as a discrete label (a ``corner'') and learns transport conditional on distribution-specific labels. At test time, this model cannot condition on a new target distribution, so we assign each target to its nearest training distribution. The second is our \emph{source–target conditioned} model from Sec.~\ref{subsec:source_target_conditioned}, which encodes both source and target sets via the distribution encoder $z_i = \mathcal E(S_i)$ and learns a transport map $\mathcal T(x \mid z_{\mathrm{src}}, z_{\mathrm{tgt}})$ that can operate between pairs of distributions within this family, including unseen targets. 


To visualize zero-shot performance, we fix a single source distribution from the training set and evaluate transport to a dense grid of target distributions, with $\mu \in [0, 5]^2$ and corresponding random covariances. Results (Figs.~\ref{fig:mvn} and \ref{fig:unsup}) show that for $K\geq100$, DCT achieves lower out-of-distribution error across all reported generators and metrics. This pattern holds not only for MVN but also GMM, suggesting that the benefits of DCT extend to complex distributions. Overall, K-to-K models specialize to distributions seen during training, while source-target-conditioned (any-to-any) models interpolate through the learned distribution embedding space to generalize to unseen targets.

\subsection{Transferring scRNA-seq batch effects}\label{subsec:unsup_batch}

\begin{wraptable}[11]{r}{0.46\textwidth}
    \vspace{-5\intextsep}
    \centering
    \small
    \setlength{\tabcolsep}{3pt}
    \begin{tabular}{lcc}
    \toprule
     Generator & $K$-to-$K$ & Any-to-any \\
    \midrule
    scVI & \multicolumn{2}{c}{0.6178{\scriptsize$\pm$0.0210}} \\
    Harmony & \multicolumn{2}{c}{0.2057{\scriptsize$\pm$0.0335}} \\
    \midrule
    FM & 0.2260{\scriptsize$\pm$0.0292} & 0.0839{\scriptsize$\pm$0.0089} \\
    SWD & 0.2835{\scriptsize$\pm$0.0217} & 0.0928{\scriptsize$\pm$0.0061} \\
    Energy & 0.1826{\scriptsize$\pm$0.0180} & \textbf{0.0428}{\scriptsize$\pm$0.0036} \\
    \bottomrule
    \end{tabular}
    \caption{\textit{Batch transfer quality.} Values are squared MMD (mean $\pm$ std. err., $\downarrow$) averaged across held-out batch pairs. Any-to-any DCT conditions on encoded source and target batch distributions; $K$-to-$K$ uses nearest-neighbor training batches at test time.}
    \label{tab:batch}
    \vspace{-0.7\intextsep}
\end{wraptable}

A common problem in scRNA-seq analysis is technical variation between experimental batches \citep{Luecken2022-pw}. We show that source-target-conditioned models can be used to push cells from a source batch to a target batch condition, making a prediction for how cells appear under different technical effects. This problem, which we refer to as batch effect transfer, is closely related to batch integration.

Source-target-conditioning allows us to fit a single batch transfer model that can be applied to unseen experimental batches at test time. We demonstrate this using a real-world dataset of murine pancreas cells from 56 samples (experimental batches) across ages and growth conditions \citep{Hrovatin2023-wd}. We hold out all batches from a single condition (old age), inducing a distribution shift.

The $K$-to-$K$ models cannot directly adapt to this distribution shift because they condition only on training distribution identities at test time, while DCTs encode the held-out source and target distributions directly. DCTs outperform $K$-to-$K$ alternatives and scVI/Harmony baselines \citep{Korsunsky2019-ji, Lopez2018-lh} on held-out batch transfer (Table~\ref{tab:batch}; details and visualization in App.~\ref{app:batch}).

\section{Semi-supervised Transport}\label{sec:applications}

Across synthetic data, mass cytometry data, and protein sequence data, we evaluate whether source-target-conditioned DCTs improve performance in the semi-supervised setting. For continuous data, we compare source-target-conditioned DCTs with their source-conditioned counterparts, across SWD-based, Energy-based, and FM transport maps (see App. \ref{app:standard_generators}). The sequence experiments use autoregressive and discrete-flow generators. We evaluate using MMD in the main text, with the additional metrics reported in each experiment appendix (App.~\ref{app:standard_metrics}).

The source-conditioned models we evaluate against include the source-conditioned FM, which is identical to MFM \citep{Atanackovic2024-ml} aside from the choice of a CLT-satisfying population encoder (see App.~\ref{app:baselines} for discussion). To deploy STC models without observing the test target, we predict the target embedding from the source embedding as in Sec.~\ref{subsec:source_target_conditioned}: standard semi-supervised entries use an expectation predictor fit by squared-error regression, while mix. rows use samples from an energy-score-trained latent sampler (App.~\ref{app:latent_predictors}). We also compare with ``oracle'' models which are STC models allowed to condition on the embedding of the true distribution. In the appendices, we baseline against KME and find it underperforms learned set encoders; we also evaluate robustness to encoder choice and find a mean-pooled Set Transformer is similar to our mean-pooled DeepSets but more computationally expensive. For targeted real-data ablations and latent-predictor investigations, we focus on the Energy generator because it is the strongest or near-strongest continuous-data generator across the main real-data benchmarks; the experiment appendices report the full benchmark metrics.

A central finding is that in real data applications we can often improve OOD generalization of supervised models by incorporating a wider set of data without explicit paired distributions (see the orphan marginals discussion in Sec.~\ref{subsec:source_target_conditioned}). Furthermore, modelling the conditional distribution of the target embedding can improve performance over modeling the conditional mean alone, narrowing the gap between oracle models and true semi-supervised models.


\subsection{Gaussian Transport}\label{subsec:semisup_gauss}

Returning to our Gaussian setting, we build a benchmark for semi-supervised learning. We define paired distributions via fixed transformations: for MVN, a simple shift $Y = X + b$; for GMM, a bimodal target with off-axis displacements. We compare three approaches: (i) a \emph{supervised} source-conditioned model trained on paired examples with means $\|\mu\|_\infty \leq 2.5$, (ii) the \emph{semi-supervised} any-to-any model from Sec.~\ref{subsec:unsup_gauss}, and (iii) an \emph{oracle} using the true target embedding. All models are evaluated across $\|\mu\|_\infty \leq 5$, testing generalization beyond the supervised training support.

Figure~\ref{fig:semisup_gaussian} shows squared MMD as a function of the source-mean norm $\|\mu\|_\infty$ for both settings. Within the training support, supervised and semi-supervised methods both achieve low MMD error.

\begin{wrapfigure}{r}{0.54\textwidth}
    \vspace{-2.5\intextsep}
    \centering
    \includegraphics[width=\linewidth]{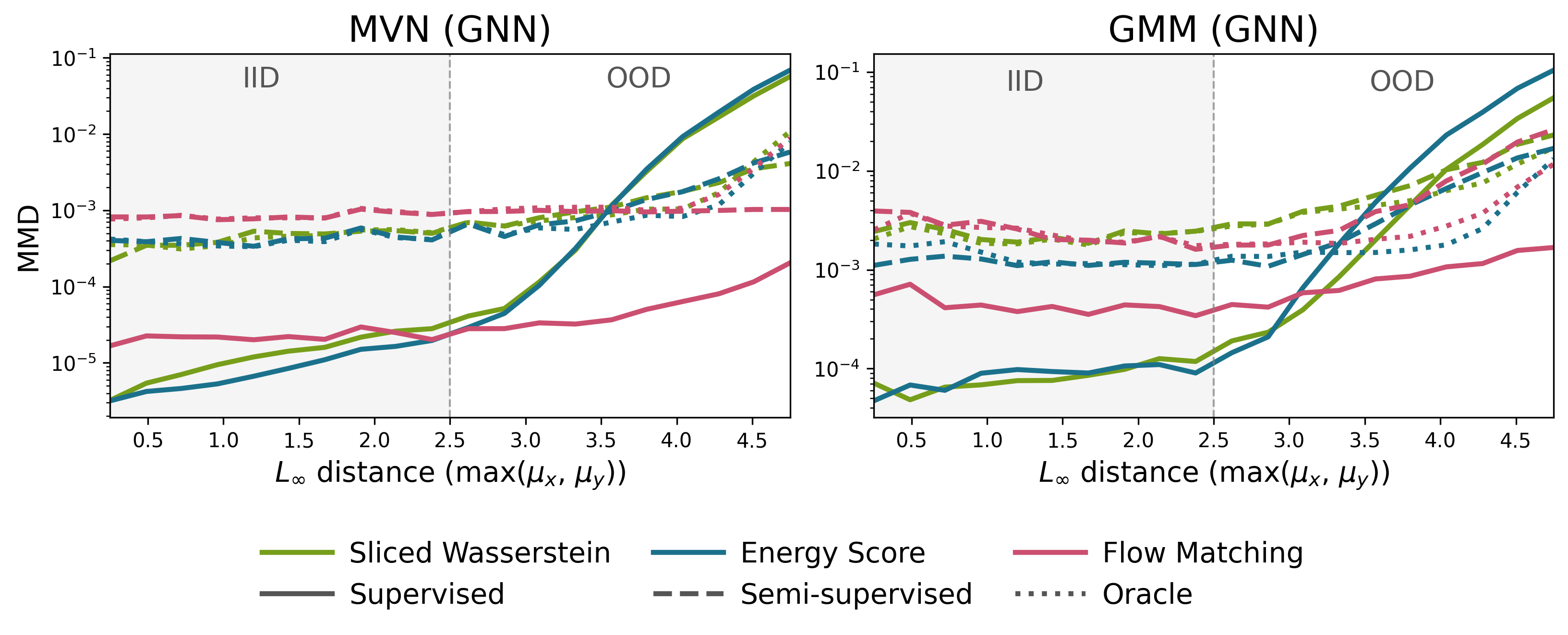}
    \caption{\textit{Semi-supervised generalization.} Supervised models are trained on paired source distributions with $\|\mu\|_\infty \leq 2.5$ (shaded) and evaluated across $\|\mu\|_\infty \leq 5$. 
    Semi-supervised DCT uses the any-to-any encoder with a predicted target embedding, while oracle STC uses the true target embedding. 
    Results show MMD distance ($\downarrow$) for multivariate normal (left) and Gaussian mixture (right) targets.}
    \label{fig:semisup_gaussian}\vspace{-1.0\intextsep}
\end{wrapfigure}

Beyond it, the semi-supervised Energy and SWD models maintain substantially lower MMD error, suggesting that the any-to-any objective learns distributional structure that enables extrapolation even when the latent predictor must generalize. The flow matching model is an exception here: it generalizes very well outside its support. The extremely strong performance of FM models here may be connected to the tightly coupled transport maps learned by FM (see App.~\ref{app:when_source_matters}), but this performance is not replicated in any of our real-world experiments. These results show that semi-supervised STC can improve extrapolation beyond the supervised support by leveraging unpaired distributions through the any-to-any objective.


\subsection{Drug perturbation prediction on organoids}\label{subsec:trellis}

We evaluate DCTs on a single-cell mass-cytometry drug screen from \citet{zapatero2023trellis}, which profiles patient-derived organoids (PDOs) from 10 colorectal cancer patients under 11 treatment regimens. Each experimental replicate consists of a matched control-treatment pair, yielding 927 population pairs. Cells are characterized by an abundance profile over 43 proteins.

\begin{table}[H]
    \centering
    \renewcommand{\arraystretch}{0.84}
    \setlength{\tabcolsep}{0.5pt}
    \begin{tabular}{@{}l ccc ccc@{}}
    \toprule
    & \multicolumn{3}{c}{\textbf{IID}} & \multicolumn{3}{c}{\textbf{OOD}} \\
    \cmidrule(lr){2-4} \cmidrule(lr){5-7}
    & Sup. & Semi & Oracle & Sup. & Semi & Oracle \\
    \midrule
    scGen & .0219{\tiny$\pm$.0057} & -- & -- & .0265{\tiny$\pm$.0022} & -- & -- \\
    CellOT & .0198{\tiny$\pm$.0046} & -- & -- & .0304{\tiny$\pm$.0033} & -- & -- \\
    CellFlow & .0212{\tiny$\pm$.0033} & -- & -- & .0338{\tiny$\pm$.0018} & -- & -- \\
    \cmidrule(lr){1-7}
    FM & {.0101}{\tiny$\pm$.0023} & .0106{\tiny$\pm$.0025} & .0030{\tiny$\pm$.0003} & .0227{\tiny$\pm$.0006} & {.0208}{\tiny$\pm$.0018} & .0053{\tiny$\pm$.0013} \\
    SWD & {.0064}{\tiny$\pm$.0022} & .0094{\tiny$\pm$.0018} & .0010{\tiny$\pm$.0000} & .0259{\tiny$\pm$.0008} & {.0197}{\tiny$\pm$.0022} & .0026{\tiny$\pm$.0009} \\
    Energy (MSE) & \textbf{.0054}{\tiny$\pm$.0019} & .0103{\tiny$\pm$.0032} & .0011{\tiny$\pm$.0002} & .0298{\tiny$\pm$.0019} & \textbf{.0191}{\tiny$\pm$.0049} & .0022{\tiny$\pm$.0002} \\
    Energy (mix.) & -- & \textbf{.0067}{\tiny$\pm$.0018} & -- & -- & .0330{\tiny$\pm$.0041} & -- \\
    \bottomrule
    \end{tabular}
    \vspace{1mm}
    \caption{\textit{Trellis PDO perturbation prediction.} MMD (mean $\pm$ SD, $\downarrow$) on IID replicate-holdout and OOD patient-holdout splits.
    }
    \label{tab:pdo}
\end{table}

A key challenge is predicting patient-specific treatment response: tumors from different patients exhibit distinct drug sensitivities. The task is to predict the post-treatment distribution given the control population and treatment identity. We evaluate on two splits: (i) held-out replicates within patients seen during training, and (ii) held-out patients, testing generalization.

We compare a source-conditioned baseline (including MFM) against our semi-supervised source-target-conditioned model. Both models condition on treatment identity. The Energy rows distinguish the expectation predictor trained by MSE from the energy-score-trained latent sampler, which evaluates the mixture of transported populations induced by sampled target latents. On unseen patients, semi-supervised models with the expectation predictor have lower MMD than their source-conditioned counterparts; the Energy mixture has higher error than the expectation predictor.

\subsection{Modeling clonal population dynamics}\label{subsec:lt}

\begin{wraptable}{r}{0.6\textwidth}
    \vspace{-2.5\intextsep}
    \centering
    \scriptsize
    \setlength{\tabcolsep}{2pt}
\begin{tabular}{lccc|c}
\toprule
& Supervised & Semi-sup. & Oracle & CT MSE \\
\midrule
LineageOT \citep{forrow2021lineageot} & 0.2865{\scriptsize$\pm$0.0067} & -- & -- & -- \\
\midrule
FM & 0.3085{\scriptsize$\pm$0.0056} & 0.2369{\scriptsize$\pm$0.0052} & 0.1362{\scriptsize$\pm$0.0029} & 0.0213 \\
SWD & 0.2536{\scriptsize$\pm$0.0056} & 0.2249{\scriptsize$\pm$0.0053} & 0.1158{\scriptsize$\pm$0.0029} & 0.0193 \\
Energy (MSE) & 0.2524{\scriptsize$\pm$0.0055} & 0.2289{\scriptsize$\pm$0.0055} & 0.1180{\scriptsize$\pm$0.0031} & 0.0187 \\
Energy (mix.) & -- & \textbf{0.1944}{\scriptsize$\pm$0.0047} & -- & -- \\
\bottomrule
\end{tabular}
    \caption{\textit{Clonal distribution forecasting on \citet{Weinreb2020-vh}.} MMD (mean $\pm$ SE, $\downarrow$) across 628 set pairs from 609 held-out clones. Oracle uses true target embeddings. LineageOT is an external paired-data baseline.
    CT MSE reports semi-supervised cell-type abundance error across 626 clone--time transitions.
    }
    \label{tab:lt}
\end{wraptable}

We next apply DCTs to learn clonal population dynamics using lineage-traced single-cell RNA-sequencing (scRNA-seq) data from \citet{Weinreb2020-vh}. A critical challenge in this dataset is sparsity: there are about $6\cdot 10^3$ clones measured in total, however only about $2\cdot10^3$ of these clones are profiled at multiple timepoints. While source-conditioned approaches can only be trained on clones observed at multiple timepoints, training DCTs with any-to-any pairings allow us to make use of all available marginals. Here, when training without coupling information we restrict source-target pairing to only allow an early timepoint clone as a source distribution and a late timepoint clone as a target distribution.

We compare source-conditioned (supervised) and source-target-conditioned (semi-supervised) approaches using identical architectures operating on the first 50 principal components of gene expression profiles (implementation details in App.~\ref{app:lt}). As shown in Table~\ref{tab:lt}, the source-target-conditioned models outperform the source-conditioned baselines. The Energy rows distinguish the expectation predictor trained by MSE from the mixture induced by the energy-score-trained latent sampler. Appendix~\ref{app:pairing} shows that adding orphan marginals together with auxiliary pair sampling improves over paired-only STC.



\subsection{Longitudinal TCR repertoire forecasting}\label{subsec:tcr}

We evaluate DCT on longitudinal T-cell receptor (TCR) repertoire sequencing from COVID-19 patients \citep{schultheiss2020next}. The dataset contains repertoires from 37 patients, of whom only 10 are profiled at multiple timepoints -- a sparsity pattern mirroring the orphan marginal regime in lineage tracing. We hold out future timepoints from three longitudinal patients; their baseline repertoires remain available to STC as orphan marginals. The task is to forecast a patient's repertoire at time $t{+}1$ given their repertoire at $t$.

\begin{wraptable}{l}{0.50\textwidth}
\vspace{-1.75\intextsep}
\centering
\small
\setlength{\tabcolsep}{3pt}
\begin{tabular}{lcc|c}
\toprule
& Supervised & Semi-sup. & Oracle \\
\midrule
ProGen2 & 0.0599{\scriptsize$\pm$.0211} & 0.0520{\scriptsize$\pm$.0372} & 0.0103{\scriptsize$\pm$.0041} \\
DFM & 0.0587{\scriptsize$\pm$.0216} & \textbf{0.0499}{\scriptsize$\pm$.0306} & 0.0110{\scriptsize$\pm$.0035} \\
\bottomrule
\end{tabular}
    \caption{\textit{TCR repertoire forecasting.} Values are MMD distance (mean $\pm$ standard error, $\downarrow$) over three patient means. Oracle uses true target embeddings.
    }
    \label{tab:tcr}
    \vspace{-1.25\intextsep}
\end{wraptable}

Each repertoire is an empirical distribution over TRB CDR3 amino acid sequences. We embed sequences using a pretrained ESM2 backbone \citep{lin2023evolutionary} with mean-pooled self-attention aggregation. We compare two discrete transport mechanisms: a ProGen2-based bridge \citep{Nijkamp2023-xq} and a discrete flow matching (DFM) bridge \citep{gat2024discreteflowmatching}, both trained in the same distribution-conditioned framework. The supervised baseline trains only on within-patient adjacent pairs $(t, t{+}1)$; the semi-supervised model trains on all ordered pairs of distinct available repertoires, incorporating single-timepoint patients as unpaired marginals.

The semi-supervised DFM and ProGen2 models both reduce MMD relative to their source-conditioned counterparts (Table~\ref{tab:tcr}). Appendix~\ref{app:tcr} reports the full distributional metrics.

\section{Conclusion}

We introduced \emph{distribution-conditioned transport}, a framework that couples distribution encoders with conditional transport models to learn transport maps which generalize to future source--target pairs drawn from a specified admissible-pair meta-distribution. By conditioning on learned embeddings of source and target distributions, DCT supports supervised, any-to-any, and semi-supervised settings, and can leverage orphan marginals observed at only one timepoint. We showed that DCT formalizes and generalizes existing approaches under a unified framework, leading to new models with empirical performance gains on synthetic benchmarks and four real world problems in biology.

\paragraph{Limitations and extensions}
DCTs require that the distributions being trained on admit a useful finite-dimensional population embedding and that the chosen conditional transport class and the chosen architecture have the expressivity to capture the relevant population embeddings and ensure transport. This is clearly true enough in our applications to lead to performance gains, but developing a real theory for when this is true and when it is not is an area for future work.
Source-target-conditioned models may have different scaling properties and require more training. In the semi-supervised setting, performance also depends on fitting good latent predictors or latent distributional samplers. It may be fruitful to explore more complex latent prediction/sampling schemes for STC models.

\begingroup
\interlinepenalty=10000 
\bibliography{ref,paperpile}
\bibliographystyle{unsrtnat}
\endgroup

\newpage
\appendix
\onecolumn

\section{Theory}
\label{app:dct_theory}

\subsection{Why can we train on minibatches?}

This appendix formalizes the minibatch claim used throughout the main text:
distribution encoders factor through empirical measures, and empirical-process CLTs yield plug-in CLTs for the shared-minibatch objectives used in DCT training.
If a population-level DCT objective conditions on $z^\star(P)=\phi(P)$, then in training we replace $P$ by a minibatch empirical measure $\widehat P_m$ and condition on $\phi(\widehat P_m)$.
The results below explain why this replacement contributes the usual empirical-process scale error, including when the same minibatch is reused both to compute the embedding and to evaluate the loss.
They should be read conditionally on the current encoder and transport parameters during training.
At a fixed SGD step, $\phi$ is the current encoder functional and $\mathcal L_\theta$ is the current transport loss; fresh minibatches give a stochastic approximation to the corresponding population loss.
This fixed-pair calculation does not by itself prove optimization convergence, consistency of the learned encoder across finitely many training distributions, or pair-level generalization over $Q_{\mathrm{joint}}$.

As in the methods section, the theory is conditional on a chosen finite-dimensional population representation
\[
    z^\star(P)=\phi(P)\in\mathbb R^d.
\]
This is a modeling assumption about the representation used by the downstream transport model, not a claim that every distribution admits a unique or lossless finite summary. In practice, $\phi$ may be specified directly or may be the population version of the current trained encoder. The results below assume that this population functional exists and is Hadamard differentiable at the distributions under consideration, and ask what happens when $\phi(P)$ is unavailable and is replaced by $\phi(\widehat P_m)$ computed from a minibatch.
We refer to \citet{fishman2025generative,van1996weak} for broader context.

\subsubsection{Empirical-measure factorization}

Let $(\mathcal X,\mathcal B)$ be a standard Borel space and let $\mathcal P(\mathcal X)$ denote probability measures on it.
This includes the Euclidean cell-state spaces and discrete sequence spaces used in our experiments, and ensures that equality of empirical measures is equality of the underlying finite multisets.
All encoder maps below are assumed measurable with respect to the product $\sigma$-algebra on $\mathcal X^m$ and the Borel $\sigma$-algebra on $\mathbb R^d$.
For $m\in\mathbb N$ and $S_m=(x_1,\dots,x_m)\in\mathcal X^m$, write the empirical measure
\[
\widehat P_m \;:=\; \frac1m\sum_{j=1}^m \delta_{x_j}.
\]
These factorization statements mirror the empirical-measure sufficiency/maximal-invariance viewpoint in \citet[Thm.\ ``Empirical measure is sufficient, minimal, and a maximal invariant'']{fishman2025generative}; we include short proofs for completeness and to match DCT notation.

\begin{definition}[Permutation and proportional invariance]
Fix $m\in\mathbb N$. A map $\mathcal E_m:\mathcal X^m\to\mathbb R^d$ is \emph{permutation invariant} if
\[
\mathcal E_m(x_1,\dots,x_m)=\mathcal E_m(x_{\pi(1)},\dots,x_{\pi(m)})
\qquad\text{for all permutations }\pi\in\mathfrak S_m.
\]
A family $(\mathcal E_m)_{m\ge 1}$ is \emph{proportionally invariant} if for every integer $k\ge 1$ and every $(x_1,\dots,x_m)\in\mathcal X^m$,
\[
\mathcal E_m(x_1,\dots,x_m)
\;=\;
\mathcal E_{km}(\underbrace{x_1,\dots,x_m,\dots,x_1,\dots,x_m}_{k\ \text{copies}}).
\]
\end{definition}

\begin{proposition}[Permutation invariance $\Rightarrow$ factorization]\label{prop:factorization_fixed_m}
Fix $m\in\mathbb N$ and let $\mathcal E_m:\mathcal X^m\to\mathbb R^d$ be permutation invariant. Then there exists a measurable map $\phi_m$ defined on the set of empirical measures $\{\widehat P_m:\,S_m\in\mathcal X^m\}$, equipped with the quotient $\sigma$-algebra induced by $\widehat P_m$, such that
\[
\mathcal E_m(S_m)=\phi_m(\widehat P_m)\qquad\text{for all }S_m\in\mathcal X^m.
\]
\end{proposition}
\begin{proof}
Let $\widehat P_m:\mathcal X^m\to \mathcal P(\mathcal X)$ be the empirical-measure map and equip its image $\{\widehat P_m(S_m):S_m\in\mathcal X^m\}$ with the quotient $\sigma$-algebra induced by $\widehat P_m$.
If $\widehat P_m(S_m)=\widehat P_m(S_m')$, then $S_m'$ is a permutation of $S_m$ (as a multiset with multiplicity), so permutation invariance implies $\mathcal E_m(S_m)=\mathcal E_m(S_m')$; hence $\mathcal E_m$ is constant on fibers of $\widehat P_m$.
Therefore there exists a well-defined measurable $\phi_m$ on the image of $\widehat P_m$ such that $\mathcal E_m=\phi_m\circ \widehat P_m$, i.e.\ $\mathcal E_m(S_m)=\phi_m(\widehat P_m)$.
\end{proof}

\begin{proposition}[Permutation + proportional invariance $\Rightarrow$ a single functional]\label{prop:factorization_all_m}
Let $(\mathcal E_m)_{m\ge 1}$ be permutation and proportionally invariant. Let
\[
\mathcal P_{\mathrm{emp}}
:=\bigcup_{m\ge 1}\{\widehat P_m:\,S_m\in\mathcal X^m\}
\]
be equipped with the quotient $\sigma$-algebra induced by the map $(m,S_m)\mapsto \widehat P_m(S_m)$ from the disjoint union $\bigsqcup_{m\ge1}\mathcal X^m$.
Then there exists a single measurable map $\phi:\mathcal P_{\mathrm{emp}}\to\mathbb R^d$ such that
\[
\mathcal E_m(S_m)=\phi(\widehat P_m)\qquad\text{for all }m\ge 1,\ S_m\in\mathcal X^m.
\]
\end{proposition}
\begin{proof}
By Prop.~\ref{prop:factorization_fixed_m}, for each $m$ there exists $\phi_m$ with $\mathcal E_m(S_m)=\phi_m(\widehat P_m)$.
Define $\phi$ on empirical measures by $\phi(\nu):=\phi_m(\nu)$ whenever $\nu$ is an $m$-empirical measure.
To see this is well-defined, suppose $\nu$ can be represented at two sample sizes $m$ and $m'$, i.e.\ there exist $S_m\in\mathcal X^m$ and $S_{m'}\in\mathcal X^{m'}$ such that $\widehat P_m(S_m)=\widehat P_{m'}(S_{m'})=\nu$.
Let $L:=\mathrm{lcm}(m,m')$, and write $k:=L/m$ and $k':=L/m'$.
Form the duplicated samples
\[
S_m^{(k)}=(\underbrace{S_m,\dots,S_m}_{k\ \text{copies}})\in\mathcal X^{L},
\qquad
S_{m'}^{(k')}=(\underbrace{S_{m'},\dots,S_{m'}}_{k'\ \text{copies}})\in\mathcal X^{L}.
\]
Then $\widehat P_L(S_m^{(k)})=\nu=\widehat P_L(S_{m'}^{(k')})$, so $S_m^{(k)}$ and $S_{m'}^{(k')}$ are permutations of each other.
By permutation invariance of $\mathcal E_L$ and proportional invariance of the family $(\mathcal E_m)_{m\ge 1}$,
\[
\mathcal E_m(S_m)=\mathcal E_L(S_m^{(k)})=\mathcal E_L(S_{m'}^{(k')})=\mathcal E_{m'}(S_{m'}),
\]
so $\phi_m(\nu)=\phi_{m'}(\nu)$ and $\phi$ is well defined.
Measurability follows from the quotient $\sigma$-algebra because $\phi\circ\widehat P_m=\mathcal E_m$ is measurable for every $m$.
\end{proof}

\paragraph{How this applies to the standard encoders.}
The standard continuous-data encoder used in the experiments is built from shared samplewise neural maps, mean-pooling operations, and a final finite-dimensional projection (Appendix~\ref{app:standard_encoder}), following the DeepSets/GDE template \citep{Zaheer2017-tu,fishman2025generative}.
Mean pooling makes the architecture permutation invariant and proportionally invariant, so Props.~\ref{prop:factorization_fixed_m}--\ref{prop:factorization_all_m} apply.
The mean-pooled Set Transformer variants \citep{lee2019set}, including the TCR encoder after its per-sequence ESM2 feature map, are also permutation invariant and proportionally invariant when no positional order is supplied: self-attention is permutation equivariant, uniform duplication only duplicates identical keys/values and leaves the pooled output unchanged, and the final mean pooling removes ordering.
For the simplest DeepSets form
\[
    \phi_\eta(P)=r_\eta\!\left(\int h_\eta(x)\,dP(x)\right),
\]
if $h_\eta\in L^2(P)$ and $r_\eta$ is differentiable at $\mu_P=\int h_\eta\,dP$, then the ordinary multivariate delta method gives
\[
    \sqrt m\{\phi_\eta(P_m)-\phi_\eta(P)\}
    =
    \frac1{\sqrt m}\sum_{i=1}^m
    Dr_\eta(\mu_P)\{h_\eta(X_i)-\mu_P\}+o_{\mathbb P}(1).
\]
The mean-pooled update blocks used in our GNN-style encoder and Set Transformer variants are finite compositions of samplewise maps, empirical means, attention/softmax maps, and finite-dimensional projections, so the same delta-method argument applies under the corresponding local differentiability, finite-moment, and nondegeneracy assumptions.
If the optional latent normalization is used, the same statement requires the pre-normalized population latent to be nonzero.
SELU activations are locally Lipschitz and differentiable except at zero; the ordinary Gaussian delta-method statement applies when the preactivation variables or finite-dimensional states at which derivatives are evaluated avoid this kink (for samplewise terms, with $P$-probability one). If a population state lies exactly on a kink, the Gaussian limit may need a nonsmooth local analysis, but the locally Lipschitz structure still gives the same $O_{\mathbb P}(m^{-1/2})$ plug-in scale.

\subsubsection{Encoder CLT via the functional delta method}

Let $P\in\mathcal P(\mathcal X)$ and let $X_1,X_2,\dots\overset{\mathrm{iid}}{\sim}P$ with empirical measure $P_m=\frac1m\sum_{i=1}^m\delta_{X_i}$.
Assume the encoder is distributionally invariant, so that $\mathcal E(S_m)=\phi(P_m)$ for a functional $\phi:\mathcal P(\mathcal X)\to\mathbb R^d$.

\begin{theorem}[Encoder CLT]\label{thm:encoder_clt}
Suppose the empirical process satisfies $\sqrt m(P_m-P)\Rightarrow \mathbb G_P$ in a function space on which $\phi$ is defined, and suppose $\phi:\mathcal P(\mathcal X)\to\mathbb R^d$ is Hadamard differentiable at $P$ tangentially to a subspace containing this empirical-process limit. Denote the derivative by $D\phi_P$.
Then
\[
\sqrt m\,\bigl(\phi(P_m)-\phi(P)\bigr)
\;\xRightarrow{d}\;
D\phi_P(\mathbb G_P),
\]
where $\mathbb G_P$ is the $P$-Brownian bridge.
In particular, if $D\phi_P(\mathbb G_P)$ is a mean-zero Gaussian in $\mathbb R^d$ with covariance $\Sigma_P$, then
\[
\sqrt m\,\bigl(\mathcal E(S_m)-\phi(P)\bigr)\xRightarrow{d}\mathcal N(0,\Sigma_P),
\]
which is Eqn.~\ref{eq:clt}.
\end{theorem}
\begin{proof}
Apply the functional delta method \citet[Thm.~3.9.4]{van1996weak} to the empirical-process CLT. For example, one may take the function space to be $\ell^\infty(\mathcal A)$ for a $P$-Donsker class $\mathcal A$ rich enough to define the encoder functional.
See \citet{fishman2025generative} for this argument in the GDE setting.
\end{proof}

\subsubsection{Shared-minibatch losses for conditioned models}

We now apply the encoder CLT to the losses used by DCT.
Fix an admissible pair $(P_u,P_v)$ and write
\[
z_u^\star:=\phi(P_u),
\qquad
z_v^\star:=\phi(P_v).
\]
Let $\widehat P_{u,m}$ and $\widehat P_{v,m'}$ be empirical measures computed from source and target minibatches.
They may be sampled independently, or, when $m=m'$, from paired observations $(X_{ui},X_{vi})_{i=1}^m\overset{\mathrm{iid}}{\sim}\Pi_{uv}$ with marginals $P_u$ and $P_v$.
The product coupling $\Pi_{uv}=P_u\otimes P_v$ recovers independent minibatches with equal sizes.

At one training step, hold the current encoder and transport parameters fixed.
For notational simplicity we write the current encoder functional as $\phi$ and the current transport loss as $\mathcal L_\theta$.
Let
\[
\mathcal L_\theta^{\mathrm{sc}}(\mu,\nu,z)
\quad\text{and}\quad
\mathcal L_\theta^{\mathrm{stc}}(\mu,\nu,z,z')
\]
denote the source-conditioned and source--target-conditioned population/set-loss functionals.
When $\mu,\nu$ are empirical measures these are the empirical losses used in training; when $\mu=P_u$ and $\nu=P_v$ they are the corresponding population losses.
The empirical objectives are therefore
\[
\widehat{\mathcal L}_{\theta,m,m'}^{\mathrm{sc}}
=
\mathcal L_\theta^{\mathrm{sc}}\bigl(
\widehat P_{u,m},\widehat P_{v,m'},\phi(\widehat P_{u,m})
\bigr),
\qquad
\widehat{\mathcal L}_{\theta,m,m'}^{\mathrm{stc}}
=
\mathcal L_\theta^{\mathrm{stc}}\bigl(
\widehat P_{u,m},\widehat P_{v,m'},\phi(\widehat P_{u,m}),\phi(\widehat P_{v,m'})
\bigr).
\]
They approximate the population objectives
\[
\mathcal L_\theta^{\mathrm{sc}}(P_u,P_v,z_u^\star),
\qquad
\mathcal L_\theta^{\mathrm{stc}}(P_u,P_v,z_u^\star,z_v^\star).
\]
For Energy/MMD losses, this marginal-measure notation is literal. For SWD, we condition on the sampled projection directions and use the population loss for those same directions. With a fixed number of resampled directions, projection Monte Carlo error does not vanish as the cell minibatch sizes grow; the theorem below concerns cell-sampling and embedding error conditional on those directions.
For flow-matching or discrete flow-matching objectives, the loss may also use an empirical source--target coupling or auxiliary random variables.
Those variables should be included as part of the ordinary empirical loss.
Under the corresponding joint empirical-process CLT, they add ordinary loss-fluctuation terms; the embedding plug-in terms still enter through $\phi(\widehat P_{u,m})$ and $\phi(\widehat P_{v,m'})$.

\begin{theorem}[Shared-minibatch loss CLT for DCT]\label{thm:plugin_loss_clts}
Assume $m/m'\to\rho\in(0,\infty)$ and that the joint empirical process satisfies
\[
\left(
\sqrt m(\widehat P_{u,m}-P_u),
\sqrt m(\widehat P_{v,m'}-P_v)
\right)
\Rightarrow
(\mathbb G_u,\mathbb G_{v,\rho})
\]
in a product function space on which the functionals below are defined.
For independent minibatches, $\mathbb G_{v,\rho}=\sqrt\rho\,\mathbb G_v$ and $\mathbb G_u,\mathbb G_v$ are independent Brownian bridges.
For paired/coupled minibatches, $m=m'$ and $(\mathbb G_u,\mathbb G_{v,\rho})$ is the joint Gaussian empirical-process limit induced by the coupling $\Pi_{uv}$.

Assume $\phi$ is Hadamard differentiable at $P_u$ and $P_v$ tangentially to the empirical-process directions, and that the composite maps
\[
(\mu,\nu)\mapsto \mathcal L_\theta^{\mathrm{sc}}(\mu,\nu,\phi(\mu)),
\qquad
(\mu,\nu)\mapsto \mathcal L_\theta^{\mathrm{stc}}(\mu,\nu,\phi(\mu),\phi(\nu))
\]
are Hadamard differentiable at $(P_u,P_v)$.
Then the shared-minibatch objectives satisfy
\[
\sqrt m\left\{
\widehat{\mathcal L}_{\theta,m,m'}^{\mathrm{sc}}
-
\mathcal L_\theta^{\mathrm{sc}}(P_u,P_v,z_u^\star)
\right\}
\Rightarrow
D\!\left[\mathcal L_\theta^{\mathrm{sc}}(\cdot,\cdot,\phi(\cdot))\right]_{(P_u,P_v)}
(\mathbb G_u,\mathbb G_{v,\rho})
\]
and
\[
\sqrt m\left\{
\widehat{\mathcal L}_{\theta,m,m'}^{\mathrm{stc}}
-
\mathcal L_\theta^{\mathrm{stc}}(P_u,P_v,z_u^\star,z_v^\star)
\right\}
\Rightarrow
D\!\left[\mathcal L_\theta^{\mathrm{stc}}(\cdot,\cdot,\phi(\cdot),\phi(\cdot))\right]_{(P_u,P_v)}
(\mathbb G_u,\mathbb G_{v,\rho}).
\]
Equivalently, if the outer loss is differentiable in its measure and embedding arguments, the chain rule shows exactly where the terms come from.
For directions $(h,k)$, the source-conditioned derivative is
\[
D\!\left[\mathcal L_\theta^{\mathrm{sc}}(\cdot,\cdot,\phi(\cdot))\right]_{(P_u,P_v)}[h,k]
=
D_\mu\mathcal L_\theta^{\mathrm{sc}}[h]
+
D_\nu\mathcal L_\theta^{\mathrm{sc}}[k]
+
D_z\mathcal L_\theta^{\mathrm{sc}}\!\left[D\phi_{P_u}[h]\right],
\]
and the source--target-conditioned derivative is
\[
\begin{aligned}
D\!\left[\mathcal L_\theta^{\mathrm{stc}}(\cdot,\cdot,\phi(\cdot),\phi(\cdot))\right]_{(P_u,P_v)}[h,k]
=\;&
D_\mu\mathcal L_\theta^{\mathrm{stc}}[h]
+
D_\nu\mathcal L_\theta^{\mathrm{stc}}[k] \\
&+
D_z\mathcal L_\theta^{\mathrm{stc}}\!\left[D\phi_{P_u}[h]\right]
+
D_{z'}\mathcal L_\theta^{\mathrm{stc}}\!\left[D\phi_{P_v}[k]\right].
\end{aligned}
\]
The first two terms are the ordinary empirical set-loss fluctuation, while the final term(s) are the additional plug-in fluctuation from replacing population embeddings by empirical embeddings.
\end{theorem}
\begin{proof}
The joint empirical-process CLT gives the weak limit of
$(\widehat P_{u,m},\widehat P_{v,m'})$ around $(P_u,P_v)$ at rate $\sqrt m$.
Since the two displayed loss maps are Hadamard differentiable at $(P_u,P_v)$, the functional delta method \citet[Thm.~3.9.4]{van1996weak} gives the two stated limits.
The derivative decompositions are just the chain rule applied to the empirical measures and to the two embedding arguments.
\end{proof}

This is the calculation used in the main text.
If the admissible family admits a useful finite-dimensional population representation and the encoder and loss are smooth enough, then empirical embeddings perturb the population objective only at the usual empirical-process scale.
Reusing the same source and target minibatches inside the encoder and inside the loss changes the covariance through the joint empirical process, but it does not change the $m^{-1/2}$ scale.
These conclusions require differentiability of the law-level composite loss, not just differentiability of a finite-sample sorting or neural-network operation. In particular, unsquared SWD at atomic populations can have a non-Gaussian limit or a different rate. When experiments resample from a finite observed cell or sequence pool, the theorem applies conditionally only if its empirical-process and loss-regularity assumptions hold at those empirical populations; an additional biological-population sampling layer is separate from this calculation.

\begin{corollary}[Mean consistency and plug-in bias]\label{cor:plugin_bias}
Under the assumptions of Thm.~\ref{thm:plugin_loss_clts}, suppose the composite maps
$\mathcal L_\theta^{\mathrm{sc}}(\cdot,\cdot,\phi(\cdot))$ and
$\mathcal L_\theta^{\mathrm{stc}}(\cdot,\cdot,\phi(\cdot),\phi(\cdot))$ are locally Lipschitz near $(P_u,P_v)$ with respect to a norm controlling the empirical-process fluctuations, and suppose the shared-minibatch losses are uniformly integrable.
Then
\[
\mathbb E\left|
\widehat{\mathcal L}_{\theta,m,m'}^{\mathrm{sc}}
-
\mathcal L_\theta^{\mathrm{sc}}(P_u,P_v,z_u^\star)
\right|\to 0
\]
and
\[
\mathbb E\left|
\widehat{\mathcal L}_{\theta,m,m'}^{\mathrm{stc}}
-
\mathcal L_\theta^{\mathrm{stc}}(P_u,P_v,z_u^\star,z_v^\star)
\right|\to 0.
\]
For a second-order bias rate, additionally suppose each composite map has a first-order Taylor expansion with remainder bounded by $C(\|\widehat P_{u,m}-P_u\|^2+\|\widehat P_{v,m'}-P_v\|^2)$ on the support of the empirical measures. Suppose the expectations of these squared norms are $O(m^{-1})$ and $O({m'}^{-1})$, and the expected linear terms have these respective orders. Then the shared-minibatch objective bias is $O(m^{-1}+{m'}^{-1})$.
\end{corollary}
\begin{proof}
The CLT in Thm.~\ref{thm:plugin_loss_clts} gives $O_{\mathbb P}(m^{-1/2})$ fluctuations of the shared-minibatch objective around its population counterpart.
Local Lipschitz continuity and uniform integrability upgrade convergence in probability to convergence in mean.
The stated bias rate follows by taking expectations in the assumed Taylor expansion: the linear terms have the stated orders, and the quadratic remainder is controlled by the assumed second-moment bounds.
\end{proof}

\section{What explains the oracle gap?}\label{app:oracle}

The oracle STC model conditions on the true target embedding, while the deployed semi-supervised model predicts this embedding from the source. Here we isolate the oracle gap with a target-aided best-of-$k$ diagnostic: we sample plausible target latents, evaluate the resulting transported populations against the held-out target, and report the best sampled latent. This diagnostic is not deployable because it uses the held-out target for selection, but it asks whether the learned latent distribution contains target states that move the model toward oracle STC.

\begin{table}[H]
\centering
\small
\caption{\textit{Trellis oracle-gap diagnostic.} The Energy-generator STC model is evaluated with the standard predicted target embedding, a target-aided best-of-$32$ diagnostic, and the oracle target embedding. Values are raw MMD-RBF (mean $\pm$ standard deviation, $\downarrow$).}
\label{tab:trellis_oracle_gap_headline}
\begin{tabular}{lcc}
\toprule
\textbf{Target-latent evaluation} & \textbf{IID} & \textbf{OOD} \\
\midrule
Expectation predictor (MSE) & $0.0103 \pm 0.0032$ & $0.0191 \pm 0.0049$ \\
Target-aided best-of-$32$ & $0.0037 \pm 0.0009$ & $0.0225 \pm 0.0024$ \\
Oracle STC & $0.0011 \pm 0.0002$ & $0.0022 \pm 0.0002$ \\
\bottomrule
\end{tabular}
\end{table}

The target-aided rows show that Trellis performance is sensitive to the target condition, but best-of-$32$ is not uniformly better than the deployed predictor. The complete Trellis latent-predictor results appear in Appendix~\ref{app:trellis_latent_predictors}. Table~\ref{tab:lt_oracle_gap_headline} gives the corresponding lineage diagnostic.

\begin{table}[H]
\centering
\small
\caption{\textit{Lineage-tracing oracle-gap diagnostic.} The Energy-generator STC model is evaluated with two deployed predictors and two target-aided conditions. Every row uses a fixed 100-cell generated endpoint. Values are energy distance over 628 empirical pairs from 609 held-out clones (mean $\pm$ standard error, $\downarrow$).}
\label{tab:lt_oracle_gap_headline}
\begin{tabular}{lc}
\toprule
\textbf{Target-latent evaluation} & \textbf{Energy} \\
\midrule
Expectation predictor (MSE) & $4.433 \pm 0.110$ \\
Mixture predictor & $3.614 \pm 0.102$ \\
Target-aided best-of-$32$ & $3.380 \pm 0.074$ \\
Oracle STC & $2.332 \pm 0.046$ \\
\bottomrule
\end{tabular}
\end{table}

In lineage tracing, the deployable mixture improves on the expectation predictor, and target-aided best-of-$32$ improves further. The true target embedding reduces energy distance to $2.332$, localizing substantial error to target-embedding prediction.

\section{Which coupling is selected?}
\label{app:when_source_matters}

The plug-in CLTs above quantify finite-minibatch approximation for regular objectives that compare \emph{distributions}.
Here we ask a different empirical question: when a deterministic model is trained to match endpoint marginals, which pointwise pairing does the learned transport map select?
Endpoint marginals are compatible with many couplings, so the selected coupling is a property of the generator, objective, and training procedure rather than of DCT conditioning alone.
We therefore diagnose the selected coupling directly.

\paragraph{A linear matrix-pairing diagnostic.}
For deterministic maps, we fit an affine map from source samples to transported samples:
\[
    \hat x_{\mathrm{tgt},i}
    =
    A x_{\mathrm{src},i} + b + \varepsilon_i .
\]
The intercept absorbs the bulk source--target mean shift, so the fitted matrix $A$ describes how source coordinates are paired with target coordinates.
In the two-dimensional MVN benchmark (App.~\ref{app:gaussian}), a fitted matrix $A\approx I$ indicates an approximately translational map, while a coordinate-swap map has an off-diagonal matrix,
\[
P_{\mathrm{swap}}=
\begin{pmatrix}
0 & 1\\
1 & 0
\end{pmatrix}.
\]

\begin{table}[H]
\centering
\small
\caption{\textit{Linear matrix-pairing diagnostic.} On the MVN any-to-any benchmark ($K=10{,}000$, 200 epochs), we fit $\hat x_{\mathrm{tgt}}=Ax_{\mathrm{src}}+b$ from source samples to transported samples. All three deterministic generators have strong affine fits: FM has an approximately identity matrix, while Energy and SWD have off-diagonal matrices. These fits describe source--output association, not optimality of transport cost.}
\label{tab:linear_coupling_mvn}
\begin{tabular}{lll}
\toprule
\textbf{Generator} & \textbf{Dominant fitted matrix structure} & \textbf{Interpretation} \\
\midrule
FM & on diagonal ($A\approx I$), $R^2>0.95$ & approximately translational \\
Energy & off diagonal ($A\approx P_{\mathrm{swap}}$), $R^2>0.95$ & coherent off-diagonal coupling \\
SWD & off diagonal ($A\approx P_{\mathrm{swap}}$), $R^2>0.95$ & coherent off-diagonal coupling \\
\bottomrule
\end{tabular}
\end{table}

Thus the fitted matrices show that all three deterministic generators learn strong, structured couplings in this benchmark; they differ in the geometry of the coupling they select.
If a particular coupling is desired, it must be supplied by the transport objective or by an additional design choice.
More generally, for a desired affine coupling $C_0(x)=A_0x+b_0$, one can add
\[
    \lambda\,\mathbb E\bigl\|\mathcal T(X\mid z_u,z_v)-C_0(X)\bigr\|_2^2
\]
to the distributional training objective.
The runs in Table~\ref{tab:coupling_regularization_mvn} use $C_0(x)=x$, penalizing squared source--output displacement.
The annealed runs have an on-diagonal fitted matrix.

\begin{table}[H]
\centering
\small
\caption{\textit{Coupling regularization.} On the MVN any-to-any benchmark ($K=10{,}000$, 200 epochs), the regularized Energy and SWD rows use the identity target $C_0(x)=x$ to select an on-diagonal coupling. Energy distance ($\downarrow$, biased empirical estimator) measures marginal distribution quality.}
\label{tab:coupling_regularization_mvn}
\begin{tabular}{lllc}
\toprule
\textbf{Generator} & \textbf{Regularization} & \textbf{Selected matrix structure} & \textbf{Energy} \\
\midrule
FM & None & on diagonal & 0.019 \\
SWD & None & off diagonal & 0.022 \\
SWD & $\lambda=0.1$ & on diagonal & 0.028 \\
SWD & $\lambda=0.1 \to 0$ & on diagonal & 0.020 \\
Energy & None & off diagonal & 0.021 \\
Energy & $\lambda=0.1$ & on diagonal & 0.042 \\
Energy & $\lambda=0.1 \to 0$ & on diagonal & 0.018 \\
\bottomrule
\end{tabular}
\end{table}

These diagnostics describe the affine structure of the learned source--target pairings in this benchmark.
The selected pairing is not fixed by DCT itself: different generators and objectives can select different coherent couplings, and an explicit regularizer can steer the model toward a desired coupling when one is known.

\section{How is the latent space structured?}
\label{app:latent_structure}

Following the manifold/trajectory viewpoint of \citet{fishman2025generative}, we use the learned embedding space as a coordinate system for a family of related distributions and probe it via \emph{latent interpolants}.
Fix a pair $(u,v)$ with embeddings $(z_u,z_v)$.
We linearly interpolate between them, $z(t)=(1-t)z_u+t z_v$, choose a grid $0=t_0<\cdots<t_K=1$, and set $z_k=z(t_k)$.
Starting from samples $x_0\sim P_u$, we form a discrete trajectory by repeatedly applying the source--target conditioned map between successive latent codes, i.e.\ $x_{k+1}=\mathcal T(x_k\mid z_k,z_{k+1})$.
The empirical laws of $(x_k)_{k=0}^K$ define a path in distribution space.

Fig.~\ref{fig:mvn_latent_trajectories} shows the resulting STC trajectory on the MVN any-to-any benchmark (App.~\ref{app:gaussian}) alongside the closed-form optimal transport (OT) displacement interpolation between the same endpoint Gaussians.
In this example, the learned trajectory closely tracks the OT path, suggesting that linear structure in the latent space induces a meaningful geometry in distribution space (at least within this controlled family). This example illustrates how latent spaces learned by distribution encoders for any-to-any transport can share the geometric properties of GDEs.

\begin{figure}[t]
\centering
\includegraphics[width=0.3\textwidth]{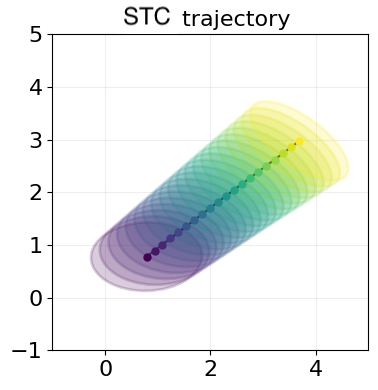}
\includegraphics[width=0.3\textwidth]{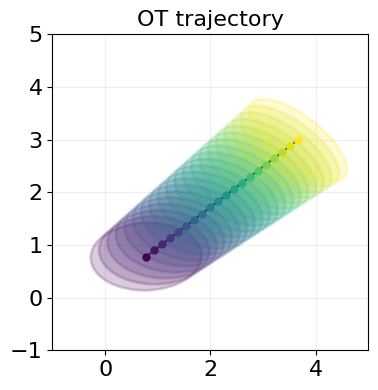}
\caption{%
\textit{Latent-space interpolants on the MVN any-to-any benchmark.}
\textbf{Left:} STC trajectory obtained by sequentially pushing samples using intermediate latent codes along a linear interpolant between endpoint embeddings.
\textbf{Right:} OT displacement interpolation (computed in closed form for Gaussians).
}
\label{fig:mvn_latent_trajectories}
\end{figure}

\section{Baseline Comparisons}\label{app:baselines}

\subsection{Unsupervised Transport}

In our unsupervised transport experiments, we compare source-target-condition models in the DCT framework to $K$-to-$K$ baselines, which condition on one-hot embeddings of distributions. This approach is closely related to that of MMSI \citep{albergo2023multimarginal}, but distinct in a subtle way.

MMSI models interpolant time on the simplex, where each point in the simplex corresponds to an interpolant taking a convex combination of a sample from each of the $K$ distributions proportional to their simplicial weight. Our $K$-to-$K$ baselines (when using flow matching generators) can be viewed as MMSI models which take ``edge paths'' on the simplex, where we flow directly between two corners without any weight on other distributions.

We choose our one-hot $K$-to-$K$ as a baseline, because it allows for direct comparison across a broader family of transport mechanisms rather than flow matching alone, and anecdotally we have observed more stable performance with MMSI models restricted to simplicial edge paths.


\subsection{Supervised Transport}

In our supervised transport experiments, we evaluate source-conditioned models. These models are closely related to Meta Flow Matching models \citep{Atanackovic2024-ml}. When using flow matching, source-conditioned DCTs differ from MFM only in the choice of distribution encoder. Rather than a distributionally invariant encoder, MFM instead uses a heuristic ``population encoder'' which involves message passing over a k-nearest-neighbor graph on the input population.

We choose to baseline against source-conditioned models with distributionally invariant encoders as these provide a CLT and robustness to variation in sample set size. Furthermore, using source-conditioned models rather than MFM alone allows us to again directly compare between source-conditioned and source-target-conditioned models across a broad range of transport mechanisms.

\section{Standard distribution-conditioned transport maps}\label{app:standard_generators}

Several experiments in this paper reuse the same generator families and only vary the dataset, encoder, or model capacity. This appendix defines the three \emph{continuous-data} generators used throughout: sliced Wasserstein regression (\textbf{SWD}) \citep{kolouri2019generalized}, energy/MMD regression (\textbf{Energy}) \citep{gneiting2007strictly,sutherland2016generative}, and flow matching (\textbf{FM}) \citep{liu2022rectified_flow,lipman2023flow_matching,albergo2023stochastic_interpolants}. Discrete sequence generators for the TCR experiments are described separately in Appendix~\ref{app:tcr_generators}. Where an experiment uses different hyperparameters (e.g.\ number of projections or solver tolerances), it is stated explicitly in that experiment's appendix.

\paragraph{Terminology.}
Throughout, \textbf{DCT} refers to the overall \emph{distribution-conditioned transport} framework of coupling a distribution encoder with a conditional generator. Within DCT we use \textbf{source-conditioned} (\textbf{SC}) to mean generators conditioned only on the source embedding $z_{\text{src}}$, and \textbf{source--target-conditioned} (\textbf{STC}) to mean generators conditioned on both $(z_{\text{src}},z_{\text{tgt}})$; STC models are also referred to as \emph{any-to-any} models in the main text. Unless otherwise stated, the generator families below apply to both SC and STC, with SC obtained by dropping $z_{\text{tgt}}$.

\subsection{Notation}

Let $S^{\text{src}}=\{x_j\}_{j=1}^{n}\sim P^{\text{src}}$ and $S^{\text{tgt}}=\{y_j\}_{j=1}^{n}\sim P^{\text{tgt}}$ denote source and target sample sets (not necessarily paired). Let $z_{\text{src}}=\mathcal E(S^{\text{src}})$ and $z_{\text{tgt}}=\mathcal E(S^{\text{tgt}})$ be the corresponding distribution embeddings.
All generators produce a transported sample set $\widehat S^{\text{tgt}}=\{\hat y_j\}_{j=1}^{n}$; in the STC setting they are conditioned on $(z_{\text{src}},z_{\text{tgt}})$, while in the SC setting they are conditioned on $z_{\text{src}}$ alone.

\subsection{Source-conditioned aggregation under unresolved targets}\label{app:sc_aggregation}

\paragraph{Setup.}
Source-conditioned models must aggregate target variation that is not determined by the source population and observed covariates.
Fix that population and its covariates, write $z$ for its embedding, and let $\Pi_z$ denote the conditional law of target distributions $P$ under the supervised pair law $Q_{\mathrm{pair}}$. All conditional laws in this subsection also condition on the fixed source population and covariates; fixing a possibly lossy embedding alone would not fix the source law.
A source-conditioned generator that cannot observe the realized target population must choose a single predicted law $Q_z$.
The usual pairwise training objective averages losses against separate target populations $P\sim\Pi_z$; it does not first pool all conditional target samples into one empirical target set.

\paragraph{Energy and MMD.}
Let $D_{\mathrm{energy}}$ denote the energy distance defined in Appendix~\ref{app:standard_metrics}.
Its source-conditioned population risk decomposes as
\begin{equation}
    \mathcal R_{\mathrm{Energy}}(Q_z)
    :=
    \mathbb E_{P\sim\Pi_z}
    D_{\mathrm{energy}}(Q_z,P)
    =
    D_{\mathrm{energy}}(Q_z,\bar P_z)
    +
    C_{\mathrm{Energy}}(z),
\label{eq:energy_mixture_decomposition}
\end{equation}
where
$\bar P_z:=\mathbb E_{P\sim\Pi_z}P$ is the conditional mixture and
$C_{\mathrm{Energy}}(z)
=
\mathbb E_{P\sim\Pi_z}D_{\mathrm{energy}}(P,\bar P_z)$
does not depend on $Q_z$.
This identity follows by expanding the three expectations in energy distance and using linearity in the target law.

The analogous decomposition for a characteristic-kernel MMD objective is
\begin{align}
    \mathcal R_{\mathrm{MMD}}(Q_z)
    &=
    \mathbb E_{P\sim \Pi_z}
    \mathrm{MMD}_k^2(Q_z,P) \nonumber\\
    &=
    \mathbb E_{P\sim \Pi_z}
    \left\|\mu_{Q_z}-\mu_P\right\|_{\mathcal H_k}^2 \nonumber\\
    &=
    \left\|\mu_{Q_z}-\mathbb E_{P\sim\Pi_z}\mu_P\right\|_{\mathcal H_k}^2
    + \mathrm{const}.
\end{align}
Kernel mean embeddings are linear in the distribution, so
$\mathbb E_{P\sim\Pi_z}\mu_P=\mu_{\bar P_z}$ for the conditional mixture
$\bar P_z:=\mathbb E_{P\sim\Pi_z}P$.
Thus Energy and MMD source-conditioned models both target the conditional mixture endpoint when the model class can realize it.

\paragraph{Flow matching.}
For flow matching, unresolved target variation also enters through the marginal law of the sampled target observation.
Let $P\sim \Pi_z$, $X_0\sim P_{\mathrm{src},z}$, and $X_1\sim P$, with
$X_0$ and $P$ independent conditional on $z$, and set
$X_t=(1-t)X_0+tX_1+\sigma\epsilon$.
The source-conditioned FM population risk is
\[
    \mathcal R_{\mathrm{FM}}(v)
    =
    \mathbb E\left[
    \left\|v(X_t,t,z)-(X_1-X_0)\right\|_2^2
    \mid z
    \right],
\]
whose pointwise minimizer is
\[
v^\star(x,t,z)
=
\mathbb E[X_1-X_0\mid X_t=x,t,z].
\]
Marginalizing the hierarchical target draw gives
\[
    \mathcal L(X_1\mid z)
    =
    \mathbb E_{P\sim\Pi_z}P
    =
    \bar P_z .
\]
The source-conditioned FM regression risk is therefore identical to the risk obtained by drawing target observations iid from $\bar P_z$ under the same interpolation and noise construction.
When that construction defines an exact endpoint-preserving flow, its terminal marginal is $\bar P_z$.
The interpolation path and coupling can still depend on the sampled joint construction.

A distributional target predictor can recover the same conditional mixture through source-target-conditioned transport.
If its sampled embeddings have the conditional law of $\phi(P)$ for $P\sim\Pi_z$, and transport conditioned on $(z,\phi(P))$ has endpoint law $P$, then mixing these endpoint laws gives $\bar P_z$.
This describes the population endpoint under calibrated prediction and exact conditional transport; finite-sample performance also depends on errors in the learned predictor and transport model.

\paragraph{Wasserstein and SWD.}
For Wasserstein-type losses, the same linearization is not available.
The implemented pairwise sliced-Wasserstein objective averages unsquared distances:
\[
    \mathcal R_{\mathrm{SW}}(Q_z)
    =
    \mathbb E_{P\sim\Pi_z,\,\omega_{1:L}}
    \mathrm{SW}_{2,L}(Q_z,P),
\]
where $\mathrm{SW}_{2,L}$ is the finite-direction approximation defined in Appendix~\ref{app:standard_metrics}. Its minimizers are Fr\'echet-median-type aggregates; squared distances instead define a Fr\'echet-mean objective.
This can differ from the mixture distribution.
For example, sliced $W_2$ equals $W_2$ in one dimension.
If $\Pi_z$ puts mass $1/2$ on $\delta_{-1}$ and mass $1/2$ on $\delta_{1}$, then the conditional mixture is
$\frac12\delta_{-1}+\frac12\delta_1$, whereas every $\delta_c$ with $c\in[-1,1]$ minimizes
$\frac12 W_2(Q,\delta_{-1})+\frac12 W_2(Q,\delta_1)$. The mixture itself is not a minimizer.
If every realized target population were instead the mixture
$\frac12\delta_{-1}+\frac12\delta_1$, then SWD would learn that mixture.
The distinction is whether the training risk averages distances to separate unresolved target populations, as in source-conditioned pairwise training, or compares to a pooled conditional target law.
The Trellis $K$-sweep gives the corresponding empirical check for the Energy generator: increasing the number of sampled target embeddings moves the IID mixture toward supervised Energy under both Energy and MMD-RBF evaluation (Table~\ref{tab:trellis_latent_predictor_all_metrics}).

\subsection{Standard distribution encoder}\label{app:standard_encoder}

Across continuous-data experiments, we use a single \emph{standard} distribution encoder $\mathcal E$ to map a finite sample set to a fixed-dimensional embedding. Concretely, given a set $S=\{x_j\}_{j=1}^{n}\subset\mathbb R^d$, the encoder outputs $z=\mathcal E(S)\in\mathbb R^{d_z}$. We adopt the mean-pooled DeepSets/GNN-style architecture used in Generative Distribution Embeddings (GDE) \citep{fishman2025generative}, which is consistently strong across their benchmark suite.

\paragraph{Architecture.}
We first compute per-sample features with an MLP,
\begin{equation}
    h_j^{(0)} = \mathrm{MLP}_{\mathrm{in}}(x_j)\in\mathbb R^{d_h}.
\end{equation}
We then apply $L$ mean-pooled update blocks. For $\ell=1,\dots,L$, define the pooled representation $\bar h^{(\ell-1)}=\frac1n\sum_{j=1}^n h_j^{(\ell-1)}$ and update each sample feature by concatenating with the pooled representation:
\begin{equation}
    h_j^{(\ell)} = \mathrm{MLP}^{(\ell)}_{\mathrm{pool}}\!\left(\big[h_j^{(\ell-1)};\bar h^{(\ell-1)}\big]\right).
\end{equation}
Finally, we average over samples and project to the latent space:
\begin{equation}
    z = \sigma\!\left(W\cdot\frac1n\sum_{j=1}^n h_j^{(L)} + b\right),
\end{equation}
with $\sigma=\mathrm{SELU}$. Optionally, we $\ell_2$-normalize $z$; this is used only when stated explicitly.

\paragraph{Hyperparameters.}
Unless stated otherwise in an experiment appendix, we use $L=2$ mean-pooled blocks and 2-layer MLPs with SELU activations. The hidden dimension $d_h$ and embedding dimension $d_z$ are reported per experiment.

\paragraph{Non-continuous data.}
For TCR repertoire forecasting, we use an ESM2-based distribution encoder described in Appendix~\ref{app:tcr_embedding}.

\subsection{Deterministic regression generators (SWD and Energy)}

Both SWD and Energy use a deterministic transport map applied pointwise:
\begin{equation}
    \hat y_j = \mathcal T_\theta(x_j \mid z_{\text{src}}, z_{\text{tgt}}).
\end{equation}
They differ only in the training objective, which compares $\widehat S^{\text{tgt}}$ to the observed target set $S^{\text{tgt}}$ using a distributional distance (Appendix~\ref{app:standard_metrics}).

\paragraph{SWD (sliced Wasserstein regression).}
We minimize the empirical sliced 2-Wasserstein distance using $L=100$ random projections $\{\omega_\ell\}_{\ell=1}^{L}$ drawn uniformly from the unit sphere. For each $\omega_\ell$, we compute scalar projections and sort them,
\(
\mathrm{sort}(\{\langle \omega_\ell, \hat y_j\rangle\}_{j=1}^{n})
\)
and
\(
\mathrm{sort}(\{\langle \omega_\ell, y_j\rangle\}_{j=1}^{n}),
\)
then take the square root of the mean squared difference across matched order statistics and projections. These unsquared distances are averaged over population pairs; a fresh set of directions is sampled for each pair and loss call.

\paragraph{Energy (energy/MMD regression).}
We minimize the energy distance, equivalently MMD with the energy kernel $k(x,y)=-\|x-y\|$ (Appendix~\ref{app:standard_metrics}). The energy kernel is parameter-free, avoiding bandwidth selection.

\subsection{Flow matching generator (FM)}

Flow matching parameterizes transport with a time-conditional velocity field $v_\theta(x,t,z_{\text{src}},z_{\text{tgt}})$. For training we sample $t\sim\mathrm{Uniform}[0,1]$, $x_0\sim P^{\text{src}}$, $x_1\sim P^{\text{tgt}}$, and $\epsilon\sim\mathcal N(0,I)$, and form the noisy linear interpolation
\begin{equation}
    x_t = (1-t)x_0 + tx_1 + \sigma\epsilon,
\end{equation}
with $\sigma=0.5$ unless otherwise stated.
We train with the conditional flow matching objective
\begin{equation}
    \mathcal L_{\text{FM}}
    =
    \mathbb E\Big[\big\|v_\theta(x_t,t,z_{\text{src}},z_{\text{tgt}}) - (x_1 - x_0)\big\|_2^2\Big].
\end{equation}

At inference, for each source sample $x(0)=x_0$ we solve the ODE
\begin{equation}
    \frac{dx}{dt} = v_\theta(x,t,z_{\text{src}},z_{\text{tgt}}),\qquad t\in[0,1],
\end{equation}
and take $\hat y=x(1)$. We use an adaptive Dormand--Prince solver (dopri5) with absolute and relative tolerances $10^{-4}$ unless otherwise stated.
For positive constant $\sigma$, the interpolation endpoints are the source and target laws convolved with $\mathcal N(0,\sigma^2I)$. Inference starts from unnoised source samples, so this implementation does not by itself satisfy the exact endpoint-preserving condition used in the aggregation argument.

\section{Latent target predictors for semi-supervised STC}\label{app:latent_predictors}

Semi-supervised STC models require a target embedding at test time, but the test target distribution is not observed. We therefore fit a post-hoc predictor on the paired training examples used for the corresponding supervised SC task. Given paired embeddings
\[
    \mathcal D_{\mathrm{pair}}
    =
    \{(z_{\mathrm{src},i}, z_{\mathrm{tgt},i})\}_{i=1}^{N},
\]
the transport generator and encoder are kept fixed, and only the latent predictor is fit.

\paragraph{Expectation predictors.}
Rows marked MSE use a point predictor
\[
    \widehat z_{\mathrm{tgt}} = r_\psi(z_{\mathrm{src}})
\]
trained by squared-error regression,
\[
    \min_\psi \frac1N \sum_{i=1}^N
    \|r_\psi(z_{\mathrm{src},i})-z_{\mathrm{tgt},i}\|_2^2 .
\]
The default predictor is RidgeCV, with regularization selected by cross-validation. At inference, the STC generator is evaluated as
\(\mathcal T_\theta(\cdot \mid z_{\mathrm{src}}, \widehat z_{\mathrm{tgt}})\).

\paragraph{Energy-score latent sampler.}
Rows marked mix. use a conditional latent sampler rather than a single point prediction. The sampler is a small MLP
\[
    \widetilde z_{\mathrm{tgt}} = z_{\mathrm{src}} + g_\phi(z_{\mathrm{src}}, \epsilon),
    \qquad \epsilon \sim \mathcal N(0,I),
\]
trained with the energy score on paired latent embeddings. For $K$ draws
\(\widetilde z_i^{(1)},\dots,\widetilde z_i^{(K)}\) from the sampler, the per-pair loss is
\[
    \frac1K \sum_{a=1}^{K}
    \|\widetilde z_i^{(a)}-z_{\mathrm{tgt},i}\|_2
    -
    \frac{1}{2K(K-1)}
    \sum_{a\ne b}
    \|\widetilde z_i^{(a)}-\widetilde z_i^{(b)}\|_2 .
\]
This is the same strictly proper energy-score idea used by the Energy generator, applied in latent space to paired source--target embeddings. At inference, we sample $k$ target embeddings and draw a fixed-size empirical endpoint from the equal-weight mixture of the resulting transports. Best-of-$k$ uses the held-out target to choose among the corresponding fixed-size endpoints. It is a target-aided diagnostic, not a deployable predictor.

\section{Standard evaluation metrics}\label{app:standard_metrics}

Across experiments we report three distributional distances between a generated sample set $\widehat S=\{\hat y_i\}_{i=1}^{n}$ and a ground-truth target set $S=\{y_j\}_{j=1}^{n}$: sliced Wasserstein distance (\textbf{SWD}), energy distance (\textbf{Energy}), and MMD with an RBF kernel (\textbf{MMD-RBF}, abbreviated \textbf{MMD} in the main text). The definitions below give the default estimators; each experiment appendix states any variations. These metrics are computed in the data space unless stated otherwise (e.g.\ in the TCR experiments, they are computed in ESM2 embedding space).

\subsection{Energy distance (MMD with energy kernel)}

The energy distance between distributions $P$ and $Q$ can be written as
\begin{equation}
    D_{\text{energy}}(P,Q)
    =
    2\mathbb{E}\|X - Y\|
    -
    \mathbb{E}\|X - X'\|
    -
    \mathbb{E}\|Y - Y'\|,
\end{equation}
where $X,X'\sim P$ i.i.d.\ and $Y,Y'\sim Q$ i.i.d. We estimate this quantity empirically from the two finite sample sets. This is equivalent to MMD with kernel $k(x,y)=-\|x-y\|$.
The energy distance is closely related to the \emph{energy score}, a strictly proper scoring rule for probabilistic prediction \citep{gneiting2007strictly}.

\subsection{Sliced 2-Wasserstein distance (SWD)}

We use the standard sliced 2-Wasserstein distance \citep{kolouri2019generalized} estimated with $L=100$ random projections $\{\omega_\ell\}_{\ell=1}^{L}$ drawn uniformly from the unit sphere. For each projection, we compute and sort the scalar projections of both sets, match order statistics, and average the squared differences across projections:
\begin{equation}
    \widehat{\mathrm{SWD}}_2^2(\widehat S,S)
    =
    \frac{1}{L}\sum_{\ell=1}^L \frac{1}{n}\sum_{j=1}^n
    \left(
        \mathrm{sort}(\{\langle \omega_\ell,\hat y_i\rangle\}_{i=1}^n)_j
        -
        \mathrm{sort}(\{\langle \omega_\ell,y_i\rangle\}_{i=1}^n)_j
    \right)^2.
\end{equation}
We report $\widehat{\mathrm{SWD}}_2(\widehat S,S)=\sqrt{\widehat{\mathrm{SWD}}_2^2(\widehat S,S)}$.

\subsection{MMD with RBF kernel (MMD-RBF)}

Given a positive definite kernel $k$, the squared MMD is
\begin{equation}
    \mathrm{MMD}^2(P,Q)
    =
    \mathbb{E}[k(X,X')]
    +
    \mathbb{E}[k(Y,Y')]
    -
    2\mathbb{E}[k(X,Y)],
\end{equation}
with $X,X'\sim P$ and $Y,Y'\sim Q$ i.i.d., estimated empirically from samples.
MMD is a standard kernel two-sample statistic \citep{gretton2012kernel}.
We use an RBF kernel
\begin{equation}
    k_{\text{RBF}}(x,y)=\exp\!\Big(-\tfrac{\|x-y\|_2^2}{2\sigma^2}\Big),
\end{equation}
with bandwidth $\sigma$ set by the median heuristic on the pooled samples from $\widehat S\cup S$.

\section{Gaussian and Gaussian Mixture Experiments}\label{app:gaussian}

\noindent This appendix provides experimental details for the Gaussian experiments in Sections~\ref{subsec:unsup_gauss} and~\ref{subsec:semisup_gauss}. We describe the data-generating processes, model architectures, training configurations, and evaluation protocols. Generator families and metrics are defined in Appendices~\ref{app:standard_generators} and~\ref{app:standard_metrics}.

\subsection{Data-Generating Process}\label{app:gaussian_dgp}

\subsubsection{Multivariate Normal (MVN) Distributions}

We construct a dataset of $n = 50{,}000$ sample sets drawn from bivariate normal distributions. The parameters for each distribution $i$ are sampled as follows:

\paragraph{Mean Prior.} Each component of the mean vector is drawn independently from a uniform distribution:
\begin{equation}
    \mu_i \sim \mathrm{Uniform}([0, 5]^2)
\end{equation}

\paragraph{Covariance Prior.} The covariance matrix is sampled from an inverse-Wishart distribution:
\begin{equation}
    \Sigma_i \sim \mathrm{InverseWishart}(\nu, \Psi)
\end{equation}
where the degrees of freedom $\nu = 10 \cdot (d - 1) = 10$ for $d = 2$ dimensions, and the scale matrix $\Psi = I_d$ (identity matrix). This parameterization ensures well-conditioned covariance matrices with moderate variability.

\paragraph{Sample Generation.} For each distribution $i$, we draw $n_{\text{set}} = 100$ i.i.d.\ samples:
\begin{equation}
    x_{ij} \sim \mathcal{N}(\mu_i, \Sigma_i), \quad j = 1, \ldots, n_{\text{set}}
\end{equation}
Sampling is performed via the Cholesky decomposition $\Sigma_i = L_i L_i^\top$:
\begin{equation}
    x_{ij} = \mu_i + L_i z_{ij}, \quad z_{ij} \sim \mathcal{N}(0, I_d)
\end{equation}

\paragraph{Unique Distribution Subsets.} For experiments studying generalization across distribution density (varying $K$), we sample $K \in \{10, 100, 1{,}000, 10{,}000\}$ unique parameter sets $(\mu_k, \Sigma_k)$ and repeat each parameter set $n / K$ times. This creates multiple sample sets from the same underlying distribution, enabling the study of how models perform when training distributions are seen multiple times versus once.

\subsubsection{Gaussian Mixture Models (GMM)}

We extend the MVN setting to Gaussian mixture models with $C = 3$ components. The parameters for each GMM $i$ are sampled as follows:

\paragraph{Mixture Weight Prior.} Component weights are drawn from a symmetric Dirichlet distribution:
\begin{equation}
    \pi_i \sim \mathrm{Dirichlet}(\alpha \cdot \mathbf{1}_C)
\end{equation}
where $\alpha = 1.0$ yields a uniform prior over the simplex.

\paragraph{Component Mean Prior.} Each component mean is drawn independently:
\begin{equation}
    \mu_{ic} \sim \mathrm{Uniform}([0, 5]^2), \quad c = 1, \ldots, C
\end{equation}

\paragraph{Component Covariance Prior.} Each component covariance is drawn independently:
\begin{equation}
    \Sigma_{ic} \sim \mathrm{InverseWishart}(\nu, \Psi), \quad c = 1, \ldots, C
\end{equation}
with the same hyperparameters as the MVN case ($\nu = 10$, $\Psi = I_d$).

\paragraph{Sample Generation.} For each GMM $i$, we draw $n_{\text{set}} = 1{,}000$ samples via ancestral sampling:
\begin{align}
    c_{ij} &\sim \mathrm{Categorical}(\pi_i) \\
    x_{ij} &\sim \mathcal{N}(\mu_{i,c_{ij}}, \Sigma_{i,c_{ij}})
\end{align}
The larger set size ($1{,}000$ vs.\ $100$ for MVN) provides sufficient samples to capture the multimodal structure.

\subsection{Supervised Gaussian Experiments}\label{app:gaussian_supervised}

For the semi-supervised experiments, we construct paired source-target datasets with known transformations.

\subsubsection{Supervised MVN Dataset}

The supervised MVN dataset creates paired distributions via a fixed shift transformation:
\begin{equation}
    Y = X + b, \quad b = (1, 1)^\top
\end{equation}

\paragraph{Source Distribution.} To study the generalization beyond the training distribution and evaluate the value of any-to-any models for generalization, source means are restricted:
\begin{equation}
    \mu_i^{\text{src}} \sim \mathrm{Uniform}([0, 2.5]^2)
\end{equation}

\paragraph{Target Generation.} Given source samples $\{x_{ij}\}$, target samples are:
\begin{equation}
    y_{ij} = x_{ij} + b
\end{equation}
The target samples are then randomly permuted to remove point-wise correspondence, creating an unpaired but distributionally related dataset.

\subsubsection{Supervised GMM Dataset}

The supervised GMM dataset creates a more challenging bimodal target structure:
\begin{equation}
    Y = X + b \pm b_{\text{off}}
\end{equation}
where $b = (1, 1)^\top$ and $b_{\text{off}} = 0.1 \cdot (-1, 1)^\top$.

\paragraph{Target Generation.} For each source sample, we create two shifted versions:
\begin{align}
    y_{ij}^{+} &= x_{ij} + b + b_{\text{off}} \\
    y_{ij}^{-} &= x_{ij} + b - b_{\text{off}}
\end{align}
These are concatenated and subsampled to maintain the original set size, then randomly permuted. This creates a bimodal target distribution from a unimodal (per-component) source, testing the model's ability to learn complex distributional transformations.

\subsection{Model Architecture}\label{app:gaussian_architecture}

\subsubsection{Distribution Encoder}

We use the standard distribution encoder described in Appendix~\ref{app:standard_encoder} (mean-pooled DeepSets/GNN encoder). For MVN experiments, we use hidden dimension $d_h = 64$ and latent dimension $d_z = 16$; for GMM, we increase these to $d_h = 256$ and $d_z = 128$ to accommodate the more complex multimodal structure.

\subsubsection{Embedding Encoder (Baseline)}

For the $K$-to-$K$ baseline that treats each distribution as a discrete label, we use a learnable embedding table:
\begin{equation}
    z_i = \mathrm{Embedding}(i) \in \mathbb{R}^{d_z}
\end{equation}
where $i \in \{1, \ldots, K\}$ is the distribution index. This encoder cannot generalize to unseen distributions at test time.

\subsubsection{Transport Map}

The transport map $\mathcal{T}(x \mid z_{\text{src}}, z_{\text{tgt}})$ is a 4-layer MLP with SELU activations that takes the concatenation of a source sample $x$ with both the source and target distribution embeddings: $\hat{y} = \mathrm{MLP}([x; z_{\text{src}}; z_{\text{tgt}}])$. The hidden dimension matches the encoder ($d_h = 64$ for MVN, $d_h = 256$ for GMM).

We compare the three standard continuous-data generator families defined in Appendix~\ref{app:standard_generators}: sliced Wasserstein regression (SWD), energy/MMD regression (Energy), and flow matching (FM).

\subsection{Training Protocol}\label{app:gaussian_training}

For the scaling and semi-supervised benchmarks, models are trained using Adam with a base learning rate of $2 \times 10^{-4}$ and evaluated at epoch 1,000, which is the final checkpoint for the scaling runs. We use a batch size of 256 distribution pairs, where each pair consists of source and target sample sets. No early stopping is applied.

\paragraph{Pair Sampling.} During training, we sample source-target pairs uniformly at random. For the unsupervised any-to-any setting, source and target distributions are sampled independently from the full dataset, so the model sees arbitrary pairings. For the supervised setting, each target is deterministically derived from its corresponding source via the fixed transformation (shift for MVN, shift plus off-axis displacement for GMM).

\paragraph{Bidirectional vs.\ Unidirectional Loss.} The unsupervised any-to-any objective trains transport in both directions, summing the losses:
\begin{equation}
    \mathcal{L}_{\text{unsup}} = \mathcal{L}(\text{src} \to \text{tgt}) + \mathcal{L}(\text{tgt} \to \text{src})
\end{equation}
This symmetric objective encourages the encoder to learn representations that support transport regardless of which distribution serves as source or target. For supervised training, we use only the forward direction $\mathcal{L}_{\text{sup}} = \mathcal{L}(\text{src} \to \text{tgt})$, since the transformation is inherently asymmetric.

\subsection{Evaluation Protocol}\label{app:gaussian_evaluation}

\subsubsection{Unsupervised Evaluation}

For each training seed and $K$, both model classes use the same fixed training source and identical evaluation samples. IID targets are $\min(K,100)$ training populations sampled without replacement, with fresh observations from each population. OOD targets are 100 newly sampled populations from the same priors, shared across model classes and values of $K$. Each population is represented by 1,000 samples. The $K$-to-$K$ model uses the training target label for IID evaluation and the nearest training target under Gaussian $W_2$ or GMM optimal transport distance for OOD evaluation. The dense target grid in Fig.~\ref{fig:mvn} is a separate visualization. Tables average the available runs, while Fig.~\ref{fig:unsup} averages differences over matched training seeds.

We measure transport quality using MMD-RBF, SWD, and energy distance (Appendix~\ref{app:standard_metrics}), estimated empirically from the finite sample sets. The Gaussian scaling and semi-supervised MMD-RBF results use the squared, biased estimator, including diagonal kernel terms, with bandwidth equal to the lower median positive pairwise Euclidean distance in the pooled generated and target sets. All three scaling metrics use the same generated and target samples, with 100 shared projection directions for SWD. Scaling Energy uses the unbiased estimator, omitting diagonal self-distances; finite-sample estimates can be negative. The MVN heatmap and ablations report the closed-form squared 2-Wasserstein distance between Gaussian laws,
\[
W_2^2\!\left(\mathcal{N}(\mu_1,\Sigma_1), \mathcal{N}(\mu_2,\Sigma_2)\right)
= \|\mu_1-\mu_2\|_2^2
+ \operatorname{tr}\!\left(\Sigma_1+\Sigma_2-2\left(\Sigma_2^{1/2}\Sigma_1\Sigma_2^{1/2}\right)^{1/2}\right).
\]
Here the generated distribution is approximated by a Gaussian with its empirical mean and covariance, so this diagnostic compares means and covariances.

\subsubsection{Semi-Supervised Evaluation}

For the semi-supervised experiments, we restrict the supervised training data to distributions with source means satisfying $\|\mu\|_\infty \leq 2.5$, then evaluate on a $20\times20$ grid over $[0.25,4.75]^2$, using 10,000 samples per population. We partition test results by the $L^\infty$ norm of the source mean to assess how performance degrades outside the supervised training support.

The semi-supervised approach first trains an any-to-any encoder on the full unsupervised dataset (50,000 sample sets from 10,000 distinct populations), then fits a ridge regression predictor $\hat{z}_{\text{tgt}} = W z_{\text{src}} + b$ on the paired subset within the restricted support. At test time, we encode the source distribution, predict the target embedding via the linear map, and apply the learned transport.

We compare against a \emph{supervised} model that conditions only on the source embedding and is trained exclusively on paired data within the support, and a target-aided \emph{oracle} diagnostic that uses the true target embedding computed by the encoder. The latter is unavailable at deployment and is not a performance bound.

\subsection{Hyperparameter Summary}\label{app:gaussian_configs}

We summarize the key hyperparameters for reproducibility. All experiments use the same random seed protocol and data splits.

\subsubsection{Dataset and Model Hyperparameters}

Table~\ref{tab:gaussian_hyperparams} provides a consolidated view of the dataset and model configurations for both experimental settings.

\begin{table}[H]
\centering
\caption{\textit{Gaussian benchmark hyperparameters.} Dataset, encoder, transport-map, and training settings are reported for the MVN and GMM experiments; the same configurations are used for the unsupervised scaling and semi-supervised Gaussian benchmarks unless otherwise stated.}
\label{tab:gaussian_hyperparams}
\begin{tabular}{lcc}
\toprule
\textbf{Hyperparameter} & \textbf{MVN} & \textbf{GMM} \\
\midrule
\multicolumn{3}{l}{\textit{Dataset}} \\
Number of sample sets & 50,000 & 50,000 \\
Samples per set & 100 & 1,000 \\
Data dimension & 2 & 2 \\
Mean prior range & $[0, 5]^2$ & $[0, 5]^2$ \\
Covariance prior (inv-Wishart $\nu$) & 10 & 10 \\
Mixture components & --- & 3 \\
Dirichlet concentration $\alpha$ & --- & 1.0 \\
\midrule
\multicolumn{3}{l}{\textit{Encoder}} \\
Latent dimension $d_z$ & 16 & 128 \\
Hidden dimension $d_h$ & 64 & 256 \\
Mean-pooled GNN layers & 2 & 2 \\
FC layers per block & 2 & 2 \\
\midrule
\multicolumn{3}{l}{\textit{Transport Map}} \\
MLP layers & 4 & 4 \\
Hidden dimension & 64 & 256 \\
Activation & SELU & SELU \\
\midrule
\multicolumn{3}{l}{\textit{Training}} \\
Batch size & 256 & 256 \\
Learning rate & $2 \times 10^{-4}$ & $2 \times 10^{-4}$ \\
Evaluation epoch & 1,000 & 1,000 \\
Optimizer & Adam & Adam \\
\bottomrule
\end{tabular}
\end{table}

For the supervised experiments, we restrict the mean prior to $[0, 2.5]^2$ so that the shifted targets remain within the original $[0, 5]^2$ range. The shift vector is $b = (1, 1)^\top$ for both MVN and GMM; GMM additionally uses an off-axis displacement of magnitude $0.1$ to create bimodal targets.

\subsection{Full Results}\label{app:gaussian_results}

We present comprehensive results from our Gaussian experiments. Full scaling and semi-supervised tables report MMD-RBF, SWD, and energy distance; single-metric ablations report the metric stated in their captions.

\subsubsection{Unsupervised Transport: Scaling with \texorpdfstring{$K$}{K}}

Tables~\ref{tab:mvn_full_scaling} and~\ref{tab:gmm_full_scaling} report the full breakdown of transport quality as a function of $K$, the number of unique training distributions. The $K$-to-$K$ baseline performs well on IID targets at small $K$, while any-to-any performance improves as the training family grows. At $K\geq100$, all three DCT generators have lower OOD error than their $K$-to-$K$ counterparts on both datasets. At $K=10$, Energy-DCT and SWD-DCT have higher GMM OOD Energy, so the advantage does not extend uniformly across metrics to small training families.

\begin{table}[H]
\centering
\small
\setlength{\tabcolsep}{2.5pt}
\caption{\textit{Unsupervised MVN scaling across training distributions.} We compare $K$-to-$K$ baselines and any-to-any DCT models as the number of unique training distributions $K$ increases. IID and OOD target sets are reported separately; values are MMD-RBF, SWD, and energy distance ($\downarrow$).}
\label{tab:mvn_full_scaling}
\textbf{IID Targets}
\vspace{0.3em}

\begin{tabular}{lcccccccccccc}
\toprule
 & \multicolumn{3}{c}{$K=10$} & \multicolumn{3}{c}{$K=100$} & \multicolumn{3}{c}{$K=1,000$} & \multicolumn{3}{c}{$K=10,000$} \\
\textbf{Model} & MMD & SWD & Energy & MMD & SWD & Energy & MMD & SWD & Energy & MMD & SWD & Energy \\
\midrule
SWD ($K$) & 0.001 & 0.036 & 0.001 & 0.003 & 0.041 & 0.002 & 0.006 & 0.083 & 0.010 & 0.016 & 0.104 & 0.016 \\
SWD (Any) & 0.000 & 0.026 & 0.000 & 0.001 & 0.029 & 0.000 & 0.001 & 0.031 & 0.000 & 0.001 & 0.031 & 0.000 \\
\midrule
Energy ($K$) & 0.001 & 0.033 & 0.000 & 0.003 & 0.056 & 0.002 & 0.005 & 0.093 & 0.008 & 0.016 & 0.104 & 0.015 \\
Energy (Any) & 0.000 & 0.026 & -0.001 & 0.001 & 0.030 & 0.000 & 0.001 & 0.031 & 0.000 & 0.001 & 0.030 & 0.000 \\
\midrule
FM ($K$) & 0.009 & 0.074 & 0.006 & 0.010 & 0.079 & 0.008 & 0.013 & 0.095 & 0.011 & 0.026 & 0.133 & 0.021 \\
FM (Any) & 0.008 & 0.071 & 0.006 & 0.008 & 0.074 & 0.006 & 0.009 & 0.079 & 0.007 & 0.011 & 0.085 & 0.009 \\
\bottomrule
\end{tabular}

\vspace{0.5em}
\textbf{OOD Targets}
\vspace{0.3em}

\begin{tabular}{lcccccccccccc}
\toprule
 & \multicolumn{3}{c}{$K=10$} & \multicolumn{3}{c}{$K=100$} & \multicolumn{3}{c}{$K=1,000$} & \multicolumn{3}{c}{$K=10,000$} \\
\textbf{Model} & MMD & SWD & Energy & MMD & SWD & Energy & MMD & SWD & Energy & MMD & SWD & Energy \\
\midrule
SWD ($K$) & 0.483 & 0.622 & 0.889 & 0.121 & 0.246 & 0.151 & 0.033 & 0.142 & 0.039 & 0.021 & 0.113 & 0.021 \\
SWD (Any) & 0.085 & 0.193 & 0.131 & 0.002 & 0.037 & 0.001 & 0.001 & 0.031 & 0.000 & 0.001 & 0.030 & 0.000 \\
\midrule
Energy ($K$) & 0.469 & 0.615 & 0.871 & 0.103 & 0.230 & 0.128 & 0.031 & 0.145 & 0.037 & 0.021 & 0.113 & 0.021 \\
Energy (Any) & 0.071 & 0.176 & 0.100 & 0.002 & 0.038 & 0.001 & 0.001 & 0.030 & 0.000 & 0.001 & 0.029 & 0.000 \\
\midrule
FM ($K$) & 0.456 & 0.608 & 0.857 & 0.099 & 0.227 & 0.126 & 0.032 & 0.136 & 0.035 & 0.027 & 0.133 & 0.025 \\
FM (Any) & 0.047 & 0.140 & 0.060 & 0.010 & 0.078 & 0.007 & 0.009 & 0.076 & 0.007 & 0.010 & 0.080 & 0.008 \\
\bottomrule
\end{tabular}

\end{table}

\begin{table}[H]
\centering
\small
\setlength{\tabcolsep}{2.5pt}
\caption{\textit{Unsupervised GMM scaling across training distributions.} We compare $K$-to-$K$ baselines and any-to-any DCT models as the number of unique training distributions $K$ increases. IID and OOD target sets are reported separately; values are MMD-RBF, SWD, and energy distance ($\downarrow$).}
\label{tab:gmm_full_scaling}
\textbf{IID Targets}
\vspace{0.3em}

\begin{tabular}{lcccccccccccc}
\toprule
 & \multicolumn{3}{c}{$K=10$} & \multicolumn{3}{c}{$K=100$} & \multicolumn{3}{c}{$K=1,000$} & \multicolumn{3}{c}{$K=10,000$} \\
\textbf{Model} & MMD & SWD & Energy & MMD & SWD & Energy & MMD & SWD & Energy & MMD & SWD & Energy \\
\midrule
SWD ($K$) & 0.001 & 0.108 & 0.000 & 0.002 & 0.125 & 0.002 & 0.003 & 0.149 & 0.008 & 0.006 & 0.189 & 0.018 \\
SWD (Any) & 0.001 & 0.066 & -0.001 & 0.001 & 0.093 & 0.001 & 0.002 & 0.111 & 0.003 & 0.002 & 0.121 & 0.005 \\
\midrule
Energy ($K$) & 0.001 & 0.121 & 0.001 & 0.001 & 0.129 & 0.001 & 0.002 & 0.151 & 0.005 & 0.005 & 0.201 & 0.016 \\
Energy (Any) & 0.000 & 0.082 & -0.002 & 0.001 & 0.098 & 0.000 & 0.002 & 0.127 & 0.002 & 0.003 & 0.145 & 0.005 \\
\midrule
FM ($K$) & 0.015 & 0.295 & 0.042 & 0.020 & 0.362 & 0.055 & 0.021 & 0.357 & 0.058 & 0.019 & 0.346 & 0.054 \\
FM (Any) & 0.012 & 0.287 & 0.036 & 0.021 & 0.363 & 0.057 & 0.019 & 0.337 & 0.052 & 0.022 & 0.333 & 0.053 \\
\bottomrule
\end{tabular}

\vspace{0.5em}
\textbf{OOD Targets}
\vspace{0.3em}

\begin{tabular}{lcccccccccccc}
\toprule
 & \multicolumn{3}{c}{$K=10$} & \multicolumn{3}{c}{$K=100$} & \multicolumn{3}{c}{$K=1,000$} & \multicolumn{3}{c}{$K=10,000$} \\
\textbf{Model} & MMD & SWD & Energy & MMD & SWD & Energy & MMD & SWD & Energy & MMD & SWD & Energy \\
\midrule
SWD ($K$) & 0.226 & 0.852 & 0.680 & 0.116 & 0.564 & 0.310 & 0.064 & 0.419 & 0.171 & 0.042 & 0.352 & 0.115 \\
SWD (Any) & 0.158 & 0.766 & 0.782 & 0.021 & 0.247 & 0.050 & 0.005 & 0.144 & 0.012 & 0.003 & 0.126 & 0.006 \\
\midrule
Energy ($K$) & 0.224 & 0.852 & 0.676 & 0.115 & 0.564 & 0.309 & 0.066 & 0.422 & 0.173 & 0.044 & 0.366 & 0.121 \\
Energy (Any) & 0.141 & 0.816 & 0.944 & 0.012 & 0.226 & 0.031 & 0.002 & 0.131 & 0.004 & 0.003 & 0.144 & 0.006 \\
\midrule
FM ($K$) & 0.236 & 0.840 & 0.680 & 0.128 & 0.609 & 0.334 & 0.082 & 0.508 & 0.206 & 0.055 & 0.445 & 0.143 \\
FM (Any) & 0.060 & 0.480 & 0.184 & 0.027 & 0.357 & 0.071 & 0.020 & 0.338 & 0.054 & 0.022 & 0.343 & 0.055 \\
\bottomrule
\end{tabular}

\end{table}

\subsubsection{Additional Encoder Ablations}

We also compare alternative distribution encoders on the MVN benchmark. First, we consider a fixed kernel mean embedding (KME) baseline. Second, to explore a more expressive learned set encoder, we include a mean-pooled Set Transformer encoder. Architectures use the same embedding dimension and the same transport map so that differences reflect the conditioning representation.

\begin{table}[H]
\centering
\small
\caption{\textit{MVN any-to-any encoder ablation.} OOD performance is reported for $K$-to-$K$, fixed KME, learned DeepSets, and Set Transformer encoders with matched transport maps. Values are $W_2^2$ distance ($\downarrow$) across $K$ unique training distributions.}
\label{tab:mvn_any_encoder_ablation_gauss}
\begin{tabular}{llcccc}
\toprule
\textbf{Encoder} & \textbf{Generator} & $K=10$ & $K=100$ & $K=1{,}000$ & $K=10{,}000$ \\
\midrule
$K$-to-$K$ & FM & 1.0926 & 0.1250 & 0.0506 & 0.0560 \\
           & Energy & 1.1110 & 0.1276 & 0.0409 & 0.0343 \\
           & SWD & 1.0310 & 0.1449 & 0.0405 & 0.0325 \\
\midrule
KME & FM & 9.6780 & 9.2970 & 6.4750 & 6.9160 \\
    & Energy & 7.8860 & 6.1360 & 6.1110 & 6.2600 \\
    & SWD & 7.2240 & 5.6940 & 4.3590 & 3.7510 \\
\midrule
DeepSets & FM & 0.0925 & 0.0426 & 0.0505 & 0.1982 \\
         & Energy & 0.1884 & 0.0570 & 0.0587 & 0.1336 \\
         & SWD & 0.2391 & 0.0499 & 0.0537 & 0.0852 \\
\midrule
Set Transformer & FM & 0.0469 & 0.0290 & 0.0353 & 0.0560 \\
                & Energy & 0.1239 & 0.0151 & 0.0096 & 0.0113 \\
                & SWD & 0.1012 & 0.0050 & 0.0038 & 0.0028 \\
\bottomrule
\end{tabular}
\end{table}

\begin{table}[H]
\centering
\small
\caption{\textit{MVN semi-supervised encoder ablation.} Supervised (Sup.), semi-supervised (Semi.), and oracle (Ora.) regimes are evaluated under alternative target-conditioning representations. Values are OOD $W_2^2$ distance ($\downarrow$).}
\label{tab:mvn_semisup_encoder_ablation_gauss}
\begin{tabular}{llccc}
\toprule
\textbf{Encoder} & \textbf{Generator} & Sup. & Semi. & Ora. \\
\midrule
KME & FM & 0.0074 & 1.5982 & 1.5982 \\
    & Energy & 0.0427 & 3.1122 & 3.1122 \\
    & SWD & 0.0308 & 3.7641 & 3.7641 \\
\midrule
DeepSets & FM & 0.0010 & 0.0111 & 0.0054 \\
         & Energy & 0.0866 & 0.0389 & 0.0135 \\
         & SWD & 0.0481 & 0.0304 & 0.0133 \\
\midrule
Set Transformer & FM & 0.0022 & 0.0046 & 0.0036 \\
                & Energy & 0.1167 & 0.0048 & 0.0041 \\
                & SWD & 0.0764 & 0.0041 & 0.0019 \\
\bottomrule
\end{tabular}
\end{table}

\subsubsection{Semi-Supervised Generalization}

Tables~\ref{tab:mvn_semisup} and~\ref{tab:gmm_semisup} compare supervised, semi-supervised, and oracle methods on the paired Gaussian prediction problems. The supervised model is trained only on paired data within the restricted support ($\|\mu\|_\infty \leq 2.5$), while the semi-supervised approach uses the any-to-any encoder with a ridge predictor for the target embedding. The oracle uses the true target embedding.

\begin{table}[H]
\centering
\small
\caption{\textit{Semi-supervised MVN prediction.} Supervised rows train only on paired source distributions with $\|\mu\|_\infty \leq 2.5$, Semi-sup. rows use the any-to-any encoder with a predicted target embedding, and Oracle rows condition on the true target embedding. Values are distributional distances ($\downarrow$), split by source mean into IID points with $\|\mu\|_\infty \leq 2.5$ and OOD points with $\|\mu\|_\infty > 2.5$.}
\label{tab:mvn_semisup}
\begin{tabular}{llcccccc}
\toprule
 & & \multicolumn{3}{c}{\textbf{IID}} & \multicolumn{3}{c}{\textbf{OOD}} \\
\cmidrule(lr){3-5} \cmidrule(lr){6-8}
\textbf{Generator} & \textbf{Method} & MMD & SWD & Energy & MMD & SWD & Energy \\
\midrule
Flow Matching & Supervised & 0.000 & 0.003 & -0.000 & 0.000 & 0.006 & -0.000 \\
 & Semi-sup. & 0.001 & 0.019 & 0.001 & 0.001 & 0.021 & 0.001 \\
 & Oracle & 0.001 & 0.020 & 0.001 & 0.002 & 0.028 & 0.003 \\
\midrule
Energy & Supervised & 0.000 & 0.003 & -0.000 & 0.018 & 0.062 & 0.020 \\
 & Semi-sup. & 0.000 & 0.021 & 0.000 & 0.002 & 0.033 & 0.002 \\
 & Oracle & 0.000 & 0.021 & 0.000 & 0.002 & 0.030 & 0.002 \\
\midrule
SWD & Supervised & 0.000 & 0.003 & -0.000 & 0.015 & 0.052 & 0.017 \\
 & Semi-sup. & 0.000 & 0.024 & 0.001 & 0.002 & 0.034 & 0.002 \\
 & Oracle & 0.000 & 0.024 & 0.001 & 0.003 & 0.036 & 0.003 \\
\bottomrule
\end{tabular}
\end{table}

\begin{table}[H]
\centering
\small
\caption{\textit{Semi-supervised GMM prediction.} Supervised rows train only on paired source distributions with $\|\mu\|_\infty \leq 2.5$, Semi-sup. rows use the any-to-any encoder with a predicted target embedding, and Oracle rows condition on the true target embedding. Values are distributional distances ($\downarrow$), split by source mean into IID points with $\|\mu\|_\infty \leq 2.5$ and OOD points with $\|\mu\|_\infty > 2.5$.}
\label{tab:gmm_semisup}
\begin{tabular}{llcccccc}
\toprule
 & & \multicolumn{3}{c}{\textbf{IID}} & \multicolumn{3}{c}{\textbf{OOD}} \\
\cmidrule(lr){3-5} \cmidrule(lr){6-8}
\textbf{Generator} & \textbf{Method} & MMD & SWD & Energy & MMD & SWD & Energy \\
\midrule
Flow Matching & Supervised & 0.000 & 0.015 & 0.000 & 0.001 & 0.022 & 0.001 \\
 & Semi-sup. & 0.002 & 0.032 & 0.002 & 0.010 & 0.052 & 0.011 \\
 & Oracle & 0.002 & 0.032 & 0.002 & 0.004 & 0.041 & 0.005 \\
\midrule
Energy & Supervised & 0.000 & 0.008 & 0.000 & 0.031 & 0.094 & 0.037 \\
 & Semi-sup. & 0.001 & 0.028 & 0.001 & 0.007 & 0.061 & 0.008 \\
 & Oracle & 0.001 & 0.029 & 0.001 & 0.004 & 0.044 & 0.005 \\
\midrule
SWD & Supervised & 0.000 & 0.008 & 0.000 & 0.015 & 0.065 & 0.018 \\
 & Semi-sup. & 0.002 & 0.036 & 0.003 & 0.010 & 0.070 & 0.013 \\
 & Oracle & 0.002 & 0.036 & 0.003 & 0.007 & 0.059 & 0.009 \\
\bottomrule
\end{tabular}
\end{table}

Outside the training support, semi-supervised Energy improves all three reported metrics on MVN and GMM; semi-supervised SWD improves all three metrics on MVN and MMD-RBF and Energy on GMM, while its GMM SWD increases from $0.065$ to $0.070$. Supervised FM remains the strongest synthetic OOD model. The benefit of any-to-any training therefore depends on the generator and metric.

\paragraph{Flow Matching Anomaly.} We observe that supervised flow matching generalizes surprisingly well outside its training support in this synthetic setting (low OOD error even for the supervised method). This behavior is not replicated in real-world experiments, suggesting it may be an artifact of the simple linear transformation and low-dimensional setting where the learned flow field happens to extrapolate correctly.

\section{Experimental details for scRNA-seq batch effect transfer}\label{app:batch}

\noindent This appendix provides experimental details for the scRNA-seq batch effect transfer experiment in Section~\ref{subsec:unsup_batch}. We describe the preprocessing pipeline, model configurations, and evaluation protocol. Generator families and metrics are defined in Appendices~\ref{app:standard_generators} and~\ref{app:standard_metrics}.

\subsection{Data preprocessing}

We use the murine pancreas scRNA-seq atlas data collected in \citet{Hrovatin2023-wd}. We use a preprocessed version of the dataset distributed by \verb|cellxgene|\footnote{\href{https://cellxgene.cziscience.com/collections/296237e2-393d-4e31-b590-b03f74ac5070}{https://cellxgene.cziscience.com/collections/296237e2-393d-4e31-b590-b03f74ac5070}}. We use the top $d=10$ principal components for all downstream analysis. These 10 PCs are standardized to have unit variance and zero mean.

\paragraph{Train-test split.}
In order to induce a distribution shift at test-time, we hold out all batches of one condition. We hold out the three batches from mice older than 2 years, leaving 53 batches for the training dataset.

\paragraph{Dataset sampling.}
For each batch in the split, we sample sets of $n = 128$ cells uniformly at random. At training time, for each source batch, a random target batch is selected from the training set, and sets are sampled from both, yielding random source-target sets $(S_i^{\text{src}}, S_i^{\text{tgt}})$. This corresponds to STC (any-to-any) training in the terminology of Appendix~\ref{app:standard_generators}.

\subsection{Model configurations}

We evaluate the three standard continuous-data generator families described in Appendix~\ref{app:standard_generators}: flow matching (FM), energy/MMD regression (Energy), and sliced Wasserstein regression (SWD). For each generator type, we compare two encoder architectures: (i) a \emph{distribution encoder} (as in DCT) that maps each sampled cell set to an embedding, and (ii) a \emph{one-hot encoder} that learns a fixed embedding vector for each batch in the training set (a $K$-conditioned baseline).
Unless stated otherwise, the distribution encoder uses the standard architecture described in Appendix~\ref{app:standard_encoder}.

For the one-hot encoder at test time, we assign each test batch to its nearest training batch based on Euclidean distance between batch centroids (computed as the mean of all cells in PCA space for that batch). The model then uses the learned embedding of this nearest training batch.

\paragraph{Baselines.} We compare against scVI \citep{Lopez2018-lh}, a deep generative model for scRNA-seq that learns a latent representation while correcting for batch effects. We train scVI on the three held-out batches, using the batch identity as the batch condition. We use a model with 10 latent dimensions and 2 hidden layers, and 128 hidden units per layer. We train for 1000 epochs. We evaluate its ability to transform source distributions to match target distributions via the \texttt{transform\_batch} functionality, which changes the batch condition passed to the decoder. The resulting normalized expression is then matched to its nearest neighbor in the target dataset. We also compare against Harmony \citep{Korsunsky2019-ji}, which is a non-generative method. In order to ``transfer'' batch effect for a source cell, we compute its nearest neighbor in the batch corrected space which has the target batch label. This approach does not allow for any notion of a held-out batch -- all cells are seen during the integration step.

\paragraph{Evaluation protocol.} For each of the three held-out source batches, we sample 10 source--target set pairs, selecting the target batch uniformly from the held-out batches. For our models, the transport map acts on the source cells and conditions on embeddings of the source and target sample sets. For scVI, we transform the source cells to the target batch and use the observed nearest neighbor for each prediction. We then compute metrics between the predicted and ground-truth target populations.

\paragraph{Metrics.} We report the standard distributional metrics defined in Appendix~\ref{app:standard_metrics}: energy distance, sliced Wasserstein distance, and MMD with an RBF kernel. Metrics are computed between generated and ground-truth target distributions for each sampled pair and reported as mean $\pm$ standard error over 30 draws per model. MMD-RBF is the biased squared statistic, with bandwidth given by the median positive pairwise distance in the pooled samples. Energy uses half the biased estimator. For the neural and scVI rows, the MMD-RBF evaluation draws are independent of those used for Energy and SWD; Harmony uses the same draws for all three metrics.

\subsection{Full results across metrics}

In Table~\ref{tab:batch}, we show MMD-RBF for all evaluated models. In Table~\ref{tab:app_batch}, we show all three evaluated metrics. Figure~\ref{fig:batch} visualizes held-out batch-pair predictions for the DCT and $K$-to-$K$ models.


\begin{table}[H]
    \centering
    \begin{tabular}{llccc}
    \toprule
    \textbf{Generator} & \textbf{Encoder} & \textbf{Energy} $\downarrow$ & \textbf{SWD} $\downarrow$ & \textbf{MMD-RBF} $\downarrow$ \\
    \midrule
    \multicolumn{2}{l}{\textsc{scVI}} & 0.9074 $\pm$ 0.0698 & 0.6643 $\pm$ 0.0566 & 0.6178 $\pm$ 0.0210 \\
    \multicolumn{2}{l}{Harmony} & 0.2584 $\pm$ 0.0402 & 0.4999 $\pm$ 0.0355 & 0.2057 $\pm$ 0.0335 \\
    \midrule
    \multirow[t]{2}{*}{SWD}
        & Any-to-any & 0.1164 $\pm$ 0.0118 & 0.3051 $\pm$ 0.0298 & 0.0928 $\pm$ 0.0061 \\
        & $K$-to-$K$ & 0.3160 $\pm$ 0.0286 & 0.5164 $\pm$ 0.0226 & 0.2835 $\pm$ 0.0217 \\
    \midrule
    \multirow[t]{2}{*}{Energy}
        & Any-to-any & 0.0639 $\pm$ 0.0046 & 0.4248 $\pm$ 0.0257 & 0.0428 $\pm$ 0.0036 \\
        & $K$-to-$K$ & 0.4000 $\pm$ 0.0453 & 0.6222 $\pm$ 0.0333 & 0.1826 $\pm$ 0.0180 \\
    \midrule
    \multirow[t]{2}{*}{FM}
        & Any-to-any & 0.0733 $\pm$ 0.0119 & 0.3309 $\pm$ 0.0436 & 0.0839 $\pm$ 0.0089 \\
        & $K$-to-$K$ & 0.2429 $\pm$ 0.0385 & 0.4659 $\pm$ 0.0390 & 0.2260 $\pm$ 0.0292 \\
    \bottomrule
    \end{tabular}
    \caption{\textit{Full batch effect transfer metrics.} Any-to-any DCT and $K$-to-$K$ baselines are evaluated on held-out murine pancreas batch pairs against ground-truth target batch distributions. scVI and Harmony are included as external baselines; values are mean $\pm$ standard error for energy distance, SWD, and MMD-RBF ($\downarrow$).}
    \label{tab:app_batch}
\end{table}

\begin{figure}[H]
    \centering
    \includegraphics[width=0.9\linewidth]{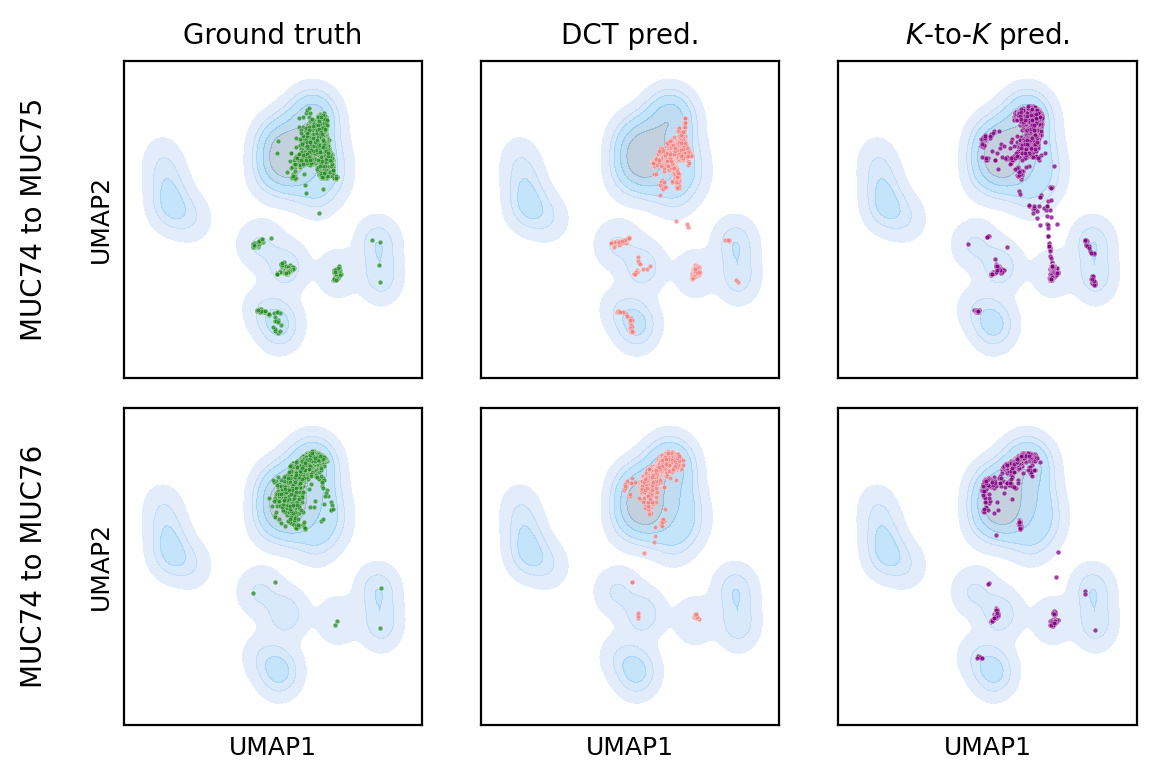}
    \caption{\textit{Batch effect transfer on held-out batches.} Each row shows one source--target batch pair from the murine pancreas benchmark. Columns compare target cells, any-to-any DCT predictions, and $K$-to-$K$ predictions over reference density contours from all batches, illustrating that DCT better matches the held-out target distributions.}
    \label{fig:batch}
\end{figure}

\subsection{Targeted encoder ablations}

Table~\ref{tab:batch_encoder_baselines_app} reports the Energy-only encoder comparisons for batch transfer. These targeted comparisons show that the Set Transformer encoder remains competitive while the fixed KME representation performs poorly in this application.

\begin{table}[H]
\centering
\small
\caption{\textit{Batch-transfer encoder comparisons.} The any-to-any Energy model is evaluated with the standard learned DeepSets encoder, a Set Transformer encoder, and a fixed KME representation. Values are energy distance (mean $\pm$ standard error, $\downarrow$).}
\label{tab:batch_encoder_baselines_app}
\begin{tabular}{lc}
\toprule
\textbf{Model} & \textbf{Energy} $\downarrow$ \\
\midrule
DeepSets & $0.0639 \pm 0.0046$ \\
Set Transformer & $0.045 \pm 0.005$ \\
KME & $0.691 \pm 0.071$ \\
\bottomrule
\end{tabular}
\end{table}

\section{Drug perturbation prediction on organoids (Trellis)}\label{app:trellis}

\noindent This appendix provides additional details for the Trellis organoid drug-perturbation benchmarks in Section~\ref{subsec:trellis}.

\subsection{Data preprocessing}

We use the Trellis single-cell mass cytometry dataset \citep{zapatero2023trellis}, representing each cell by 43 protein abundance features and forming 927 matched control--treatment population pairs across 10 patients and 11 treatments. Each population is represented at training time by a set of $n=256$ cells sampled uniformly at random.

\paragraph{Train-test split.}
We evaluate on two replicate-holdout splits (\texttt{replicas-1}, \texttt{replicas-2}) and three patient-holdout splits (\texttt{pdo21}, \texttt{pdo27}, \texttt{pdo75}).

\subsection{Model configurations}

Models use the standard continuous-data generators (SWD/Energy/FM; Appendix~\ref{app:standard_generators}). We report three conditioning regimes: a supervised source-conditioned (SC) baseline (MFM) that conditions on treatment identity, a semi-supervised source--target-conditioned (STC) model with a post-hoc ridge predictor for the target embedding, and an oracle STC evaluation that conditions on the true target embedding. Unless stated otherwise, the distribution encoder uses the standard architecture described in Appendix~\ref{app:standard_encoder}.

\paragraph{Key hyperparameters.} Unless otherwise stated, we use embedding dimension 256, hidden dimension 256, batch size 4, learning rate $2\times 10^{-4}$, and train for 20{,}000 epochs.

\subsection{Evaluation protocol}

We evaluate generated populations against the held-out target population using the standard distributional metrics defined in Appendix~\ref{app:standard_metrics}: energy distance, sliced Wasserstein distance (SWD), and MMD with an RBF kernel (MMD-RBF). Results are reported separately for IID replicate-holdout and OOD patient-holdout splits.

\subsection{Full results across metrics}

\begin{table}[H]
    \centering
    \caption{\textit{Full Trellis PDO perturbation prediction results.} Generator families are compared across supervised SC, semi-supervised STC, and oracle STC regimes. Energy, SWD, and MMD-RBF are reported as mean $\pm$ standard deviation ($\downarrow$) for IID replicate-holdout and OOD patient-holdout splits; Energy STC (mix.) evaluates the distributional latent predictor by mixing transports over $k=32$ sampled target latents.}
    \label{tab:trellis_combined}
    \begin{tabular}{llccc}
    \toprule
    \textbf{Generator} & \textbf{Regime} & \textbf{Energy} & \textbf{SWD} & \textbf{MMD-RBF} \\
    \midrule
    \multicolumn{5}{c}{\textit{IID (Replicate Holdout)}} \\
    \midrule
    \multirow[t]{3}{*}{SWD}
        & Supervised & $0.0897 \pm 0.0332$ & $0.1266 \pm 0.0256$ & $0.0064 \pm 0.0022$ \\
        & Semi-sup. & $0.1305 \pm 0.0305$ & $0.1524 \pm 0.0185$ & $0.0094 \pm 0.0018$ \\
        & Oracle & $0.0154 \pm 0.0005$ & $0.0633 \pm 0.0009$ & $0.0010 \pm 0.0000$ \\
    \midrule
    \multirow[t]{4}{*}{Energy}
        & Supervised & $0.0779 \pm 0.0290$ & $0.1199 \pm 0.0209$ & $0.0054 \pm 0.0019$ \\
        & Semi-sup. & $0.1442 \pm 0.0520$ & $0.1588 \pm 0.0297$ & $0.0103 \pm 0.0032$ \\
        & Semi-sup. (mix.) & $0.0940 \pm 0.0279$ & $0.1329 \pm 0.0205$ & $0.0067 \pm 0.0018$ \\
        & Oracle & $0.0169 \pm 0.0026$ & $0.0659 \pm 0.0037$ & $0.0011 \pm 0.0002$ \\
    \midrule
    \multirow[t]{3}{*}{FM}
        & Supervised & $0.1431 \pm 0.0377$ & $0.1621 \pm 0.0187$ & $0.0101 \pm 0.0023$ \\
        & Semi-sup. & $0.1450 \pm 0.0380$ & $0.1602 \pm 0.0168$ & $0.0106 \pm 0.0025$ \\
        & Oracle & $0.0415 \pm 0.0034$ & $0.0929 \pm 0.0054$ & $0.0030 \pm 0.0003$ \\
    \midrule
    \multicolumn{5}{c}{\textit{OOD (Patient Holdout)}} \\
    \midrule
    \multirow[t]{3}{*}{SWD}
        & Supervised & $0.3391 \pm 0.0093$ & $0.2452 \pm 0.0078$ & $0.0259 \pm 0.0008$ \\
        & Semi-sup. & $0.2580 \pm 0.0403$ & $0.2103 \pm 0.0196$ & $0.0197 \pm 0.0022$ \\
        & Oracle & $0.0346 \pm 0.0098$ & $0.0819 \pm 0.0066$ & $0.0026 \pm 0.0009$ \\
    \midrule
    \multirow[t]{4}{*}{Energy}
        & Supervised & $0.3932 \pm 0.0258$ & $0.2622 \pm 0.0074$ & $0.0298 \pm 0.0019$ \\
        & Semi-sup. & $0.2503 \pm 0.0556$ & $0.2091 \pm 0.0246$ & $0.0191 \pm 0.0049$ \\
        & Semi-sup. (mix.) & $0.4516 \pm 0.0697$ & $0.2872 \pm 0.0230$ & $0.0330 \pm 0.0041$ \\
        & Oracle & $0.0315 \pm 0.0016$ & $0.0808 \pm 0.0028$ & $0.0022 \pm 0.0002$ \\
    \midrule
    \multirow[t]{3}{*}{FM}
        & Supervised & $0.2969 \pm 0.0209$ & $0.2247 \pm 0.0089$ & $0.0227 \pm 0.0006$ \\
        & Semi-sup. & $0.2668 \pm 0.0347$ & $0.2138 \pm 0.0167$ & $0.0208 \pm 0.0018$ \\
        & Oracle & $0.0705 \pm 0.0214$ & $0.1185 \pm 0.0200$ & $0.0053 \pm 0.0013$ \\
    \bottomrule
\end{tabular}
\end{table}

\subsection{Trellis-specific evaluation analysis}

As an additional evaluation analysis, Table~\ref{tab:trellis_excluding_test_sources_app} reports Trellis metrics when test source samples are excluded from the evaluation sets.

\begin{table}[H]
    \centering
    \small
    \caption{\textit{Trellis evaluation excluding test source samples.} Values are energy distance, SWD, and MMD-RBF (mean $\pm$ standard deviation, $\downarrow$) for IID replicate-holdout and OOD patient-holdout splits.}
    \label{tab:trellis_excluding_test_sources_app}
    \begin{tabular}{lccc}
    \toprule
    \textbf{Generator} & \textbf{Energy} & \textbf{SWD} & \textbf{MMD-RBF} \\
    \midrule
    \multicolumn{4}{c}{\textit{IID (Replicate Holdout)}} \\
    \midrule
    SWD    & $0.1191 \pm 0.0297$ & $0.1461 \pm 0.0199$ & $0.0087 \pm 0.0019$ \\
    Energy & $0.1200 \pm 0.0239$ & $0.1461 \pm 0.0159$ & $0.0087 \pm 0.0015$ \\
    FM     & $0.1526 \pm 0.0290$ & $0.1646 \pm 0.0140$ & $0.0112 \pm 0.0020$ \\
    \midrule
    \multicolumn{4}{c}{\textit{OOD (Patient Holdout)}} \\
    \midrule
    SWD    & $0.6188 \pm 0.1697$ & $0.3381 \pm 0.0484$ & $0.0454 \pm 0.0111$ \\
    Energy & $0.5843 \pm 0.1265$ & $0.3257 \pm 0.0369$ & $0.0430 \pm 0.0079$ \\
    FM     & $0.3424 \pm 0.0656$ & $0.2437 \pm 0.0226$ & $0.0256 \pm 0.0052$ \\
    \bottomrule
    \end{tabular}
\end{table}

\subsection{Latent-predictor diagnostics}\label{app:trellis_latent_predictors}

These diagnostics expand the latent-prediction comparisons used in the Trellis benchmark.
Table~\ref{tab:trellis_predictor_ablation} compares point predictors for the semi-supervised latent embedding, and Table~\ref{tab:trellis_latent_predictor_all_metrics} evaluates the distributional latent predictor for the Energy-generator STC model.
We focus this distributional latent-predictor analysis on Energy because it is the strongest or near-strongest continuous-data generator across the main real-data benchmarks.
Mixture-of-$k$ rows are deployable mixtures over sampled target latents, while best-of-$k$ rows are target-aided diagnostics that use the held-out target only for analysis.

\begin{table}[H]
    \centering
    \small
    \caption{\textit{Trellis latent point-predictor ablation.} Ridge, random forest, and MLP predictors are used to map source embeddings to target embeddings for semi-supervised STC models. Metrics are mean $\pm$ standard deviation ($\downarrow$) across SWD, Energy, and FM generators on IID replicate-holdout and OOD patient-holdout splits.}
    \label{tab:trellis_predictor_ablation}
    \begin{tabular}{llccc}
    \toprule
    \textbf{Generator} & \textbf{Predictor} & \textbf{Energy} & \textbf{SWD} & \textbf{MMD-RBF} \\
    \midrule
    \multicolumn{5}{c}{\textit{IID (Replicate Holdout)}} \\
    \midrule
    \multirow[t]{3}{*}{SWD}
        & Ridge & $0.1305 \pm 0.0305$ & $0.1524 \pm 0.0185$ & $0.0094 \pm 0.0018$ \\
        & Random Forest & $0.1530 \pm 0.0526$ & $0.1598 \pm 0.0295$ & $0.0106 \pm 0.0032$ \\
        & MLP & $0.1200 \pm 0.0524$ & $0.1415 \pm 0.0295$ & $0.0082 \pm 0.0032$ \\
    \midrule
    \multirow[t]{3}{*}{Energy}
        & Ridge & $0.1442 \pm 0.0520$ & $0.1589 \pm 0.0297$ & $0.0103 \pm 0.0032$ \\
        & Random Forest & $0.1515 \pm 0.0619$ & $0.1619 \pm 0.0351$ & $0.0107 \pm 0.0040$ \\
        & MLP & $0.1268 \pm 0.0583$ & $0.1434 \pm 0.0313$ & $0.0086 \pm 0.0037$ \\
    \midrule
    \multirow[t]{3}{*}{FM}
        & Ridge & $0.1450 \pm 0.0380$ & $0.1600 \pm 0.0166$ & $0.0106 \pm 0.0025$ \\
        & Random Forest & $0.1684 \pm 0.0505$ & $0.1704 \pm 0.0270$ & $0.0121 \pm 0.0033$ \\
        & MLP & $0.1436 \pm 0.0503$ & $0.1580 \pm 0.0237$ & $0.0097 \pm 0.0030$ \\
    \midrule
    \multicolumn{5}{c}{\textit{OOD (Patient Holdout)}} \\
    \midrule
    \multirow[t]{3}{*}{SWD}
        & Ridge & $0.2580 \pm 0.0403$ & $0.2102 \pm 0.0195$ & $0.0197 \pm 0.0022$ \\
        & Random Forest & $0.2792 \pm 0.0113$ & $0.2115 \pm 0.0121$ & $0.0216 \pm 0.0006$ \\
        & MLP & $0.2504 \pm 0.0248$ & $0.2054 \pm 0.0176$ & $0.0185 \pm 0.0011$ \\
    \midrule
    \multirow[t]{3}{*}{Energy}
        & Ridge & $0.2503 \pm 0.0556$ & $0.2093 \pm 0.0246$ & $0.0191 \pm 0.0049$ \\
        & Random Forest & $0.2497 \pm 0.0243$ & $0.2023 \pm 0.0189$ & $0.0190 \pm 0.0009$ \\
        & MLP & $0.2445 \pm 0.0026$ & $0.2072 \pm 0.0168$ & $0.0172 \pm 0.0006$ \\
    \midrule
    \multirow[t]{3}{*}{FM}
        & Ridge & $0.2668 \pm 0.0347$ & $0.2140 \pm 0.0170$ & $0.0208 \pm 0.0018$ \\
        & Random Forest & $0.2966 \pm 0.0385$ & $0.2188 \pm 0.0193$ & $0.0232 \pm 0.0012$ \\
        & MLP & $0.2614 \pm 0.0486$ & $0.2098 \pm 0.0259$ & $0.0201 \pm 0.0013$ \\
    \bottomrule
    \end{tabular}
\end{table}

\begin{table}[H]
\centering
\small
\caption{\textit{Trellis distributional latent-predictor diagnostic.} For the Energy-generator STC model, best-of-$k$ rows use the held-out target only for analysis, while mixture-of-$k$ rows are deployable mixtures over sampled target latents. Values are mean $\pm$ standard deviation ($\downarrow$) across metrics for IID replicate-holdout and OOD patient-holdout splits; the MSE column is the expectation latent predictor used for the standard semi-supervised row.}
\label{tab:trellis_latent_predictor_all_metrics}
\resizebox{\textwidth}{!}{%
\begin{tabular}{llcccccccc}
      \toprule
      \textbf{Metric} & \textbf{Regime} & $k=1$ & $k=2$ & $k=4$ & $k=8$ & $k=16$ & $k=32$ & \textbf{Oracle} & \textbf{MSE} \\
      \midrule
      \multicolumn{9}{c}{\textit{IID (Replicate Holdout)}} \\
      \midrule
      \multirow[t]{2}{*}{Energy} & Best-of-k & $0.1176 \pm 0.0310$ & $0.0934 \pm 0.0258$ & $0.0766 \pm 0.0209$ & $0.0647 \pm 0.0163$ & $0.0580 \pm
  0.0158$ & $0.0511 \pm 0.0135$ & \multirow{2}{*}{$0.0169 \pm 0.0026$} & \multirow{2}{*}{$0.1442 \pm 0.0520$} \\
       & Mixture-of-k & $0.1176 \pm 0.0310$ & $0.1057 \pm 0.0300$ & $0.0992 \pm 0.0295$ & $0.0960 \pm 0.0281$ & $0.0953 \pm 0.0285$ & $0.0940 \pm
  0.0279$ &  \\
      \addlinespace
      \multirow[t]{2}{*}{SWD} & Best-of-k & $0.1480 \pm 0.0203$ & $0.1323 \pm 0.0187$ & $0.1207 \pm 0.0167$ & $0.1118 \pm 0.0143$ & $0.1059 \pm 0.0141$
  & $0.1003 \pm 0.0127$ & \multirow{2}{*}{$0.0659 \pm 0.0037$} & \multirow{2}{*}{$0.1588 \pm 0.0297$} \\
       & Mixture-of-k & $0.1479 \pm 0.0203$ & $0.1405 \pm 0.0206$ & $0.1363 \pm 0.0208$ & $0.1341 \pm 0.0203$ & $0.1336 \pm 0.0207$ & $0.1329 \pm
  0.0205$ &  \\
      \addlinespace
      \multirow[t]{2}{*}{MMD-RBF} & Best-of-k & $0.0086 \pm 0.0020$ & $0.0068 \pm 0.0017$ & $0.0055 \pm 0.0013$ & $0.0047 \pm 0.0010$ & $0.0042 \pm
  0.0010$ & $0.0037 \pm 0.0009$ & \multirow{2}{*}{$0.0011 \pm 0.0002$} & \multirow{2}{*}{$0.0103 \pm 0.0032$} \\
       & Mixture-of-k & $0.0086 \pm 0.0020$ & $0.0076 \pm 0.0019$ & $0.0071 \pm 0.0019$ & $0.0069 \pm 0.0018$ & $0.0068 \pm 0.0018$ & $0.0067 \pm
  0.0018$ &  \\
      \midrule
      \multicolumn{9}{c}{\textit{OOD (Patient Holdout)}} \\
      \midrule
      \multirow[t]{2}{*}{Energy} & Best-of-k & $0.4920 \pm 0.0706$ & $0.4276 \pm 0.0607$ & $0.3829 \pm 0.0527$ & $0.3510 \pm 0.0477$ & $0.3243 \pm
  0.0430$ & $0.3039 \pm 0.0398$ & \multirow{2}{*}{$0.0315 \pm 0.0016$} & \multirow{2}{*}{$0.2503 \pm 0.0556$} \\
       & Mixture-of-k & $0.4920 \pm 0.0706$ & $0.4698 \pm 0.0695$ & $0.4600 \pm 0.0686$ & $0.4560 \pm 0.0693$ & $0.4531 \pm 0.0695$ & $0.4516 \pm
  0.0697$ &  \\
      \addlinespace
      \multirow[t]{2}{*}{SWD} & Best-of-k & $0.2982 \pm 0.0223$ & $0.2780 \pm 0.0207$ & $0.2631 \pm 0.0192$ & $0.2516 \pm 0.0181$ & $0.2418 \pm 0.0171$
  & $0.2340 \pm 0.0166$ & \multirow{2}{*}{$0.0809 \pm 0.0028$} & \multirow{2}{*}{$0.2091 \pm 0.0246$} \\
       & Mixture-of-k & $0.2982 \pm 0.0222$ & $0.2921 \pm 0.0226$ & $0.2897 \pm 0.0226$ & $0.2884 \pm 0.0227$ & $0.2876 \pm 0.0229$ & $0.2872 \pm
  0.0230$ &  \\
      \addlinespace
      \multirow[t]{2}{*}{MMD-RBF} & Best-of-k & $0.0360 \pm 0.0041$ & $0.0314 \pm 0.0035$ & $0.0282 \pm 0.0031$ & $0.0259 \pm 0.0028$ & $0.0240 \pm
  0.0025$ & $0.0225 \pm 0.0024$ & \multirow{2}{*}{$0.0022 \pm 0.0002$} & \multirow{2}{*}{$0.0191 \pm 0.0049$} \\
       & Mixture-of-k & $0.0360 \pm 0.0041$ & $0.0344 \pm 0.0041$ & $0.0337 \pm 0.0040$ & $0.0334 \pm 0.0041$ & $0.0331 \pm 0.0041$ & $0.0330 \pm
  0.0041$ &  \\
      \bottomrule
      \end{tabular}%
      }
\end{table}

On the IID split, increasing the number of sampled target embeddings improves the mixture's endpoint accuracy.
From $k=1$ to $k=32$, its Energy decreases from $0.1176$ to $0.0940$ and its MMD-RBF from $0.0086$ to $0.0067$; the corresponding supervised values are $0.0779$ and $0.0054$.
On patient holdout, the mixture has higher error than the expectation predictor across all three metrics.

\subsection{Targeted encoder and baseline comparisons}

Table~\ref{tab:trellis_baseline_encoder_summary_app} summarizes Trellis external baselines and encoder comparisons on the same raw MMD-RBF scale as the main text, including the exact MFM-FM baseline using the k-NN graph encoder. The fixed KME encoder is substantially worse than learned set encoders, while the Set Transformer encoder is comparable to the DeepSets encoder in the semi-supervised regime. Tables~\ref{tab:trellis_kme_metrics_app} and~\ref{tab:trellis_set_transformer_metrics_app} give the full metric breakdown for the KME and Set Transformer comparisons. Oracle-gap diagnostics are summarized in Appendix~\ref{app:oracle}.

\begin{table}[H]
\centering
\small
\caption{\textit{Trellis baseline and encoder summary.} External perturbation baselines, the MFM-FM source-conditioned baseline, and DCT encoder comparisons are placed on the same raw MMD-RBF scale as the main text. Values are mean $\pm$ standard deviation ($\downarrow$) on IID replicate-holdout and OOD patient-holdout splits; for DCT encoder rows, the Energy-generator run is shown. Dashes indicate regimes that are not applicable or were not run.}
\label{tab:trellis_baseline_encoder_summary_app}
\begin{tabular}{llccc}
\toprule
\textbf{Model} & \textbf{Split} & \textbf{SC / baseline} & \textbf{Semi-sup.} & \textbf{Oracle} \\
\midrule
scGen & IID & $0.0219 \pm 0.0057$ & -- & -- \\
      & OOD & $0.0265 \pm 0.0022$ & -- & -- \\
\midrule
CellOT & IID & $0.0198 \pm 0.0046$ & -- & -- \\
       & OOD & $0.0304 \pm 0.0033$ & -- & -- \\
\midrule
CellFlow & IID & $0.0212 \pm 0.0033$ & -- & -- \\
         & OOD & $0.0338 \pm 0.0018$ & -- & -- \\
\midrule
\multirow[t]{2}{*}{DeepSets (MSE)} & IID & $0.0054 \pm 0.0019$ & $0.0103 \pm 0.0032$ & $0.0011 \pm 0.0002$ \\
         & OOD & $0.0298 \pm 0.0019$ & $0.0191 \pm 0.0049$ & $0.0022 \pm 0.0002$ \\
\midrule
\multirow[t]{2}{*}{DeepSets (mix.)} & IID & -- & $0.0067 \pm 0.0018$ & -- \\
         & OOD & -- & $0.0330 \pm 0.0041$ & -- \\
\midrule
Set Transformer & IID & -- & $0.0092 \pm 0.0019$ & -- \\
                & OOD & -- & $0.0190 \pm 0.0012$ & -- \\
\midrule
KME & IID & -- & $0.1159 \pm 0.0118$ & -- \\
    & OOD & -- & $0.1124 \pm 0.0114$ & -- \\
\midrule
MFM-FM graph encoder & IID & $0.0088 \pm 0.0012$ & -- & -- \\
                     & OOD & $0.0235 \pm 0.0029$ & -- & -- \\
\bottomrule
\end{tabular}
\end{table}

\begin{table}[H]
\centering
\small
\caption{\textit{Trellis KME encoder comparison.} Semi-supervised STC models use a fixed kernel mean embedding representation in place of the standard learned set encoder. Values are energy distance, SWD, and MMD-RBF (mean $\pm$ standard deviation, $\downarrow$) for IID replicate-holdout and OOD patient-holdout splits.}
\label{tab:trellis_kme_metrics_app}
\begin{tabular}{lccc}
\toprule
\textbf{Generator} & \textbf{Energy} & \textbf{SWD} & \textbf{MMD-RBF} \\
\midrule
\multicolumn{4}{c}{\textit{IID (Replicate Holdout)}} \\
\midrule
SWD      & $2.1424 \pm 0.2856$ & $0.6303 \pm 0.0473$ & $0.1542 \pm 0.0199$ \\
Energy   & $1.7904 \pm 0.1825$ & $0.6489 \pm 0.0323$ & $0.1159 \pm 0.0118$ \\
FM       & $0.3485 \pm 0.0684$ & $0.2446 \pm 0.0293$ & $0.0251 \pm 0.0046$ \\
\midrule
\multicolumn{4}{c}{\textit{OOD (Patient Holdout)}} \\
\midrule
SWD      & $2.3059 \pm 0.3605$ & $0.6524 \pm 0.0532$ & $0.1661 \pm 0.0232$ \\
Energy   & $1.7688 \pm 0.2028$ & $0.6264 \pm 0.0301$ & $0.1124 \pm 0.0114$ \\
FM       & $0.3674 \pm 0.0423$ & $0.2434 \pm 0.0212$ & $0.0284 \pm 0.0020$ \\
\bottomrule
\end{tabular}
\end{table}

\begin{table}[H]
\centering
\small
\caption{\textit{Trellis Set Transformer encoder comparison.} Semi-supervised STC models use a Set Transformer representation in place of the standard learned DeepSets encoder. Values are energy distance, SWD, and MMD-RBF (mean $\pm$ standard deviation, $\downarrow$) for IID replicate-holdout and OOD patient-holdout splits.}
\label{tab:trellis_set_transformer_metrics_app}
\begin{tabular}{lccc}
\toprule
\textbf{Generator} & \textbf{Energy} & \textbf{SWD} & \textbf{MMD-RBF} \\
\midrule
\multicolumn{4}{c}{\textit{IID (Replicate Holdout)}} \\
\midrule
SWD      & $0.1220 \pm 0.0275$ & $0.1479 \pm 0.0186$ & $0.0089 \pm 0.0017$ \\
Energy   & $0.1275 \pm 0.0298$ & $0.1520 \pm 0.0200$ & $0.0092 \pm 0.0019$ \\
FM       & $0.1904 \pm 0.0772$ & $0.1754 \pm 0.0336$ & $0.0139 \pm 0.0051$ \\
\midrule
\multicolumn{4}{c}{\textit{OOD (Patient Holdout)}} \\
\midrule
SWD      & $0.2567 \pm 0.0137$ & $0.2080 \pm 0.0096$ & $0.0195 \pm 0.0010$ \\
Energy   & $0.2463 \pm 0.0163$ & $0.2033 \pm 0.0103$ & $0.0190 \pm 0.0012$ \\
FM       & $0.2574 \pm 0.0328$ & $0.2073 \pm 0.0175$ & $0.0200 \pm 0.0019$ \\
\bottomrule
\end{tabular}
\end{table}

The Trellis perturbation experiments also include a source-conditioned encoder comparison against the k-NN graph encoder used in Meta Flow Matching. This is a supervised SC comparison rather than an STC encoder swap, but it addresses the same encoder-sensitivity question. Performance is close to the main source-conditioned Trellis results in Table~\ref{tab:trellis_combined}, suggesting that the Trellis conclusions are not sensitive to this encoder choice.

\begin{table}[H]
    \centering
    \small
    \caption{\textit{Trellis source-conditioned graph-encoder comparison.} Source-conditioned models use the k-NN graph encoder from Meta Flow Matching rather than the standard DCT set encoder. Metrics compare SWD, Energy, and FM generators on IID replicate-holdout and OOD patient-holdout splits (mean $\pm$ standard deviation, $\downarrow$).}
    \label{tab:trellis_mfm_app}
    \begin{tabular}{lccc}
    \toprule
    \textbf{Generator} & \textbf{Energy} & \textbf{SWD} & \textbf{MMD-RBF} \\
    \midrule
    \multicolumn{4}{c}{\textit{IID (Replicate Holdout)}} \\
    \midrule
    SWD & $0.0807 \pm 0.0253$ & $0.1238 \pm 0.0176$ & $0.0058 \pm 0.0016$ \\
    Energy & $0.0836 \pm 0.0230$ & $0.1254 \pm 0.0145$ & $0.0059 \pm 0.0014$ \\
    FM & $0.1235 \pm 0.0194$ & $0.1477 \pm 0.0091$ & $0.0088 \pm 0.0012$ \\
    \midrule
    \multicolumn{4}{c}{\textit{OOD (Patient Holdout)}} \\
    \midrule
    SWD & $0.3524 \pm 0.0299$ & $0.2493 \pm 0.0111$ & $0.0268 \pm 0.0018$ \\
    Energy & $0.3595 \pm 0.0058$ & $0.2512 \pm 0.0009$ & $0.0274 \pm 0.0007$ \\
    FM & $0.3071 \pm 0.0427$ & $0.2288 \pm 0.0140$ & $0.0235 \pm 0.0029$ \\
    \bottomrule
\end{tabular}
\end{table}

\paragraph{External baselines.}
The Trellis perturbation benchmark also includes method-specific baselines for single-cell perturbation prediction. These baselines are included in the compact IID/OOD summary in Table~\ref{tab:trellis_baseline_encoder_summary_app}; Table~\ref{tab:trellis_external_baselines_app} gives the full IID/OOD metric breakdown in a single format. For scGen \citep{Lotfollahi2019scGen}, we train the VAE with default parameters and compute one latent difference vector per drug perturbation. For CellOT \citep{Bunne2023-cd}, we group cells by drug perturbation and train a separate instance for each perturbation. For CellFlow \citep{klein2025cellflow}, we condition on drug perturbation with a one-hot encoding, since the evaluation splits do not require generalization to unseen drugs. None of these baselines uses patient-level source-distribution conditioning in the same way as DCT, which likely explains their weaker OOD patient-holdout performance.

\begin{table}[H]
    \centering
    \small
    \caption{\textit{External Trellis perturbation baselines.} scGen, CellOT, and CellFlow are evaluated on the same Trellis PDO perturbation splits used for DCT models. Energy, SWD, and MMD-RBF are reported as mean $\pm$ standard deviation ($\downarrow$) for IID replicate-holdout and OOD patient-holdout splits.}
    \label{tab:trellis_external_baselines_app}
    \begin{tabular}{llccc}
    \toprule
    \textbf{Baseline} & \textbf{Split} & \textbf{Energy} & \textbf{SWD} & \textbf{MMD-RBF} \\
    \midrule
    scGen & IID & 0.2924{\scriptsize$\pm$0.0881} & 0.2189{\scriptsize$\pm$0.0383} & 0.0219{\scriptsize$\pm$0.0057} \\
          & OOD & 0.3278{\scriptsize$\pm$0.0345} & 0.2272{\scriptsize$\pm$0.0162} & 0.0265{\scriptsize$\pm$0.0022} \\
    \midrule
    CellOT & IID & 0.2795{\scriptsize$\pm$0.0760} & 0.2155{\scriptsize$\pm$0.0366} & 0.0198{\scriptsize$\pm$0.0046} \\
           & OOD & 0.3990{\scriptsize$\pm$0.0181} & 0.2499{\scriptsize$\pm$0.0047} & 0.0304{\scriptsize$\pm$0.0033} \\
    \midrule
    CellFlow & IID & 0.2968{\scriptsize$\pm$0.0574} & 0.2305{\scriptsize$\pm$0.0292} & 0.0212{\scriptsize$\pm$0.0033} \\
             & OOD & 0.4600{\scriptsize$\pm$0.0141} & 0.2828{\scriptsize$\pm$0.0102} & 0.0338{\scriptsize$\pm$0.0018} \\
    \bottomrule
    \end{tabular}
\end{table}

\subsection{Semi-supervised pairing diagnostics}\label{app:trellis_pairing}

Table~\ref{tab:trellis_paired_only_ablation_app} reports the paired-only Trellis ablation for the Energy-generator STC model. The paired-only run is close on the IID split and slightly worse on the OOD split than the all-data semi-supervised Energy row in Table~\ref{tab:trellis_combined}; the degradation is milder than in the lineage-tracing paired-only diagnostic. In Table~\ref{tab:trellis_pairing_policy}, we show the performance of the semi-supervised models trained with a more restrictive pairing policy. In particular, rather than using all pairs of source and target at training time, we only use control samples for the source distributions and perturbed samples for the target distributions. The OOD performance of the restricted pairing policy models is worse than the all-pairing model.

\begin{table}[H]
    \centering
    \small
    \caption{\textit{Trellis semi-supervised paired-only ablation.} The Energy-generator STC model is trained using only paired source--target examples. Values are energy distance, SWD, and MMD-RBF (mean $\pm$ standard deviation, $\downarrow$) for IID replicate-holdout and OOD patient-holdout splits.}
    \label{tab:trellis_paired_only_ablation_app}
    \begin{tabular}{lccc}
    \toprule
    \textbf{Generator} & \textbf{Energy} & \textbf{SWD} & \textbf{MMD-RBF} \\
    \midrule
    \multicolumn{4}{c}{\textit{IID (Replicate Holdout)}} \\
    \midrule
    Energy & $0.1340 \pm 0.0480$ & $0.1495 \pm 0.0269$ & $0.0096 \pm 0.0030$ \\
    \midrule
    \multicolumn{4}{c}{\textit{OOD (Patient Holdout)}} \\
    \midrule
    Energy & $0.2633 \pm 0.0167$ & $0.2126 \pm 0.0086$ & $0.0200 \pm 0.0016$ \\
    \bottomrule
    \end{tabular}
\end{table}

\begin{table}[H]
    \centering
    \small
    \caption{\textit{Trellis pairing policy ablation.} Performance of semi-supervised models trained with pairing policy such that only control samples are used for source distributions and only perturbed samples are used for target distributions.}
    \begin{tabular}{lccc}
    \toprule
    \textbf{Generator} & \textbf{Energy} & \textbf{SWD} & \textbf{MMD-RBF} \\
    \midrule
    \multicolumn{4}{c}{\textit{IID (Replicate Holdout)}} \\
    \midrule
    Energy & $0.1285 \pm 0.0313$ & $0.1503 \pm 0.0187$ & $0.0092 \pm 0.0019$ \\
    \midrule
    \multicolumn{4}{c}{\textit{OOD (Patient Holdout)}} \\
    \midrule
    Energy & $0.6243 \pm 0.0747$ & $0.3348 \pm 0.0198$ & $0.0457 \pm 0.0058$ \\
    \bottomrule
    \end{tabular}
    \label{tab:trellis_pairing_policy}
\end{table}

\section{Experimental details for lineage-traced scRNA-seq}\label{app:lt}

\noindent This appendix provides experimental details for the lineage-traced scRNA-seq forecasting experiment in Section~\ref{subsec:lt}. We describe dataset construction, model configurations, and evaluation. Generator families and metrics are defined in Appendices~\ref{app:standard_generators} and~\ref{app:standard_metrics}.

\subsection{Data preprocessing}

We use the \textit{in vitro} lineage-traced scRNA-seq dataset from \citet{Weinreb2020-vh}, which tracks hematopoietic stem cell differentiation over time points $t \in \{2, 4, 6\}$ days.

We download normalized count matrices, metadata, gene names, and clone assignment matrices from the paper data repository\footnote{\url{https://kleintools.hms.harvard.edu/paper_websites/state_fate2020/}}. We discard all cells which have not been assigned to a clone.

\paragraph{Gene expression preprocessing.}
Following standard scRNA-seq preprocessing pipelines \citep{Heumos2023-ez}, we apply:
\begin{enumerate}
    \itemsep0em 
    \item Normalization to $10^4$ counts per cell
    \item $\log(x + 1)$ transform
    \item Feature selection of top 10,000 highly variable genes
    \item Unit variance, zero mean rescaling per gene with maximum value clipped at 10
    \item PCA to $d = 50$ principal components
\end{enumerate}
The distributional endpoint comparison uses a transform fit on all cells and shared across the semi-supervised, supervised, and LineageOT models. No held-out clone enters the paired training data or either endpoint of an auxiliary pair. For the fate-composition analysis, highly variable genes, scaling moments, and principal components are fit only on the 36,400 cells from non-held-out clones and then applied to the 12,897 held-out-clone cells.

\paragraph{Clone set construction.}
For each (clone, timepoint) pair, we partition cells into sets of size $n = 100$ with minimum $n_{\min} = 3$ cells per set. Sets smaller than $n$ are padded by sampling with replacement from the same (clone, timepoint) group. This yields $N$ sets $\{S_i\}_{i=1}^N$. 

\paragraph{Dataset sampling (supervised).}
We construct source-target pairs $(S_i^{\text{src}}, S_i^{\text{tgt}})$ where both sets share the same clone identity but differ by $\Delta t = 2$ days (i.e., $t \in \{2, 4\}$ paired with $t + 2 \in \{4, 6\}$). This results in 1218 clones observed at consecutive timepoints, split into 609 training and 609 test clones. This split is at the level of clones, not cells -- each clone appears in only part of the split.

\paragraph{Dataset sampling (semi-supervised).}
For the semi-supervised setting, we additionally construct unstructured pairs by matching source sets from clones without complete targets to random target sets at the correct timepoint. During training, each batch samples from the true paired distribution with probability $p = 0.25$ and from the random pairing with probability $1 - p = 0.75$. The auxiliary source and target pools exclude every held-out clone: they contain 103,500 candidate pairs with zero overlap with the 609 evaluation clones.

\subsection{Model configurations}

We evaluate the three standard continuous-data generator families described in Appendix~\ref{app:standard_generators}: flow matching (FM), energy/MMD regression (Energy), and sliced Wasserstein regression (SWD). For each generator type, we compare a supervised source-conditioned (SC) baseline trained only on paired clones with a semi-supervised source--target-conditioned (STC) model trained on a mixture of true pairs and random early$\to$late pairings. The deployed STC models use a post-hoc ridge predictor for the target embedding; target-aided evaluations appear separately in Appendix~\ref{app:oracle}.
Unless stated otherwise, the distribution encoder uses the standard architecture described in Appendix~\ref{app:standard_encoder}.
We evaluate all SC and STC models at their final epoch 1,000.

For the semi-supervised setting, the default target-embedding predictor is a post-hoc RidgeCV model fit on paired training embeddings, as described in Appendix~\ref{app:latent_predictors}. In the lineage-tracing runs, ridge regularization is selected by 5-fold cross-validation over $\alpha \in [10^{-6}, 10^6]$.

\subsection{Evaluation protocol}

We evaluate generated populations against held-out target clones using MMD-RBF, Energy, and SWD (Appendix~\ref{app:standard_metrics}). For this experiment, Energy uses half the biased energy distance, and SWD averages the one-dimensional $W_2$ distances over 50 fixed unit Gaussian projections per pair. MMD-RBF uses the biased squared estimator with bandwidth set by the median positive pooled pairwise distance. All methods use these same conventions. The distributional table reports mean $\pm$ standard error over 628 empirical set pairs from 609 clones, matching the aggregation used for the supervised and LineageOT baselines; clone-macro aggregation gives the same ordering.

\subsection{Full results across metrics}

Table~\ref{tab:lt} reports MMD-RBF for the main transport comparison and cell-type abundance MSE for the semi-supervised DCT models. Table~\ref{tab:app_lt} adds Energy and SWD.

\begin{table}[H]
    \centering
    \small
    \begin{tabular}{llccc}
    \toprule
    \textbf{Generator} & \textbf{Regime} & \textbf{MMD-RBF} $\downarrow$ & \textbf{Energy} $\downarrow$ & \textbf{SWD} $\downarrow$ \\
    \midrule
    LineageOT & External baseline & $0.2865 \pm 0.0067$ & $6.034 \pm 0.164$ & $1.932 \pm 0.028$ \\
    \midrule
    \multirow[t]{4}{*}{FM} & Supervised (SC) & $0.3085 \pm 0.0056$ & $7.281 \pm 0.141$ & $2.306 \pm 0.025$ \\
     & Semi-supervised (STC, MSE) & $0.2369 \pm 0.0052$ & $5.028 \pm 0.123$ & $1.778 \pm 0.023$ \\
     & Semi-supervised (STC, mix.) & $0.1961 \pm 0.0042$ & $4.061 \pm 0.101$ & $1.779 \pm 0.023$ \\
     & Oracle STC & $0.1362 \pm 0.0029$ & $3.158 \pm 0.071$ & $1.448 \pm 0.020$ \\
    \midrule
    \multirow[t]{4}{*}{SWD} & Supervised (SC) & $0.2536 \pm 0.0056$ & $4.727 \pm 0.116$ & $1.642 \pm 0.022$ \\
     & Semi-supervised (STC, MSE) & $0.2249 \pm 0.0053$ & $4.374 \pm 0.108$ & $1.586 \pm 0.021$ \\
     & Semi-supervised (STC, mix.) & $0.2003 \pm 0.0047$ & $3.708 \pm 0.100$ & $1.553 \pm 0.022$ \\
     & Oracle STC & $0.1158 \pm 0.0029$ & $2.332 \pm 0.046$ & $1.152 \pm 0.013$ \\
    \midrule
    \multirow[t]{4}{*}{Energy} & Supervised (SC) & $0.2524 \pm 0.0055$ & $4.735 \pm 0.118$ & $1.651 \pm 0.022$ \\
     & Semi-supervised (STC, MSE) & $0.2289 \pm 0.0055$ & $4.433 \pm 0.110$ & $1.589 \pm 0.021$ \\
     & Semi-supervised (STC, mix.) & $0.1944 \pm 0.0047$ & $3.614 \pm 0.102$ & $1.536 \pm 0.021$ \\
     & Oracle STC & $0.1180 \pm 0.0031$ & $2.332 \pm 0.046$ & $1.155 \pm 0.013$ \\
    \bottomrule
    \end{tabular}
    \caption{\textit{Lineage-tracing forecasting metrics.} Held-out clonal populations are evaluated across supervised SC and semi-supervised STC regimes, with LineageOT shown as an external paired-data baseline. Oracle STC uses true target embeddings. Values are mean $\pm$ standard error over 628 empirical set pairs ($\downarrow$). MSE rows use the expectation latent predictor; mixture rows draw a fixed 100-cell endpoint from the equal-weight mixture over 32 sampled target embeddings.}
    \label{tab:app_lt}
\end{table}

\subsection{Lineage cell-type-composition forecasting}

Following \citet{Weinreb2020-vh}, we represent a clone's future fate by the distribution of its cells across annotated hematopoietic lineages. PRESCIENT likewise labels simulated terminal cells with a nearest-neighbor classifier and evaluates the resulting clonal fate bias against lineage tracing \citep{Yeo2021-zo}. Table~\ref{tab:lt_fate_composition} evaluates the predicted composition using MSE and TV. The analysis uses versions of the three semi-supervised models from Table~\ref{tab:app_lt} with preprocessing fit on non-held-out clones and target embeddings predicted by Ridge from the source. Observed compositions use the released cell-type annotations over the complete 11-cell-type vocabulary. Generated cells are labeled by a 15-nearest-neighbor classifier fit without any held-out clone, which attains $94.7\%$ accuracy on the 12,897 cells from the held-out clones.

The evaluation deduplicates source sets repeated by source--target chunking and contains 626 unique clone--time transitions from 609 held-out clones. Each clone--transition is one evaluation unit.

\begin{table}[H]
\centering
\small
\caption{\textit{Lineage cell-type-composition forecasting.} Values are mean $\pm$ standard error across 626 held-out clone--time transitions. Arrows mark the preferred direction. DCT uses predicted target embeddings.}
\label{tab:lt_fate_composition}
\begin{tabular}{lcc}
\toprule
\textbf{Method} & \textbf{MSE} $\downarrow$ & \textbf{TV} $\downarrow$ \\
\midrule
DCT (FM) & $0.0213 \pm 0.0012$ & $0.270 \pm 0.010$ \\
DCT (SWD) & $0.0193 \pm 0.0013$ & $0.248 \pm 0.010$ \\
DCT (Energy) & $\mathbf{0.0187 \pm 0.0012}$ & $\mathbf{0.247 \pm 0.009}$ \\
\bottomrule
\end{tabular}
\end{table}

\subsection{Latent-predictor diagnostics}\label{app:lt_latent_predictors}

The Energy model isolates the error introduced when an unobserved target embedding must be predicted from the source. Both predictors are trained on the same 631 clone-disjoint paired rows. The expectation predictor minimizes squared error in embedding space. The distributional predictor instead learns a conditional distribution over target-embedding displacements and draws 32 candidates at inference. Every endpoint comparison uses 100 generated cells.

\begin{table}[H]
\centering
\small
\caption{\textit{Lineage-tracing target-embedding diagnostics.} Deployed predictors use only the source clone. Target-aided diagnostics access the held-out target. Values are energy distance over 628 empirical pairs from 609 held-out clones (mean $\pm$ standard error, $\downarrow$).}
\label{tab:lt_distributional_latent_predictor}
\begin{tabular}{llc}
\toprule
\textbf{Target embedding} & \textbf{Access} & \textbf{Energy} \\
\midrule
Expectation (MSE) & Source only & $4.433 \pm 0.110$ \\
Mixture of 32 samples & Source only & $3.614 \pm 0.102$ \\
Best of 32 samples & Held-out target & $3.380 \pm 0.074$ \\
True target embedding & Held-out target & $2.332 \pm 0.046$ \\
\bottomrule
\end{tabular}
\end{table}

The deployable mixture improves on the expectation predictor, reducing energy distance from $4.433$ to $3.614$. Selecting the best sampled endpoint with the held-out target lowers it further to $3.380$, and supplying the true target embedding reaches $2.332$. Together these results show that target-embedding prediction remains a substantial source of error.

\subsection{Orphan-marginal ablation}\label{app:pairing}

The paired-only ablation compares conditioning architectures and the addition of orphan marginals with auxiliary pair sampling. The paired-only SC and STC models use the same 631 paired training rows. The all-marginal STC model uses a mixture of 25\% observed pairs and 75\% forward-time pairs among training marginals, while excluding every evaluation clone from the auxiliary pair pool.

\begin{table}[H]
\centering
\small
\caption{\textit{Lineage-tracing orphan-marginal ablation.} Energy-distance DCT models trained on paired clones only or with additional orphan marginals. Values are energy distance over 628 empirical pairs from 609 held-out clones (mean $\pm$ standard error, $\downarrow$).}
\label{tab:lt_paired_only_ablation_app}
\begin{tabular}{lc}
\toprule
\textbf{Model} & \textbf{Energy} \\
\midrule
SC, paired only & $4.735 \pm 0.118$ \\
STC, paired only & $4.961 \pm 0.121$ \\
STC, with orphan marginals & $\mathbf{4.433 \pm 0.110}$ \\
\bottomrule
\end{tabular}
\end{table}

Changing from SC to STC does not improve the paired-only result. Holding the STC architecture fixed, adding orphan marginals together with auxiliary pair sampling reduces energy distance by 10.6\%. This comparison measures the combined change in data access and pair law; it does not isolate their separate effects.

\section{Longitudinal TCR repertoire forecasting}\label{app:tcr}

\noindent This appendix describes the dataset, model architectures, and evaluation protocol for the longitudinal T-cell receptor experiment in Section~\ref{subsec:tcr}. Later timepoints from three evaluation patients are held out, while each patient's baseline repertoire remains available in STC training as an orphan marginal. The task is to forecast the next repertoire from the observed source repertoire.

\subsection{Data preprocessing}\label{app:tcr_data}

\subsubsection{Source Data}

We use longitudinal TCR repertoire sequencing data from COVID-19 patients collected by \citet{schultheiss2020next}. The dataset contains TRB (T-cell receptor beta chain) CDR3 amino acid sequences from 37 patients, of whom 10 were profiled at multiple timepoints during disease progression and recovery. Each repertoire comprises thousands of CDR3 junction observations, represented with their empirical multiplicities and typically 10--20 amino acids in length.

\subsubsection{Sequence Filtering}

Raw CDR3 junction sequences are filtered to retain only those composed of the 20 standard amino acids (ACDEFGHIKLMNPQRSTVWY). Sequences containing ambiguous or non-standard residues are discarded. We truncate or pad sequences to a maximum length of 30 tokens (including special tokens), which accommodates the vast majority of TCR CDR3 sequences.

\subsubsection{Train-test split}

We partition the 10 multi-timepoint patients into seven longitudinal training patients and three evaluation patients. For each evaluation patient, the first timepoint is available to the semi-supervised STC model as an orphan marginal; all later timepoints are held out as forecast targets. The 27 single-timepoint patients are also included only in STC training as orphan marginals. The evaluation therefore tests future-timepoint prediction for patients whose baseline repertoire is observed, not generalization to a wholly unseen patient.

The SC model uses only consecutive timepoint pairs $(t,t{+}1)$ from the seven longitudinal training patients. The STC model uses every ordered pair of distinct repertoires available in its training split. This any-to-any pair law incorporates the single-timepoint repertoires and observed evaluation-patient baselines without using any held-out future repertoire.

\subsubsection{Tokenization}

Each sequence is tokenized for both ESM2 and ProGen2 vocabularies and cached for efficient training. ESM2 tokenization follows the standard protocol with \texttt{<cls>} and \texttt{<eos>} special tokens, producing input sequences of the form $\langle\texttt{cls}\rangle\texttt{SEQUENCE}\langle\texttt{eos}\rangle\langle\texttt{pad}\rangle\cdots$. ProGen2 tokenization uses \texttt{<|bos|>} and \texttt{<|eos|>} delimiters with \texttt{<|pad|>} for padding.

\subsection{Embedding architecture}\label{app:tcr_embedding}

All models use the same distribution-encoder architecture, with separately trained parameters for each generator and conditioning regime. The encoder operates in two stages: a pretrained protein language model extracts per-sequence representations, which are then aggregated into a distribution embedding.

\subsubsection{Sequence Encoder}

We use ESM2 \citep{lin2023evolutionary} with 8 million parameters as the backbone sequence encoder. For each sequence, we extract the final hidden states and apply mean pooling over non-padding positions to obtain a 320-dimensional sequence embedding. The backbone is initialized from pretrained weights and updated jointly with the distribution encoder and generator.

\subsubsection{Distribution Encoder}

Given a set of $n$ sequence embeddings $\{h_1, \ldots, h_n\} \subset \mathbb{R}^{320}$, we apply a permutation-invariant self-attention aggregator \citep{lee2019set}. The aggregator consists of:
\begin{enumerate}
    \item An input projection $\texttt{Linear}(320, 256) + \texttt{SELU}$
    \item Two self-attention layers with 4 heads each
    \item Mean pooling across sequences
    \item A latent projection $\texttt{Linear}(256, 128) + \texttt{SELU}$
    \item $\ell_2$ normalization of the final embedding
\end{enumerate}
The output is a 128-dimensional distribution embedding $z \in \mathbb{R}^{128}$ that summarizes the entire repertoire.

\subsection{Generator architectures}\label{app:tcr_generators}

We compare two discrete sequence generators with the same conditioning interface: ProGen2 samples autoregressively, whereas DFM integrates a discrete flow.

\subsubsection{ESM2 Discrete Flow Matching (DFM)}

The DFM generator adapts the discrete flow matching framework of \citet{gat2024discreteflowmatching} to protein sequence generation by repurposing ESM2 as a denoising backbone. It iteratively refines all sequence positions from source to target through a learned flow over discrete tokens.

\paragraph{Architecture.} We extend \texttt{EsmForMaskedLM} with conditioning modules that inject source and target distribution embeddings at two points:
\begin{enumerate}
    \item \textbf{Input-level conditioning:} The concatenated latent $[z_\text{src}; z_\text{tgt}] \in \mathbb{R}^{256}$ is projected through a two-layer MLP ($\texttt{Linear}(256, 32) \to \texttt{GELU} \to \texttt{Linear}(32, 320)$) and added to the token embeddings before the transformer encoder. This ensures all 6 transformer layers process conditioned representations.
    \item \textbf{Output-level conditioning:} A parallel projection is added to the hidden states after the transformer, before the language modeling head.
\end{enumerate}
Time conditioning uses sinusoidal embeddings following the original transformer formulation, with the continuous time $t \in [0,1]$ embedded into the model's hidden dimension and added to token embeddings.

\paragraph{Training.} Given source sequence $x^\text{src}$ and target sequence $x^\text{tgt}$, we sample $t \sim \text{Uniform}(0, 1)$ and construct an interpolated sequence $x_t$ via Bernoulli masking: at each position, $x_t$ equals $x^\text{tgt}$ with probability $t$ and $x^\text{src}$ otherwise. The model predicts $x^\text{tgt}$ from $x_t$ using cross-entropy loss over all positions, including padding tokens. This allows the model to learn length changes (insertions and deletions) between source and target.

\paragraph{Sampling.} Generation proceeds by discrete flow integration from $t=0$ to $t=1$ with step size $\Delta t = 0.05$. At each step, the model predicts a distribution over tokens at each position. The update uses a mixture formula:
\begin{equation}
    p(x_{t+\Delta t}) = \mathbf{1}_{x_t} + \frac{\Delta t}{1-t}\left(p_\theta(x^\text{tgt} \mid x_t, t) - \mathbf{1}_{x_t}\right)
\end{equation}
where $\mathbf{1}_{x_t}$ is the one-hot encoding of the current sequence. Vocabulary constraints enforce that position 0 contains only \texttt{<cls>} and subsequent positions contain only amino acids, \texttt{<eos>}, or \texttt{<pad>}.

\subsubsection{ProGen2 with Prefix Conditioning}

The ProGen2 generator uses the pretrained ProGen2 language model \citep{Nijkamp2023-xq} with prefix-based conditioning for source-to-target generation.

\paragraph{Architecture.} We use \texttt{progen2-small} as the backbone and update its parameters during training. Conditioning is implemented by inserting a learned prefix embedding between source and target sequences:
\begin{enumerate}
    \item The concatenated latent $[z_\text{src}; z_\text{tgt}]$ is projected through $\texttt{Linear}(256, 256) \to \texttt{GELU} \to \texttt{Linear}(256, d_\text{model})$
    \item The resulting vector is inserted as a single prefix token between the source and target portions of the input
    \item The attention mask is adjusted to allow the prefix to attend to source tokens and target tokens to attend to the prefix
\end{enumerate}

\paragraph{Training.} We concatenate the source and target as
$$
\langle\texttt{bos}\rangle\texttt{SOURCE}\langle\texttt{eos}\rangle
\langle\texttt{bos}\rangle\texttt{TARGET}\langle\texttt{eos}\rangle.
$$
The prefix is inserted at the boundary. The causal language modeling loss is computed only on target tokens, excluding the source and prefix positions.

\paragraph{Sampling.} Given the source sequence and supplied source and target distribution embeddings, the model samples target tokens autoregressively until \texttt{<eos>} or the maximum length. Both embeddings are projected to the unit sphere before generation, matching the support on which the generator is trained. Sequence diversity comes from token sampling at temperature 1; the distribution embeddings are held fixed across samples from the same forecast.

\subsection{Evaluation protocol}\label{app:tcr_eval}

All models are evaluated on the same prediction problem: given a patient's repertoire at time $t$, forecast the held-out repertoire at time $t{+}1$. We evaluate only on consecutive timepoint pairs from the three evaluation patients.

\subsubsection{Embedding computation}

During generator training, each ProGen2 set contains 256 sequences and each DFM set contains 512. For target-embedding prediction and endpoint evaluation, we sample 512 sequences uniformly from each repertoire and apply the corresponding trained encoder.

\subsubsection{Target-embedding prediction}

For each STC model, we fit one multi-output ridge regression to the change in embedding,
$$
\widehat z_{t+1}=z_t+\widehat{\Delta z}_t,
\qquad
\widehat{\Delta z}_t=Wz_t+b,
$$
using the 15 consecutive transitions from the seven longitudinal training patients. We choose one shared penalty from 20 log-spaced values between $10^{-4}$ and $10^4$ by leave-one-patient-out cross-validation. The validation loss is squared distance after normalizing the predicted and observed target embeddings, averaged first within patient and then across patients. This selects $\alpha=78.476$ for ProGen2 and $\alpha=11.288$ for DFM. We then refit each predictor on all seven training patients and normalize its predicted target embedding before generation.

\subsubsection{Sequence generation and comparison}

The test set contains 17 consecutive-timepoint pairs from three patients. For each pair, we encode the current repertoire as $z_t$ and generate 512 sequences with either the separately trained SC generator or the STC generator conditioned on $\widehat z_{t+1}$. For each generator, the target-aided oracle diagnostic instead supplies the observed target embedding $z_{t+1}$ to the same STC checkpoint. We re-tokenize the generated sequences for ESM2, compute their mean-pooled embeddings, and compare them with the ESM2 embeddings of 512 sequences sampled from the observed next-timepoint repertoire. Endpoint sampling uses seed 42.

\subsubsection{Metrics}

We report Energy, SWD, Sinkhorn divergence, and MMD-RBF in frozen ESM2 embedding space. SWD uses the root mean squared projection error from Appendix~\ref{app:standard_metrics}, with 50 unit Gaussian directions shared across models for each transition. Sinkhorn uses the GeomLoss estimator with blur $0.01$ and scaling $0.9$. MMD-RBF uses the biased quadratic estimator with fixed bandwidth $\sigma=1$.

Because the 17 transitions are repeated observations from only three patients, descriptive endpoint summaries first average transitions within each patient and then weight the three patients equally.

\subsection{Hyperparameter summary}\label{app:tcr_hyperparams}

Table~\ref{tab:tcr_hyperparams} summarizes the key hyperparameters for reproducibility.

\begin{table}[H]
\centering
\caption{\textit{TCR forecasting hyperparameters.} Dataset, encoder, generator, training, and evaluation settings for DFM and ProGen2 in the held-out future-timepoint task.}
\label{tab:tcr_hyperparams}
\begin{tabular}{lcc}
\toprule
\textbf{Hyperparameter} & \textbf{DFM} & \textbf{ProGen2} \\
\midrule
\multicolumn{3}{l}{\textit{Dataset}} \\
Sequences per set (training) & 512 & 256 \\
Maximum sequence length & 30 & 30 \\
Test fraction (multi-TP patients) & 0.3 & 0.3 \\
\midrule
\multicolumn{3}{l}{\textit{Encoder}} \\
ESM2 backbone & \texttt{esm2\_t6\_8M} & \texttt{esm2\_t6\_8M} \\
Distribution embedding dim & 128 & 128 \\
Self-attention layers & 2 & 2 \\
Attention heads & 4 & 4 \\
\midrule
\multicolumn{3}{l}{\textit{Generator}} \\
Backbone & \texttt{esm2\_t6\_8M} & \texttt{progen2-small} \\
Backbone updated & Yes & Yes \\
Conditioning hidden dimension & 32 & 256 \\
Concatenated latent dimension & 256 & 256 \\
Temperature & 1.0 & 1.0 \\
Sample step size ($\Delta t$) & 0.05 & --- \\
\midrule
\multicolumn{3}{l}{\textit{Training}} \\
Batch size & 2 & 1 \\
Learning rate & $2 \times 10^{-5}$ & $2 \times 10^{-5}$ \\
Epochs & 20{,}000 & 5{,}000 \\
Optimizer & Adam & Adam \\
\midrule
\multicolumn{3}{l}{\textit{Evaluation}} \\
Sequences per distribution embedding & 512 & 512 \\
Generated sequences per test pair & 512 & 512 \\
Held-out transitions & 17 & 17 \\
Predictor regularization & Patient-grouped ridge & Patient-grouped ridge \\
\bottomrule
\end{tabular}
\end{table}

\subsection{Full results across metrics}

\begin{table}[H]
\centering
\small
\setlength{\tabcolsep}{2.5pt}
\caption{\textit{TCR repertoire forecasting metrics.} Every row uses the same 512 sampled source sequences, 512 target sequences, and 512 outputs generated one per source prompt in each transition. Values are mean $\pm$ standard error across three equally weighted patient means ($\downarrow$). Oracle STC embeds the held-out target and is a target-aided diagnostic, not a feasible forecast.}
\label{tab:app_tcr}
\begin{tabular}{llcccc}
\toprule
\textbf{Generator} & \textbf{Regime} & \textbf{Energy} & \textbf{SWD} & \textbf{Sinkhorn} & \textbf{MMD-RBF} \\
\midrule
\multirow[t]{3}{*}{DFM}
  & Source-conditioned & $0.0694 \pm 0.0267$ & $0.0239 \pm 0.0037$ & $0.4114 \pm 0.0140$ & $0.0587 \pm 0.0216$ \\
  & Predicted-target STC & $0.0569 \pm 0.0342$ & $0.0219 \pm 0.0061$ & $0.3406 \pm 0.0889$ & $0.0499 \pm 0.0306$ \\
  & Oracle STC & $0.0125 \pm 0.0037$ & $0.0137 \pm 0.0021$ & $0.2396 \pm 0.0482$ & $0.0110 \pm 0.0035$ \\
\midrule
\multirow[t]{3}{*}{ProGen2}
  & Source-conditioned & $0.0703 \pm 0.0264$ & $0.0244 \pm 0.0037$ & $0.4125 \pm 0.0113$ & $0.0599 \pm 0.0211$ \\
  & Predicted-target STC & $0.0612 \pm 0.0439$ & $0.0216 \pm 0.0069$ & $0.3261 \pm 0.1017$ & $0.0520 \pm 0.0372$ \\
  & Oracle STC & $0.0121 \pm 0.0047$ & $0.0135 \pm 0.0019$ & $0.2192 \pm 0.0538$ & $0.0103 \pm 0.0041$ \\
\bottomrule
\end{tabular}
\end{table}

Both predicted-target STC models improve on their source-conditioned counterparts on every metric in the patient-macro averages, with improvement for two patients and deterioration for one.


\end{document}